\documentclass[11pt]{article}
\usepackage[margin=1in]{geometry}

\usepackage{natbib}
\usepackage{booktabs} 

\usepackage{authblk}
\usepackage{graphicx}
\usepackage{subcaption}
\usepackage[utf8]{inputenc} 
\usepackage[T1]{fontenc}    
\usepackage{hyperref}       
\usepackage{url}            
\usepackage{booktabs}       
\usepackage{amsfonts}       
\usepackage{nicefrac}       
\usepackage{microtype}      
\usepackage{xcolor}  
\usepackage{amsmath}

\usepackage{enumitem}

\newcommand{\Utility}[1]{U^{#1}}

\newcommand{\Quality}{Q}
\newcommand{\Detector}{F}

\newcommand{\ContentNoArg}{z}

\newcommand{\TypeNoArg}{\mathbf{t}}
\newcommand{\Combination}[1]{m^{#1}}
\newcommand{\HumanContent}[1]{h^{#1}}
\newcommand{\LLMContent}[1]{\ell^{#1}}
\newcommand{\FinalContent}[1]{z^{#1}}
\newcommand{\FinalContentNoArg}{z}
\newcommand{\CostFn}{C}

\newcommand{\Penalty}{\beta}

\newcommand{\Threshold}{\nu}
\newcommand{\Fraction}[1]{\alpha^{#1}}
\newcommand{\FractionNoArg}{\alpha}
\newcommand{\ProductionCost}[1]{\gamma^{#1}}

\usepackage{amsthm}
\newtheorem{theorem}{Theorem}
\newtheorem{lemma}{Lemma}
\newtheorem{proposition}{Proposition}
\newtheorem{corollary}{Corollary}
\newtheorem{example}{Example}

\title{LLM Detection as an Intervention: Downstream Impact under Strategic User Behavior}

\author[1,3]{Meena Jagadeesan}
\author[1]{Tatsunori Hashimoto}
\author[2]{Jon Kleinberg}

\affil[1]{\textit{Stanford University}}
\affil[2]{\textit{Cornell University}}
\affil[3]{\textit{University of Pennsylvania}}

\begin{document}

\maketitle

\begin{abstract}
As LLM adoption becomes more widespread, there is a growing interest in detecting LLM-generated content, for example through LLM detection tools and through heuristics based on language patterns. Detectors operate as an intervention that steers not only the detected attribute itself, but also downstream metrics such as LLM usage and output quality. In this work, we demonstrate how imperfect LLM detectors lead to counterintuitive impacts on these downstream metrics, by distorting how users are incentivized to use LLMs in their workflow. We develop a stylized model which captures how users strategically choose how much to use the LLM and how to post-process content to reduce the detected attribute. Using this model, we show that LLM detection can counterintuitively lead humans to increase their LLM usage. Moreover, even when reducing the detected attribute improves output quality, we find that introducing an LLM detector can lead users to produce lower quality outputs. In contrast, we show that detectors result in a clean ``rise-then-fall'' pattern for the detected attribute, which we empirically reproduce for word frequencies on arXiv abstracts. Altogether, our work illustrates how LLM detection can distort LLM usage and output quality, uncovering failure modes when LLM detectors operate as an intervention on these downstream metrics.  
\end{abstract}

\section{Introduction}

As large language models (LLMs) get deployed at larger scale, content is increasingly scrutinized for whether it is LLM-generated or human-generated. This has led to the deployment of LLM detection tools \citep{DBLP:conf/icml/Mitchell0KMF23, DBLP:conf/naacl/VermaFTK24, DBLP:conf/icml/HansSCKSGGG24, DBLP:journals/corr/abs-2602-13042}  such as GPTZero and Pangram\footnote{See \url{https://gptzero.me} and \url{https://www.pangram.com/}.} in academic conferences \citep{ICLR} and in education \citep{prothero2024more}, among other domains (Example \ref{example:auditor}). Beyond these automated tools, LLM detection is performed implicitly by society using heuristics based on language patterns (e.g.,  increased usage of specific words such as ``delve'' \citep{liang2025quantifying, JW25} and stylistic patterns such as em-dashes \citep{emdash}). Society-based detection can harm authors, for example by leading to negative perceptions of content quality \citep{knight2023generative} (Example \ref{example:societal}).

LLM detection, whether performed by institutions or by society, operates as an intervention on several downstream metrics. One such metric is LLM usage: for example, educators may deploy LLM detectors in an attempt to \textit{reduce} LLM usage. Another such metric is output quality: academic research institutions may deploy LLM detectors in order to \textit{increase} the quality of academic publications. Yet another metric is the attribute that the detector directly penalizes. For example, the frequency of specific word patterns on arXiv \citep{liang2025quantifying, GT24} empirically exhibits a ``rise-then-fall'' pattern,   dramatically increasing after the release of ChatGPT, but then subsequently dropping in frequency in the following months \citep{GT25, L24}, perhaps due to society becoming aware of LLM-sounding words. 

However, the link between LLM detection and these downstream metrics is complicated by user incentives: imperfect LLM detectors fundamentally distort how users are incentivized to integrate LLMs into their workflow. Automated and heuristic LLM detectors routinely misclassify content \citep{DBLP:journals/patterns/LiangYMWZ23, weber2023testing} and are non-robust to small perturbations \citep{DBLP:conf/nips/KrishnaSKWI23, DBLP:journals/tmlr/SadasivanKB0F25}, and correctness is inherently ambiguous under human-LLM collaboration \citep{zhang2024llm}. To avoid detection, users  may be incentivized to post-process LLM-generated content\footnote{\url{https://www.quetext.com/blog/top-prompts-humanize-ai-generated-text}} or even their own human-generated content \citep{emdash}, for example by removing instances of LLM-sounding words (Example \ref{example:societal}).
Users may also be incentivized to change how much work they delegate to the LLM, for example by deciding to perform some tasks (e.g., reference generation) manually to avoid detection (Example \ref{example:auditor}). 

In this work, we study this distortion of user incentives, and show how imperfect LLM detection can lead 
to counterintuitive impacts on downstream metrics.   In our stylized model (Section \ref{sec:model}), outputs are embedded in $D$-dimensions: the LLM detector flags outputs using a threshold-based classifier on the first dimension (detected attribute), but quality depends on all $D$ dimensions. If an output is flagged by the LLM detector, the user who authored it incurs a penalty.  Each user chooses how much work to allocate to an LLM, which determines the convex combination between their own natural human-generated output and their fully LLM-generated output. The user also decides how much to post-process content by reducing the detected attribute (which may help them avoid being flagged by the detector). They optimize for their utility which is a combination of output quality, the detection penalty, production costs, and post-processing costs.  

Using this model, we prove that LLM detection can distort both LLM usage and output quality, even though it has a clean impact on the metric that it optimizes. 
\begin{itemize}[leftmargin=*]
\item \textbf{LLM usage:} Introducing an LLM detector can lead some users to \textit{increase} their LLM usage relative to the no-detector baseline (Section \ref{subsec:llmusage}; Figure \ref{fig:llmusage}).
Specifically, we show sufficient conditions under which there exist users who increase LLM usage under detection (Theorem \ref{thm:usageincrease}), as well as necessary conditions (Theorem \ref{thm:usagedecrease}) that match in the case where quality is linear in LLM usage. While one may expect that detection would lead users to reduce LLM usage to lower the detected attribute, our result shows the opposite can occur. The intuition is that post-processing can remove some quality-reducing characteristics produced by the LLM, which may incentivize users to selectively take greater advantage of those characteristics where LLM-generated content may be higher quality than human-generated content (Example \ref{ex:conceptualllmusage}).

\item \textbf{Output quality:} Introducing an LLM detector can lead users to produce \textit{lower-quality} outputs relative to the baseline, even when the detected attribute is quality-reducing (Section \ref{subsec:quality}; Figure \ref{fig:quality}). We prove sufficient conditions under which there exist users who produce lower-quality outputs under detection (Theorem \ref{thm:qualitydecrease}). While one may expect the LLM detector to improve quality when detection penalizes a quality-reducing characteristic, our result shows that detection can counterintuitively reduce quality. The intuition is that detection can lead users to under-use the LLM, and thereby lose the (non-detected) dimensions in which LLM use can increase quality (Example \ref{ex:conceptualqualitydecrease}).\footnote{The conditions in Theorem \ref{thm:qualitydecrease} do not satisfy the necessary conditions in Theorem \ref{thm:usagedecrease}, meaning that LLM detection weakly decreases LLM usage.}   

\item \textbf{Detected attribute:} Our model reproduces the empirically observed ``rise-then-fall'' pattern for the detected attribute \citep{GT25, L24} (Section \ref{sec:inverseU}; Figure \ref{fig:detectedattribute}). We prove the detected attribute always weakly increases when the user is given access to an LLM and then weakly decreases with the detector (Theorem \ref{thm:detectedattributeUshaped}). From an empirical perspective, we provide more systematic empirical validation of this pattern. Moving beyond the specific word lists studied in \citet{GT25} and \citet{L24}, we show the ``rise-then-fall'' pattern has increased in prevalence in the LLM era relative to prior time windows over the past decade (Figure \ref{fig:empirical}). 
\end{itemize}

Altogether, our work illustrates the unintended impacts of LLM detection, when accounting for how it shapes user workflow decisions. More broadly, these distortion results uncover key failure modes when institutions deploy detectors to intervene on LLM usage and output quality. 

\subsection{Related Work}

Our model builds on the literature on \textit{strategic classification} and \textit{human-AI collaboration}. 

\paragraph{Strategic classification.}   
A vast theoretical literature on strategic classification studies how users strategically change their features to improve their classification outcomes (see \citep{Rosenfeld24, P25} for surveys). One line of work focuses on how users can ``game'' the classifier to shift the features of their input without improving their true outcome (e.g., \citep{HardtMPW16, DongRSWW18, HuIV19, MilliMDH19, GhalmeNETR21, BechavodPWZ22}). The typical focus is on classifiers that perform well under gaming. Another line of work focuses on how users can take actions that not only change their features but also ``improve'' their true outcome (e.g., \citep{KleinbergR20, AlonDPTT20}), investigating how to account for and disentangle gaming and improvement (e.g., \citep{MillerMH20, DBLP:conf/icml/ShavitEA20, HaghtalabILW20, DBLP:conf/forc/AhmadiBBN22}).  Our model of post-processing directly builds on the model for gaming in \cite{HardtMPW16}. Similar to the improvement models, we capture how changing the detected attribute affects quality. A distinguishing feature of our work is that we combine post-processing with human-AI collaboration, showing how this joint decision fundamentally shapes user incentives and detector performance along downstream metrics. 

More broadly, our work contributes to a growing literature on the strategic behaviors in the context of LLMs. This includes the strategic LLM usage on online platforms \citep{DBLP:conf/icml/YaoLNWX24, DBLP:journals/corr/abs-2502-20783, DBLP:conf/www/KeinanB26}, the strategic impacts of LLMs in markets \citep{DBLP:journals/corr/abs-2603-25893, DBLP:journals/corr/abs-2404-00806}, and strategic LLM manipulations in hiring \citep{DBLP:conf/icml/0003HHS25}.

Our work also fits into a broader research perspective of viewing predictions as interventions in social systems \citep{L25}.   

\paragraph{Human-AI collaboration.} A vast theoretical literature on human-AI collaboration studies how to design and evaluate human-AI teams. These works consider many models: for example, models based on presenting subsets (e.g., \citep{DBLP:conf/icml/StraitouriWOR23, DBLP:conf/aaai/DonahueGK24, DBLP:journals/corr/abs-2503-11709}), delegation (e.g., \citep{DBLP:conf/chi/LaiCBLZT22, DBLP:conf/sigecom/GreenwoodLBHK25}), and learning over time (e.g., \citep{DBLP:conf/stoc/CollinaG0025, DBLP:journals/corr/abs-2602-17646}). A key desideratum is complementarity (e.g., \citep{DBLP:conf/chi/BansalWZFNKRW21, DonahueCK22, DBLP:journals/corr/abs-2204-10806}), which captures that the human-LLM team must perform better than either agent can perform in isolation. Our model takes a simplified view of human-AI collaboration as interpolation; a distinguishing feature is that we combine human-AI collaboration with post-processing. Our model exhibits complementarity in idealized cases, but we show LLM detection can compromise complementarity (Figure \ref{fig:quality}; Theorem \ref{thm:qualitydecrease}).

\section{Model}\label{sec:model}

We develop a stylized model which captures the following key features of user incentives under LLM detection. First, the detector is correlated with some, but not all, aspects of output quality. Second, users can assign fractional parts of the task to the LLM. Third, users can post-process their outputs to try to avoid detection. Finally, users are heterogeneous in terms of ability level and production costs.

\subsection{Output Detection and Quality}\label{subsec:prelims}

We embed outputs $\ContentNoArg$ into $\mathbb{R}^D$. Outputs are detected as being LLM-generated on the basis of one of the dimensions, which we take without loss of generality to be the \textit{first dimension}. We consider threshold-based detectors of the form $\Detector(\ContentNoArg) = 1[\ContentNoArg_1 > \Threshold]$, which predict whether a given $\ContentNoArg \in \mathbb{R}^D$ was generated by an LLM. The other $D-1$ dimensions thus capture attributes of the content that contribute to content quality (and may be affected by LLM use) but are independent of detection. The user faces a penalty $\Penalty > 0$ from being detected. 
The case of no LLM detector can be embedded into this model as either having zero penalty ($\Penalty = 0$) or having an infinite threshold ($\Threshold = \infty$), or both. 

Quality is measured by a separable function $\Quality(\ContentNoArg) := \sum_{i=1}^D \Quality_i(\ContentNoArg_i)$. Each $\Quality_i$ is twice continuously differentiable and concave. Moreover, $\Quality_1$ is weakly increasing or weakly decreasing, capturing whether the detected attribute is positively or negatively correlated with quality.

\subsection{User Types}\label{subsec:usertypes}

Each user is specified by a type $\TypeNoArg = (\HumanContent{\TypeNoArg}, \LLMContent{\TypeNoArg}, \ProductionCost{\TypeNoArg})$ where\footnote{Quality $\Quality(\HumanContent{\TypeNoArg})$ is affected by the user's ability, and $\Quality(\LLMContent{\TypeNoArg}) - \Quality(\HumanContent{\TypeNoArg})$ captures the quality change from using the LLM.}: 
\begin{itemize}[leftmargin=*, nosep]
    \item The \textit{human-generated output} $\HumanContent{\TypeNoArg} \in \mathbb{R}^D$ captures what the user would generate without the LLM.
    \item The \textit{LLM-generated output} $\LLMContent{\TypeNoArg} \in \mathbb{R}^D$ captures what the user would generate if they fully automate output generation using the LLM. 
    \item The \textit{human-generated production cost} $\ProductionCost{\TypeNoArg} \ge 0$ captures how much more expensive it is to produce human-generated content $\HumanContent{\TypeNoArg}$ than LLM-generated content $\LLMContent{\TypeNoArg}$.
\end{itemize}
The type space is $\mathcal{T} = \left\{ (\HumanContent{\TypeNoArg}, \LLMContent{\TypeNoArg}, \ProductionCost{\TypeNoArg}) \in \mathbb{R}^D \times \mathbb{R}^D \times \mathbb{R}_{\ge 0} \mid \LLMContent{\TypeNoArg}_1 > \HumanContent{\TypeNoArg}_1 \right\}$, which assumes that the detected attribute  is higher for the LLM-generated output than for the human-generated output. 

\subsection{User Workflow Decisions}\label{subsec:userdecisions}

Each user decides how to allocate work to the LLM vs. to themselves, and how to post-process their output to avoid detection. 
They derive value from quality, but face penalty $\Penalty > 0$ from detection. 

\paragraph{Work allocation.} The user chooses a fraction $\FractionNoArg \in [0,1]$ of the task to outsource to the LLM, which produces outputs according to the convex combination $\Combination{\TypeNoArg}(\FractionNoArg) = \FractionNoArg \cdot \LLMContent{\TypeNoArg} + (1 - \FractionNoArg) \cdot \HumanContent{\TypeNoArg}$. The user pays production cost $\ProductionCost{\TypeNoArg} \cdot (1-\FractionNoArg)$ that scales with the fraction of the task that they allocate to themselves. For notational convenience, let $\Combination{\TypeNoArg}_i(\FractionNoArg) =(\Combination{\TypeNoArg}(\FractionNoArg))_i$ for $i \in [D]$. 

\paragraph{Post-processing.} For post-processing, users can decide to decrease any of the coordinates of their output $\Combination{\TypeNoArg}(\FractionNoArg)$ to produce a final output $\FinalContentNoArg' \in \mathbb{R}^D$, facing a cost for these adjustments. This post-processing produces \textit{gaming costs} that are specified by a function $\CostFn: \left\{\ContentNoArg,\ContentNoArg' \in \mathbb{R}^D \mid \ContentNoArg'_i \le \ContentNoArg_i \;\;\; \forall i \in [D] \right\} \rightarrow \mathbb{R}_{\ge 0}$, where $\CostFn(\ContentNoArg, \ContentNoArg')$ captures the cost of post-processing output $\ContentNoArg$ into an output $\ContentNoArg'$ with weakly lower coordinates. Gaming costs are separable ($\CostFn(\ContentNoArg, \ContentNoArg') = \sum_{i=1}^D \CostFn_i (\ContentNoArg_i, \ContentNoArg'_i)$). We consider two cases: (a) \textit{finite costs} $\CostFn$ where each $\CostFn_i$ is finite and twice-continuously differentiable up to the boundary, and (b) \textit{infinite costs} $\CostFn^{\infty}$ where $\CostFn_i(\ContentNoArg_i, \ContentNoArg'_i) = \infty$ when $\ContentNoArg_i \neq \ContentNoArg'_i$. In the case of finite costs, we place the following additional assumptions on  each component:
\begin{enumerate}[label=(A\arabic*), leftmargin=*]
\item For any $\ContentNoArg_i \neq \ContentNoArg'_i$, costs exceed quality improvements (i.e., $\CostFn_i(\ContentNoArg_i, \ContentNoArg'_i) > \Quality_i(\ContentNoArg'_i) - \Quality_i(\ContentNoArg_i)$). This means that post-processing is not incentivized if detection penalties are zero (Lemma \ref{lemma:nodetectionnopostprocessing}). 
    \item $\CostFn_i(\ContentNoArg_i, \ContentNoArg'_i)$ is strictly decreasing in $\ContentNoArg'_i$ and strictly increasing in $\ContentNoArg_i$. 
    \item $\CostFn_i(\ContentNoArg_i, \ContentNoArg'_i)$ is nonnegative, and is $0$ if and only if $\ContentNoArg'_i  = \ContentNoArg_i$. 
    \item For any $\ContentNoArg'_i$, the costs $\CostFn_i(\ContentNoArg_i, \ContentNoArg'_i)$ approaches $\infty$ as $\ContentNoArg_i \rightarrow \infty$.
    \item For any $\ContentNoArg''_i$, it holds that the utility $\sup_{\ContentNoArg_i \ge \ContentNoArg''_i} (\Quality_i(\ContentNoArg'_i) -\CostFn_i(\ContentNoArg_i, \ContentNoArg'_i))$ approaches $-\infty$ as $\ContentNoArg'_i \rightarrow -\infty$. 
\end{enumerate}

\paragraph{Utility-optimization.} 
Users derive utility from quality, pay costs for production and post-processing, and face a penalty if they are detected. 
Given a work allocation $\FractionNoArg$ and final output $\FinalContentNoArg$ such that  $\ContentNoArg_i \le \Combination{\TypeNoArg}_i(\FractionNoArg)$ for all $i \in [D]$, the utility of a user with type $\TypeNoArg$ is: 
\[\Utility{\TypeNoArg}(\FractionNoArg,\FinalContentNoArg; \Penalty, \Threshold, \Quality, \CostFn) := \underbrace{\Quality(\FinalContentNoArg)}_{\text{output quality}} - \underbrace{\Penalty \cdot 1[\FinalContentNoArg_1 > \Threshold]}_{\text{detection penalty}} - \underbrace{\ProductionCost{\TypeNoArg} \cdot (1-\FractionNoArg)}_{\text{production costs}} - \underbrace{\CostFn(\Combination{\TypeNoArg}(\FractionNoArg), \FinalContentNoArg)}_{\text{post-processing costs}}.   \]
The user chooses a work allocation and post-processed output to maximize their utility:
\[(\Fraction{\TypeNoArg}(\Penalty, \Threshold; \Quality, \CostFn), \FinalContent{\TypeNoArg}(\Penalty, \Threshold; \Quality, \CostFn)) := \text{argmax}_{\FractionNoArg \in [0,1], \ContentNoArg \in \mathbb{R}^{D} \mid  \forall i \in [D]: \ContentNoArg_i \le \Combination{\TypeNoArg}_i(\FractionNoArg)} \Utility{\TypeNoArg}(\FractionNoArg,\FinalContentNoArg; \Penalty, \Threshold, \Quality, \CostFn).  \]
We show this optimization program is well-defined in Appendix \ref{appendix:characerization}. We assume users first tiebreak in favor of lower values of $\Fraction{\TypeNoArg}$, and then tiebreak in favor of lower values of $\FinalContent{\TypeNoArg}_1$.  

Let $(\Fraction{\TypeNoArg}(\emptyset; \Quality, \CostFn), \FinalContent{\TypeNoArg}(\emptyset; \Quality, \CostFn)) := (\Fraction{\TypeNoArg}(0, \infty; \Quality, \CostFn), \FinalContent{\TypeNoArg}(0, \infty; \Quality, \CostFn)) $ capture the case of no LLM detector.

\subsection{Illustrative Examples}\label{subsec:examples}

We discuss how to instantiate our model in the context of two illustrative examples.

\begin{example}
\label{example:auditor}
Institutions such as schools \citep{prothero2024more}, academic conference committees \citep{ICLR}, and grant agencies\footnote{See \url{https://grants.nih.gov/grants/guide/notice-files/NOT-OD-25-132.html}.} increasingly use tools such as GPTZero or Pangram to detect if authors are producing LLM-generated content. 
Let $z_1$ denote characteristics used for detection: for example, if the detection process focuses on hallucinated references, then first attribute $z_1$ may capture the estimated number of hallucinated references. Such detectors are often non-robust to superficial changes such as paraphrasing \citep{DBLP:conf/nips/KrishnaSKWI23}, and may misclassify content written by non-native English writers \citep{DBLP:journals/patterns/LiangYMWZ23}. Authors may thus post-process their LLM-generated content or human-generated content, facing a time-based cost $\CostFn$ for doing so. Authors may also opt to manually perform  steps in the research process, such as reference generation or writing, which affects their production cost $\ProductionCost{\TypeNoArg} \cdot (1-\FractionNoArg)$. 
\end{example}

\begin{example}
\label{example:societal}
Society also implicitly forms opinions about whether content is LLM-generated by using  heuristics. Detection may be performed based on stylistic quirks: let $z_1$ denote the frequency of stylistically LLM-sounding words such as ``delve'' \citep{JW25, liang2025quantifying} or stylistic patterns such as em-dashes \citep{emdash}. Alternatively, detection may be done based on polish: let $z_1$ denote the level of polish of the content (e.g., lack of typos). The penalty $\Penalty$ captures the implicit reputational cost to the authors, coming from perceived correlations between LLM usage and author identity \citep{lepp2025you}, content quality \citep{knight2023generative}, or authenticity. If detection is based on stylistic quirks, authors may manually remove some of these quirks in LLM-generated content, or even in fully human-generated content if they happen to naturally adopt the same patterns \citep{emdash2}. They may also post-process content by prompting an LLM to avoid these patterns.\footnote{See \url{https://www.quetext.com/blog/top-prompts-humanize-ai-generated-text}} If detection is based on polish, then authors may introduce typos into outputs using automated tools.\footnote{See \url{https://sinceerly.com}} 
Authors may also opt to use models for brainstorming rather than for writing, which affects their production cost $\ProductionCost{\TypeNoArg} \cdot (1-\FractionNoArg)$. 
\end{example}

\begin{figure}[t]
\centering
\begin{subfigure}[c]{0.32\linewidth}
    \centering
    \includegraphics[width=\linewidth]{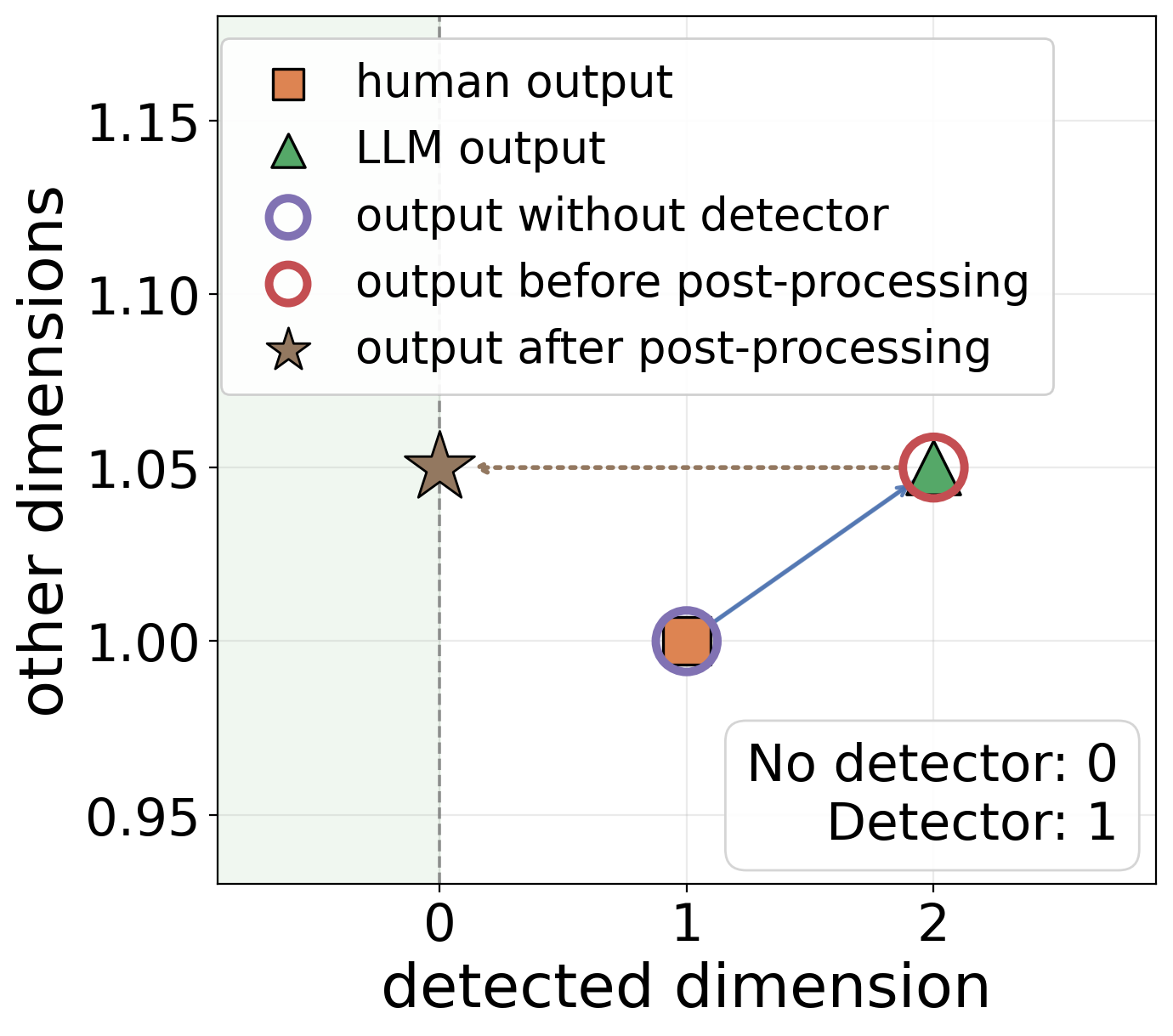}
    \caption{LLM usage}
    \label{fig:llmusage}
\end{subfigure}
\begin{subfigure}[c]{0.32\linewidth}
    \centering
    \includegraphics[width=\linewidth]{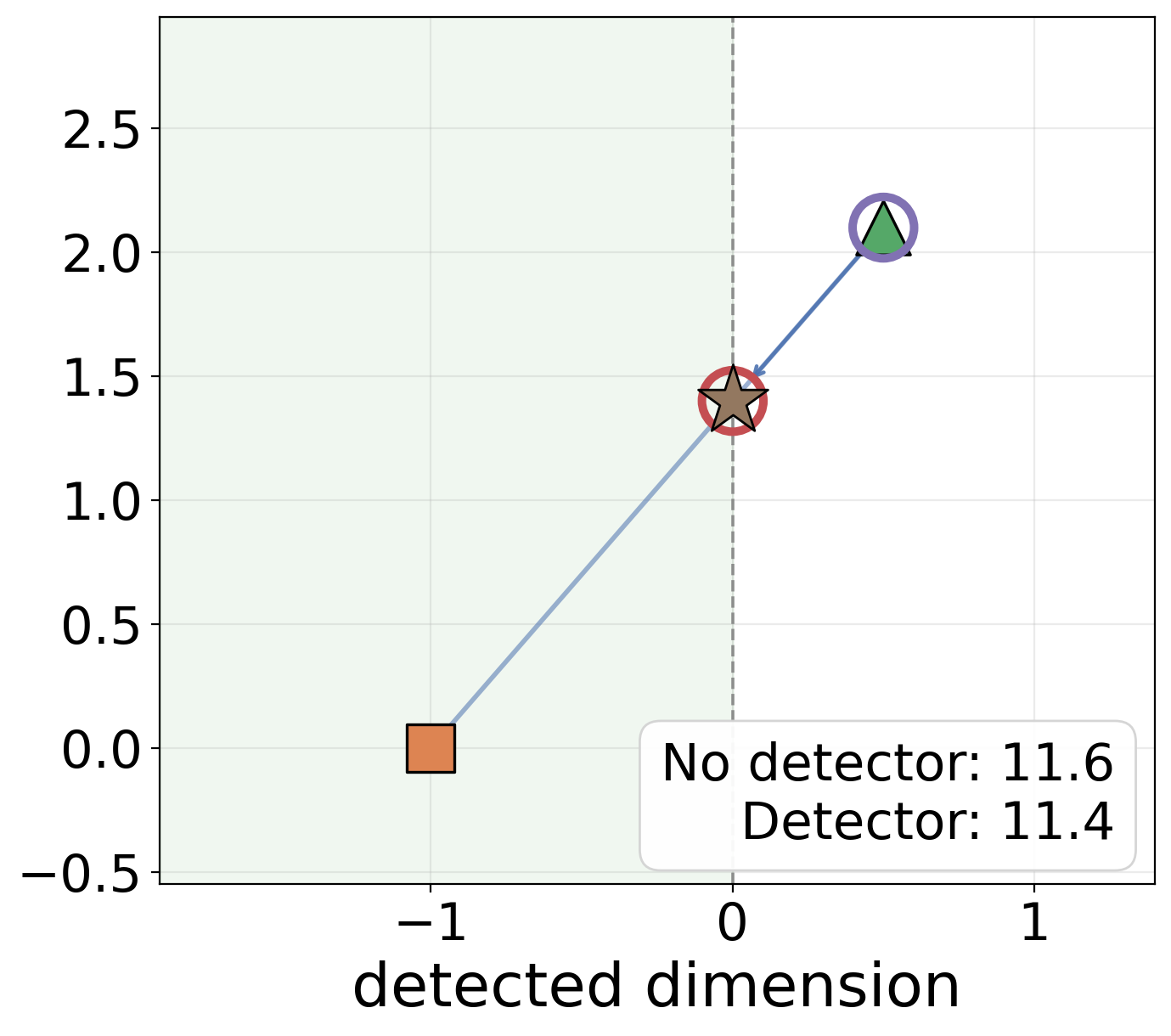}
    \caption{Output quality}
    \label{fig:quality}
\end{subfigure}
\begin{subfigure}[c]{0.32\linewidth}
    \centering
    \includegraphics[width=\linewidth]{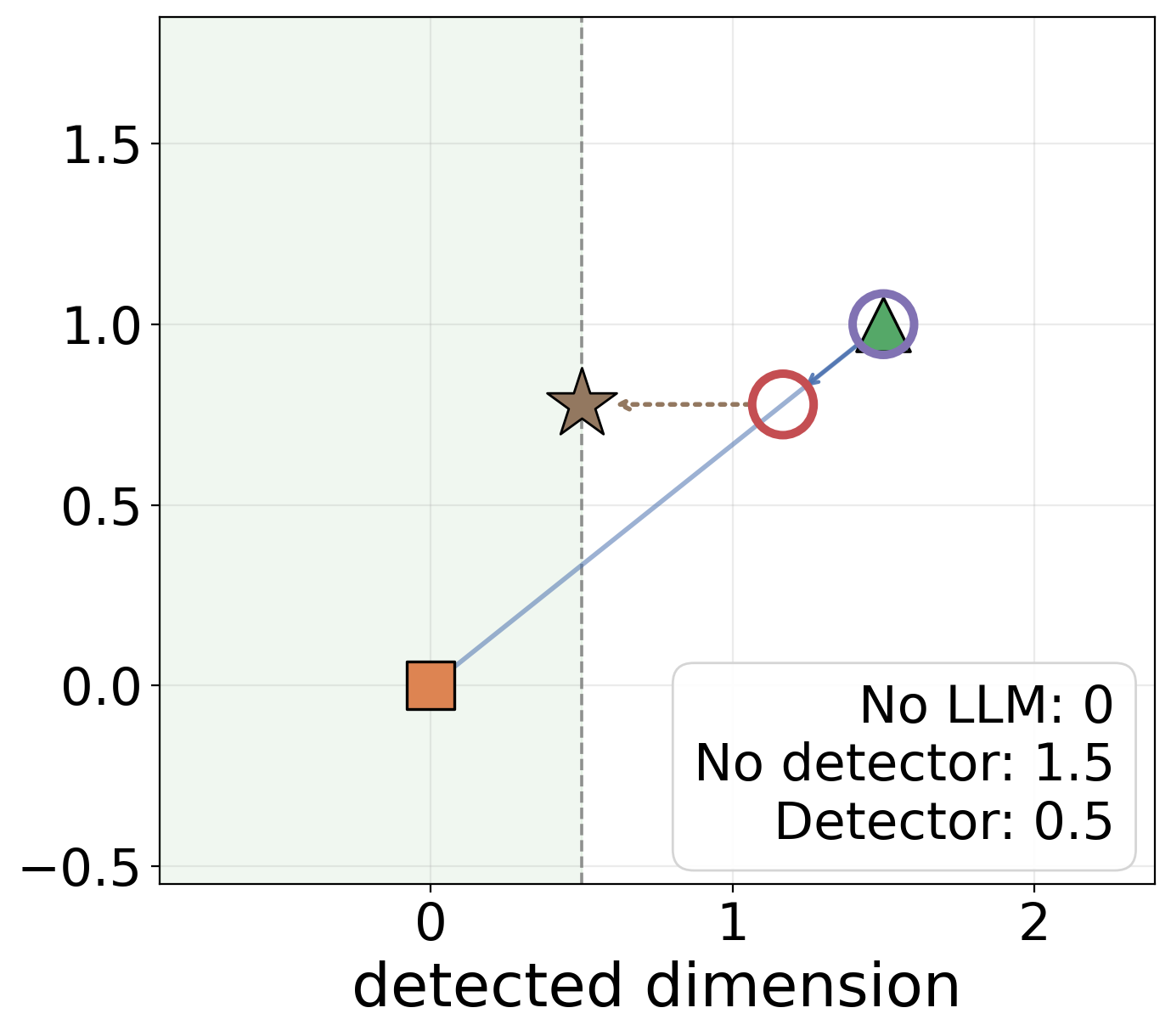}
    \caption{Detected attribute}
    \label{fig:detectedattribute}
\end{subfigure}
\caption{Impact of LLM detector on LLM usage (left), output quality (middle), and the detected attribute (right). Quality is linearly decreasing in the detected dimension (i.e., the detector targets quality-reducing characteristics), and linearly increasing in the other dimension. The plot shows how the presence of the LLM detector can lead users to \textit{increase their LLM usage} relative to the no-detector baseline (Theorem \ref{thm:usageincrease}; Example \ref{ex:conceptualllmusage}), can lead users to produce \textit{lower-quality content} than the no-detector baseline (Theorem \ref{thm:qualitydecrease}; Example \ref{ex:conceptualqualitydecrease}), and leads to a clean ``rise-then-fall'' pattern for the detected attribute (Theorem \ref{thm:detectedattributeUshaped}; Figure \ref{fig:empirical}). }
\label{fig:fig1}
\end{figure}

\section{Impact on LLM Usage}\label{subsec:llmusage}

At first glance, it may seem that introducing an LLM detector would weakly decrease LLM usage $\Fraction{\TypeNoArg}(\Penalty, \Threshold; \Quality, \CostFn)$ relative to the no-detector baseline $\Fraction{\TypeNoArg}(\emptyset; \Quality, \CostFn)$. Since we assume that the detected attribute for a user's LLM-generated content is strictly larger than the detected attribute for the user's human-generated content ($\LLMContent{\TypeNoArg}_1 > \HumanContent{\TypeNoArg}_1$), LLM detection may seem to lead users to reduce their LLM usage in order to lower the detected attribute. 

Our results in this section illustrate how this argument breaks down: we find that introducing an LLM detector can counterintuitively lead some users to increase their LLM usage. 

\subsection{Sufficient condition for increase in LLM usage}

The following result  shows sufficient conditions under which there exists some user type $\TypeNoArg$ that increases their LLM usage relative to the no-detector baseline. 
\begin{theorem}
\label{thm:usageincrease}
Fix $\Threshold < \infty$, $\Quality$, $\CostFn$. If the costs $\CostFn$ are finite and there exists $z^u_1 > z^l_1  > \Threshold$ such that:
\begin{equation}
\label{eq:conditionincrease}
 \max_{z'_1 \le \Threshold} \left(\Quality_1(z'_1) - \CostFn_1(z_1^u, z'_1) \right) - \max_{z'_1 \le \Threshold} \left(\Quality_1(z'_1) - \CostFn_1(z^l_1, z'_1) \right) > \Quality'_1(z_1^l) \cdot (z_1^u - z_1^l),
\end{equation}
then there exists a type $\TypeNoArg \in \mathcal{T}$ and a penalty threshold $\bar{\Penalty} > 0$ such that the presence of the detector strictly increases LLM usage: $\Fraction{\TypeNoArg}(\Penalty, \Threshold; \Quality, \CostFn) > \Fraction{\TypeNoArg}(\emptyset; \Quality, \CostFn) \text{ for all } \Penalty > \bar{\Penalty}$. 
\end{theorem}

Theorem \ref{thm:usageincrease} (Figure \ref{fig:llmusage}) shows conditions under which the presence of the detector strictly increases LLM usage for some user type. The condition \eqref{eq:conditionincrease} is based on the \textit{after-post-processing utility} $\max_{z'_1 \le \Threshold} \left(\Quality_1(z'_1) - \CostFn_1(z_1, z'_1) \right)$. This captures a user's utility from dimension 1 if they start with content $z_1$, decide to post-process their content to fall below the boundary, and then optimally post-process their content to maximize their utility. The after-post-processing utility is always weakly decreasing in $z_1$, because costs increase from being further away from the boundary (Lemma \ref{lemma:maxweaklydecreasing}). This means that the left-hand-side is always negative, so the condition is violated when the quality $\Quality_1$ is weakly increasing in $z_1$. When $\Quality_1$ is weakly decreasing, the condition can be restated as $|\max_{z'_1 \le \Threshold} \left(\Quality_1(z'_1) - \CostFn_1(z_1^u, z'_1) \right) - \max_{z'_1 \le \Threshold} \left(\Quality_1(z'_1) - \CostFn_1(z^l_1, z'_1) \right)| < |\Quality'_1(z_1^l) \cdot (z_1^u - z_1^l)|$: this requires the after-post-processing utility decreases at a slower rate than quality on the interval $[z_1^l, z_1^u]$. 

To see that \eqref{eq:conditionincrease} is achievable, 
let $\Quality_1(z_1) = k_1 - k_2 z_1$, let $\CostFn_1(z_1, z'_1) = k_2(z_1 - z'_1) + \frac{k_3}{1 + e^{z_1}} (1 - e^{-(z_1 - z'_1)}) + k_4 (z_1 - z'_1)^2$, 
and let $\Threshold = 0$. As long as $k_2$ is sufficiently large, $k_3$ is sufficiently small, and $k_4$ is sufficiently small, this setup satisfies \eqref{eq:conditionincrease} (Proposition \ref{proposition:nonmonotoneusage} in Appendix \ref{appendix:propconstruction}). Observe that the quality is linear and weakly decreasing in the detected attribute $z_1$, and the costs have a linear structure with exponential and quadratic perturbations. 

\paragraph{Proof sketch of Theorem \ref{thm:usageincrease}.} The intuition is as follows and is depicted in Figure \ref{fig:llmusage}.\footnote{We focus on an intuitive 2-d construction from Theorem \ref{thm:alternateusageincrease}, rather than the 1-d construction from the proof of Theorem \ref{thm:usageincrease}.} When the penalty $\Penalty$ is sufficiently high, users post-process their content to avoid detection. Users choose their work allocation $\FractionNoArg$ taking into account that they will post-process their content. Since \eqref{eq:conditionincrease} ensures that $\Quality_1$ is weakly decreasing, post-processing removes quality-reducing characteristics arising from LLM usage. When the LLM-generated content is higher quality than human-generated content along other characteristics (as captured by the second dimension), the user can take greater advantage of the LLM performing well along these characteristics under detection, without suffering from the quality-reducing characteristics. The condition \eqref{eq:conditionincrease} guarantees that the after-post-processing utility decreases sufficiently slowly that assigning more work to the LLM is worth it.

The proofs in Appendix \ref{appendix:proofthmincrease} formalize this intuition. We construct a user type where $\LLMContent{\TypeNoArg}_1  = z_1^u$ and  $\HumanContent{\TypeNoArg}_1 = z_1^l$. We can set the second dimension so that the quality of the LLM-generated output slightly exceeds the quality of the human-generated output (i.e., $\Quality_2(\LLMContent{\TypeNoArg}_2) > \Quality_2(\HumanContent{\TypeNoArg}_2)$ with a small gap). We set the production costs ($\ProductionCost{\TypeNoArg} \approx -\Quality'_1(z_1^l) (z_1^u - z_1^l)$) to ensure that the user will not use the LLM in the no-detector baseline (i.e., $\Fraction{\TypeNoArg}(\emptyset; \Quality; \CostFn) = 0$). The condition in \eqref{eq:conditionincrease} ensures that when there is a detector, the user achieves greater utility from fully allocating their work to the LLM (i.e., $\FractionNoArg = 1$) than from not using the LLM (i.e., $\FractionNoArg = 0$), so the user will at least partially use the LLM. 

\paragraph{Illustrative example.} To help distill the intuition from the proof, we provide an illustrative example. 
\begin{example}
\label{ex:conceptualllmusage}
While LLMs exhibit stylistic quirks (e.g., overusage of the word ``delve'' and em-dash) or make errors in terms of references or facts, human authors also exhibit some of these characteristics to a lesser degree. Without a detector, even when the LLM is slightly more skilled at producing interesting arguments than many authors, authors may decide to produce outputs themselves to reduce stylistic quirks and errors. When the detector catches not only the LLM-generated content but also the human-generated content produced by the author \citep{emdash, emdash2, DBLP:journals/patterns/LiangYMWZ23}, authors may post-process all content to remove stylistic quirks and to correct errors. Given that they have to do this post-processing anyway, authors may be incentivized to use an LLM to take advantage of its interesting arguments.   
\end{example}

\paragraph{Implications for institutions.} Theorem \ref{thm:usageincrease} highlights that while institutions (e.g., educators) may deploy LLM detectors in an effort to reduce LLM usage, this can backfire and lead some users to use the LLM more. Due to the availability of post-processing as an option for users, this result holds even in the extreme case where penalty becomes arbitrarily severe ($\Penalty \rightarrow \infty$), so institutions cannot even rely on severe penalties to reduce incentives for LLM usage.

\subsection{Necessary condition for increase in LLM usage}

To complement Theorem \ref{thm:usageincrease}, we derive necessary conditions under which there exists a user type that increases their LLM usage relative to the no-detector baseline.\footnote{Note that the necessary condition is more general, since it provides a penalty-independent condition that applies to every value of $\Penalty$, rather than only for sufficiently large $\Penalty$.}
\begin{theorem}
\label{thm:usagedecrease}
Fix $\Penalty > 0$, $\Threshold < \infty$, $\Quality$, $\CostFn$. If there exists a type $\TypeNoArg \in \mathcal{T}$ such that the presence of the detector strictly increases LLM usage $\Fraction{\TypeNoArg}(\Penalty, \Threshold; \Quality, \CostFn) > \Fraction{\TypeNoArg}(\emptyset; \Quality, \CostFn)$, then gaming costs $\CostFn$ are finite and there exist $z^u_1 > z^l_1  > \Threshold$ such that:
\begin{equation}
\label{eq:condition}
 \max_{z'_1 \le \Threshold} \left(\Quality_1(z'_1) - \CostFn_1(z_1^u, z'_1) \right) - \max_{z'_1 \le \Threshold} \left(\Quality_1(z'_1) - \CostFn_1(z^l_1, z'_1) \right) > \Quality_1(z_1^u) - \Quality_1(z_1^l).
\end{equation}
\end{theorem}

The condition \eqref{eq:condition} in Theorem \ref{thm:usagedecrease} has a similar structure to the condition \eqref{eq:conditionincrease} from the sufficient condition, but the right-hand side is changed from the first-order approximation $\Quality'_1(z_1^l) (z_1^u - z_1^l)$ to the difference $\Quality_1(z_1^u) - \Quality_1(z_1^l)$. These conditions are equivalent when $\Quality_1$ is linear, meaning that Theorems \ref{thm:usageincrease} and \ref{thm:usagedecrease} fully characterize when there exists a user type where detection leads to greater LLM usage. However, these conditions can differ slightly in the nonlinear case.\footnote{As a sanity check, recall that since $\Quality_1$ is assumed to be concave, we know that $\Quality'_1(z_1^l) (z_1^u - z_1^l) \ge \Quality_1(z_1^u) - \Quality_1(z_1^l)$, meaning that the sufficient condition implies the necessary condition as expected.} We defer the proof to  Appendix \ref{appendix:thmdecrease}. 

\paragraph{Proof sketch of Theorem \ref{thm:usagedecrease}.} The intuition for Theorem \ref{thm:usagedecrease} is as follows. It is easier to show the contrapositive: if gaming costs are infinite or if \eqref{eq:condition} is not satisfied, then the presence of the detector weakly decreases LLM usage for all user types. Intuitively, when gaming costs are infinite, then the users do not post-process their content, so the only way that they can avoid detection is through reducing LLM usage, and LLM detection thus leads to reduced LLM usage. When gaming costs are finite, if \eqref{eq:condition} is not satisfied, the after-post-processing utility $\max_{z'_1 \le \Threshold} \left(\Quality_1(z'_1) - \CostFn_1(z_1, z'_1) \right)$ decreases sufficiently quickly in $z_1$ that it is not worth it to take greater advantage of any characteristics where the LLM-generated content outperforms human-generated content. 

The proof, which is formalized in Appendix \ref{appendix:thmdecrease}, leverages a partial characterization of the optimal user workflow decisions for general user types (Appendix \ref{appendix:characerization}). Two key technical complexities arise: (1) users may be incentivized to post-process beyond the threshold, and (2) the set of work allocations where post-processing is incentivized may not be an interval. We nonetheless are able to show the optimal work allocation falls within a restricted set of possible values (Theorem \ref{thm:detectorfinitecosts}). In contrast, the case where gaming costs are infinite ($\CostFn = \CostFn^{\infty}$) is cleaner because  there is no post-processing (Theorem \ref{thm:infinitecostdetector}). 

\paragraph{Clean cases where LLM detection weakly decreases LLM usage.} As a consequence of Theorem \ref{thm:usagedecrease}, we can derive several cases where LLM detection does not increase LLM usage for any user. 
\begin{corollary}
\label{cor:monotoneusage}
Fix $\Quality$, $\Penalty$, $\Threshold$, and $\CostFn$. Suppose that at least one of the following conditions holds:
\begin{enumerate}[leftmargin=*, nosep]
    \item Gaming costs $\CostFn = \CostFn^{\infty}$ are infinite.
    \item Quality $\Quality_1$ is weakly increasing.
    \item Gaming costs $\CostFn$ satisfy the mixed partial condition $\frac{\partial^2 \CostFn_1 (z_1, z'_1)}{\partial z_1 \partial z'_1} \le 0$ for all $z'_1 \le z_1$.
\end{enumerate}
For any type $\TypeNoArg$, the detector weakly decreases LLM usage: $\Fraction{\TypeNoArg}(\Penalty, \Threshold; \Quality, \CostFn) \le  \Fraction{\TypeNoArg}(\emptyset; \Quality, \CostFn)$. 
\end{corollary}

\paragraph{Implications for institutions.} 
Corollary \ref{cor:monotoneusage} provides conditions under which institutions do not need to worry about whether LLM detection will backfire in terms of LLM usage. The first condition illustrates the benefits of robust LLM detectors, but does not reflect current detectors which are vulnerable to post-processing (e.g., \citep{DBLP:conf/nips/KrishnaSKWI23}). The second condition illustrates potential benefits of detecting using quality-improving characteristics. While this condition is unlikely to be met when detecting on the basis of stylistic quirks or factual errors, it may be met by detection on the basis of document polish.\footnote{That being said, there may be practical barriers to adopting detectors which penalize LLM characteristics which improve quality.}  
The third condition is more technical, and may be difficult to control via detector design.

\section{Impact on Quality}\label{subsec:quality}
 
When the detected attribute is negatively correlated with output quality (i.e., $\Quality_1$ is weakly decreasing in $z_1$), one may expect the LLM detector to \textit{improve} quality since it penalizes quality-reducing characteristics. However, the results in this section show how this argument breaks down. Specifically, we show that even when the clean cases from Corollary \ref{cor:monotoneusage} do apply, LLM detection can still distort output quality, leading users to produce lower-quality content than the no-detector baseline.

\paragraph{Decrease in output quality.} The following result constructs sufficient conditions under which there exist user types who produce strictly lower-quality outputs than in the no-detector baseline.  

\begin{theorem}
\label{thm:qualitydecrease}
Fix $\Penalty > 0$, $\Threshold < \infty$, $\Quality$, and $\CostFn$, and suppose that at least one of the three conditions in Corollary \ref{cor:monotoneusage} holds. If gaming costs $\CostFn$ are finite, then suppose also that $\nabla_1 (\CostFn_1(\Threshold, \Threshold)) > \max(0, -\Quality'_1(\Threshold))$. Suppose also  $D \ge 2$, that there exists a dimension $2 \le i' \le D$ such that $\Quality_{i'}$ is non-constant. Then, the presence of the detector decreases quality for some user type $\TypeNoArg_D$: that is $\Quality( \FinalContent{\TypeNoArg_D}(\Penalty, \Threshold; \Quality, \CostFn)) < \Quality( \FinalContent{\TypeNoArg_D}(\emptyset; \Quality, \CostFn))$. 
\end{theorem}

Theorem \ref{thm:qualitydecrease} (Figure \ref{fig:quality}) illustrates how introducing an LLM detector can hurt output quality, even when $\Quality_1$ is weakly decreasing. The conditions in Theorem \ref{thm:qualitydecrease} build on the clean cases from Corollary \ref{cor:monotoneusage} where LLM detection weakly reduces LLM usage. When gaming costs are finite, we place the additional condition that the marginal costs of post-processing are sufficiently large close to the boundary. As a concrete example, note that for any linear quality function $\Quality_1$ which is weakly decreasing in the detected attribute, it is easy to see the conditions in Theorem \ref{thm:qualitydecrease} are satisfied for linear costs $\CostFn_1(z_1, z'_1) = (-Q'_1(\Threshold) + \epsilon) \cdot (z_1 - z'_1)$ for any $\epsilon > 0$ as well as for  infinite costs $\CostFn = \CostFn^{\infty}$.

\paragraph{Proof sketch of Theorem \ref{thm:qualitydecrease}.} The intuition  is as follows and is depicted in Figure \ref{fig:quality}. The idea is that  quality increase from post-processing can be offset by decrease in LLM usage. Specifically, suppose that the LLM-generated content is higher quality than human-generated content on non-detected characteristics (as captured by the second dimension). We take this quality gap to be small enough that users strictly reduce LLM usage under detection. The quality gap needs to be large enough that users create LLM-generated content in the no-detector baseline, and the gap also needs to be large enough that reducing LLM usage reduces output quality enough to offset any quality increase from post-processing. The proof is in Appendix \ref{appendix:quality}.

\paragraph{Illustrative examples.} To distill the intuition underpinning Theorem \ref{thm:qualitydecrease}, we provide illustrative examples for how the detector can lead to decreased quality. First, we consider the case where the quality is negatively correlated with the detected attribute, where the intuition for the quality decrease is more subtle. 
\begin{example}
\label{ex:conceptualqualitydecrease}
Like in Example \ref{ex:conceptualllmusage}, consider detectors based on stylistic quirks (e.g., overusage of the word ``delve'') or errors in terms of references or facts; however, suppose that the detector is much more conservative and does not flag content produced with a sufficient level of human collaboration (even though it flags fully LLM-generated content). Without a detector, if the LLM is sufficiently more skilled at producing interesting arguments than the human author, the author may produce LLM-generated content. With detection, if post-processing is costly, authors may instead be incentivized to use the LLM less, leading to content with less interesting arguments but also a lesser degree of stylistic quirks. If the quality reduction from losing interesting arguments is sufficiently large relative to the quality gain from mitigating stylistic quirks, detection would reduce content quality. 
\end{example}
Next, we turn to the case where quality is positively correlated with the detected attribute. In this case, the intuition is simpler: the quality decrease can come directly from post-processing. 
\begin{example}
\label{ex:conceptualquality}
Content generated by LLMs tends to be highly polished relative to human-generated content. When polish is used for detection, users may be incentivized to reduce the polish of their LLM-generated outputs, for example using automated tools to introduce typos.\footnote{See \url{https://sinceerly.com/}.} This type of post-processing can lead users to produce lower-quality outputs than if there had not been a detector.   
\end{example}

\paragraph{Implications for institutions.} Theorem \ref{thm:qualitydecrease} illustrates that institutions cannot assume that detection will translate to improvements in downstream output quality, regardless of whether the detector penalizes quality-reducing or quality-increasing attributes. In Appendix \ref{appendix:quality}, we also show that there exist types where LLM detection increases quality and where LLM detection leads to no change in quality relative to the no-detector baseline. This illustrates that detection has a heterogeneous impact on quality across users.

\begin{figure}[t]
\centering
\begin{subfigure}{0.48\linewidth}
\centering 
\includegraphics[height=5.5cm]{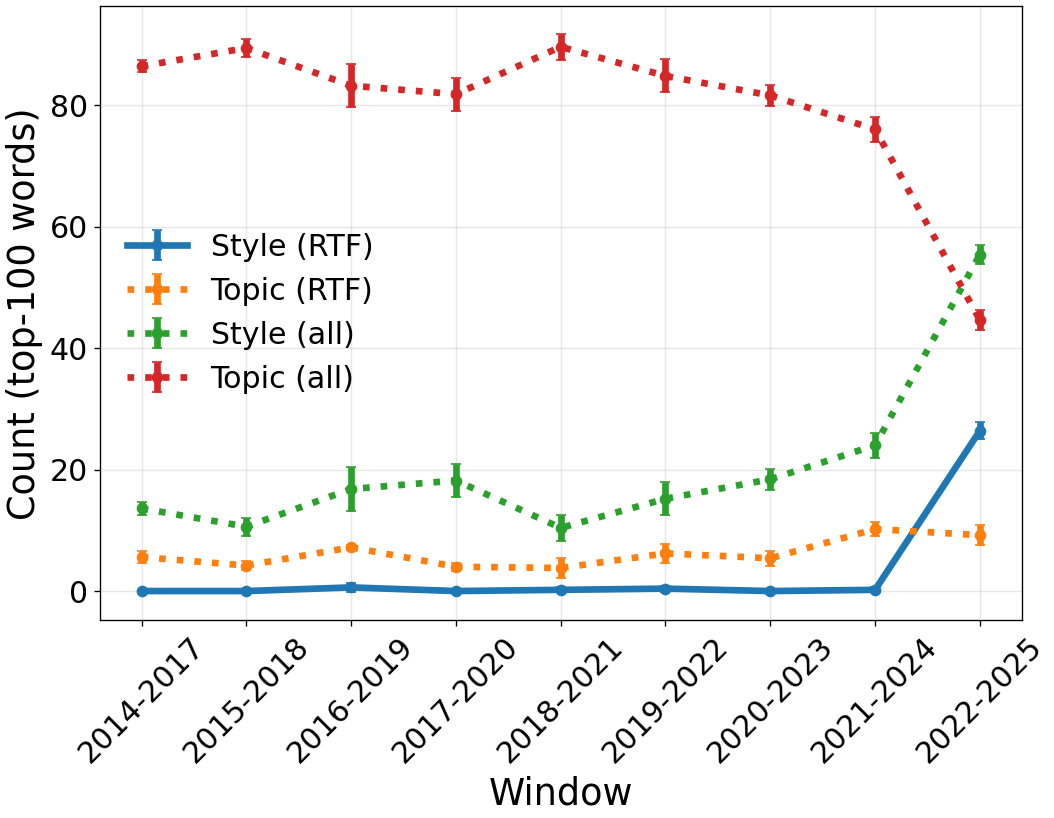}
\caption{Counts of rise-then-fall (RTF) patterns}
\label{tab:empirical}
\end{subfigure}
\begin{subfigure}{0.48\linewidth}
    \centering
    \includegraphics[height=5.5cm]{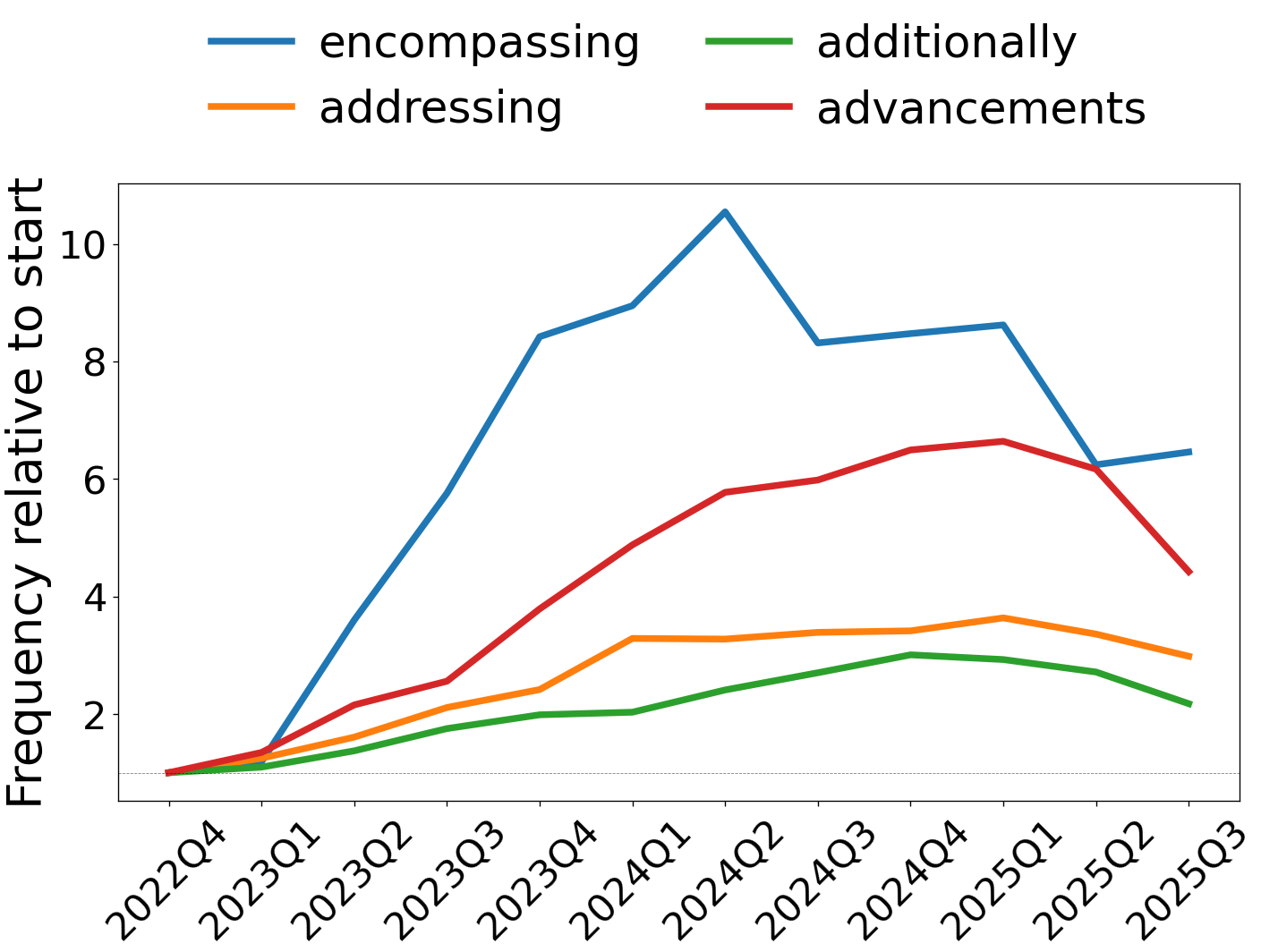}
        \caption{Word frequency curve} 
    \label{fig:figempirical}
\end{subfigure}
\caption{Empirical analysis of word frequencies on arXiv abstracts (Section \ref{subsec:empirical}). The top 100 words with the greatest change in each 3-year time window are computed. The left figure counts the number of style words and topic words exhibiting a ``rise-then-fall'' (RTF) pattern, and the total number of style words, over five trials with 2 standard errors. The right figure illustrates 4 example style words from the 2022-2025 window exhibiting the ``rise-then-fall pattern''. The full word list is deferred to Appendix \ref{appendix:empirical}. These results show how the ``rise-then-fall'' pattern became significantly more prevalent in the LLM era (2022-2025), relative to prior time windows. }
\label{fig:empirical}
\end{figure}

\section{Impact on Detected Attribute}\label{sec:inverseU}

In contrast with the distortion from Section \ref{subsec:llmusage} and Section \ref{subsec:quality}, we show that LLM detection has a clean impact on the detected attribute, reproducing an empirical ``rise-then-fall'' pattern. 

\subsection{Theoretical Analysis}

We prove that the detected attribute always weakly increases when LLMs are introduced into the world, and then always weakly decreases when an LLM detector is deployed. We also construct a user type for which both of these inequalities are strict.  
\begin{theorem}
\label{thm:detectedattributeUshaped}
Fix $\Penalty > 0$, $\Threshold < \infty$, $\Quality$, and $\CostFn$.  
For any type $\TypeNoArg$, the presence of the LLM weakly increases the detected attribute (i.e., $\HumanContent{\TypeNoArg}_1 \le \FinalContent{\TypeNoArg}_1(\emptyset; \Quality, \CostFn)$), and the presence of the detector weakly decreases the detected attribute (i.e., $\FinalContent{\TypeNoArg}_1(\emptyset; \Quality, \CostFn) \ge  \FinalContent{\TypeNoArg}_1(\Penalty, \Threshold; \Quality, \CostFn)$). There exists a type $\TypeNoArg$ for which both of these inequalities are strict. 
\end{theorem}

Theorem \ref{thm:detectedattributeUshaped} shows that the detected attribute always follows a clean pattern, thus confirming that LLM detection does not distort the metric that it optimizes. We defer the proof to Appendix \ref{appendix:inverseU}.

\subsection{Empirical Analysis: Word Frequency on ArXiv Abstracts}\label{subsec:empirical}

To complement Theorem \ref{thm:detectedattributeUshaped}, we empirically show that the LLM era exhibits a significantly greater prevalence of the rise-then-fall pattern for style words in arXiv papers  \citep{arxiv_org_submitters_2024}. Our analysis builds on the analysis in \citet{liang2025quantifying, GT25}. 

\paragraph{Setup.} We summarize the empirical setup, deferring details to Appendix \ref{appendix:empirical}.\footnote{The code is available at \url{https://github.com/mjagadeesan/detection-analysis-empirical}.} In each of 5 trials, we sample 3000 abstracts per month from the cs category on arXiv between 1/1/2013 and 12/31/2025 \citep{arxiv_org_submitters_2024}. We pre-process each abstract to build a vocabulary $V$. We build nine, 3-year time windows of the form [10/1/$s$, 9/30/($s+3$)], where $s \in \left\{2014, 2015, \ldots, 2022\right\}$.
For each time interval $s$: 
\begin{itemize}[leftmargin=*]
    \item  
    \textit{Top 100 words:} For every word $w \in V$, we compute a reference probability $r^s_w$ which captures the fraction of words in [10/1/($s$), 9/30/($s+1$)] which are equal to $w$. Let $N_{\text{total}}^s$ be the number of words in [10/1/($s+1$), 9/30/($s+2$)], and let $N^s(w)$ be the number of instances of word $w \in V$. For each word, we compute the likelihood $\ell^s(w) := \mathbb{P}[\text{Bin}(N^s_{\text{total}}, r^s_w) = N^s(w)]$. We then select the 100 words with the lowest values of $\ell^s(w)$. This produces a set $W^s$ of words.\footnote{This builds on a prior approach \citep{kleinberg2016temporal, monroe2008fightin}. Specifically, like these works, we consider a reference distribution with words sampled i.i.d. from a categorical distribution $\left\{r_w^s\right\}_{w \in V}$, and we compute the probability $\ell^s(w)$ of the next year's observed word count for $w$ under this distribution.} 
    \item \textit{Classification:} For each word $w \in W^s$, we use an LLM judge to classify the word  as ``style'' ($J(w) = S$) or ``topic'' ($J(w) = T$). We also determine whether the word fits a ``rise-then-fall'' pattern in the interval [10/1/$s$, 9/30/($s+3$)] looking at the gap between the start (and end) versus the maximum and the number of sign changes. Let $R^s(w) = 1$ if $w$ fits the rise-then-fall pattern in this interval, and let $R^s(w) = 0$ otherwise. For each value of $s$, we compute the number of words in different subcategories based on the value of $R^s(w)$ and $J(w)$. 
\end{itemize}

\paragraph{Results.} Figure \ref{tab:empirical} shows that the number of style words in the top 100 words fitting the ``rise-then-fall'' pattern substantially increases in the 2022-2025 relative to every other window in the preceding eight three-year time windows. This illustrates that high prevalence of the ``rise-then-fall'' pattern is unique to the LLM era. Moreover, the number of topic ``rise-then-fall'' words changes much less dramatically, highlighting how this trend is unique to style words. As an illustrative example, we depict the shape of these curves in 2022-2025 for one trial in Figure \ref{fig:figempirical}. Ablations and full word lists are deferred to Appendix \ref{appendix:empirical}.

Our analysis differs from prior work \citep{GT25} in two ways: (1) we systematically study this pattern across all words with a high degree of change, moving beyond the word lists identified in \citet{liang2025quantifying}, and (2) we compare the prevalence of the ``rise-then-fall'' pattern in 2022-2025 to prior windows.

One possible driver of this rise-then-fall pattern is that users strategically change their workflow decisions in response to LLM detection, along the lines of our theoretical analysis (Theorem \ref{thm:detectedattributeUshaped}). However, we note our analysis does not disentangle a causal claim. Other factors may have contributed to the rise-then-fall pattern, including newer LLMs using these words less frequently. 

\section{Discussion}\label{sec:discussion}

We study how LLM detection affects downstream metrics, when taking into account user incentives. Using a stylized model, we show LLM detection can counterintuitively increase LLM usage, and prove sufficient conditions and necessary conditions for this phenomenon which are matching when quality is linear. We also show that LLM detection can lead to reduced output quality, even when the detected attributes are quality-reducing. In contrast, LLM detection has a clean impact on the metric that it optimizes, and our model reproduces an empirically observed ``rise-then-fall'' pattern. 

Our findings reveal key failure modes when current detection practices are used as an intervention to steer these downstream metrics. Specifically, our results show that detection through quality-reducing characteristics can induce the opposite effect on LLM usage and output quality that an institution may have intended. This effect persists even when the penalty becomes arbitrarily severe. On the flip side, our results suggest that detecting via quality-improving characteristics would ensure that detection weakly decreases LLM usage, although it would not remove the distortion in output quality. 

\paragraph{Model discussion and limitations.} Our model captures key aspects of user incentives under LLM detection, including detection quality effects, fractional work allocation, post-processing, and user heterogeneity. However, our model is stylized and makes several simplifications. For example, we assume that quality $\Quality$ and gaming costs $\CostFn$ are separable across dimensions. Moreover, we assume that users exactly know the detection boundary, following standard assumptions from strategic classification (e.g., \citep{HardtMPW16}). Finally, we assume that post-processing costs are identical for human-generated and LLM-generated content. An interesting direction for future work would be to relax these assumptions and incorporate further practical complexities into our model.

\section{Acknowledgments}
We would like to thank Luke Bailey and Sanmi Koyejo for useful feedback on the paper. MJ was partially supported by a SAIL postdoctoral fellowship. TH was supported by a grant by HAI, DSO labs, gifts from Open Philanthropy, Amazon, Schmidt Sciences, the Tianqiao and Chrissy Chen Foundation and a grant under the NSF CAREER IIS-2338866, ONR N00014-24-1-2609, and DARPA Cooperative Agreement HR00112520013. JK was supported in part by a Simons Collaboration grant, AFOSR grant FA9550-23-1-0410, and a grant from the MacArthur Foundation.
This work does not necessarily reflect the position or policy of the government and no official endorsement should be inferred. 

\bibliographystyle{plainnat}
\bibliography{bib.bib}

\appendix

\section{LLM usage}\label{appendix:llmusage}

Claude Opus 4.7 assisted us in writing code. GPT-5.4, Claude Opus 4.7, and Opus 4.6 were consulted to help gather related work, brainstorm proof ideas, and provide feedback on the writing and ideas in the paper. We also use LLM judges in our empirical analysis. All outputs were verified by the (human) authors of this paper. 

\section{Key objects and auxiliary lemmas}

The following work allocation objects emerge in our analysis. Let $\FractionNoArg^* \in [0,1]$ be the minimum optimizer of $\text{max}_{\FractionNoArg \in [0,1]} \left(\Quality(\Combination{\TypeNoArg}(\FractionNoArg)) - \ProductionCost{\TypeNoArg} \cdot (1 - \FractionNoArg) \right)$, which we know exists because this is a continuous function over a compact set. If $\HumanContent{\TypeNoArg}_1 \le \Threshold < \LLMContent{\TypeNoArg}_1$, we  let $\FractionNoArg^{\text{game}} \in [0,1]$ be the unique value where $\Combination{\TypeNoArg}_1(\FractionNoArg) = \Threshold$. Otherwise, we define $\FractionNoArg^{\text{game}} := \infty$. 

When gaming costs are finite, we also define several other objects.

\paragraph{The function $G$.} One key object is the function $G: \left\{z_1 \ge \Threshold: z_1 \in \mathbb{R}  \right\} \rightarrow \mathbb{R}$:
\[G(z_1) = \text{argmax}_{z'_1 \le \Threshold}\left(\Quality_1(z'_1) - \CostFn_1(z_1, z'_1)  \right), \]
where we tiebreak in favor of lower values. To see that $G(z_1)$ is well-defined, note that by (A5), we know that $Q_1(z'_1) - k_1(z_1, z'_1) \rightarrow -\infty$ as $z'_1 \rightarrow -\infty$. 
The function $G(z_1)$ captures the post-processing decision that the user will make, should they decide to post-process their content. 

\paragraph{The sets $S(\Penalty)$ and $A^{\TypeNoArg}(\Penalty)$.} Another key object is the set:
\[S(\Penalty) := \left\{ z_1  > \Threshold \mid \Quality_1(z_1) - \max_{z'_1 \le \Threshold} \left(\Quality_1(z'_1) - \CostFn_1(z_1, z'_1) \right) \le \Penalty \right\}.\]
which captures the values for which it is worthwhile for the user to post-process their content to avoid detection. At a high-level, the user must factor not only explicit gaming costs but also the change in quality from post-processing. We also define an additional type-dependent function $A^{\TypeNoArg}(\Penalty)$ that reparameterizes the set $S(\Penalty)$ into work allocation. We let: 
\[A^{\TypeNoArg}(\Penalty) := \left\{\FractionNoArg \in [0,1] \mid \Combination{\TypeNoArg}_1(\FractionNoArg) \in S(\Penalty)  \right\} \]
capture the set of work allocations that are in the post-processing region. 

\paragraph{Domain truncation property and implicit definition $V(w_1)$.} Finally, we also show the following domain truncation property which will enable us to apply Berge's maximum theorem. Suppose that we have an upper bound $w_1 > \Threshold$ on $z_1$. Then, we show we can also truncate the domain $z'_1$ as follows. By (A5), we know that $\sup_{z_1 \ge \Threshold} \left(\Quality_1(z'_1) - \CostFn_1(z_1, z'_1)\right) \rightarrow -\infty$ as $z'_1 \rightarrow -\infty$. We can thus truncate the domain to $z'_1 \ge V(w_1)$ where $V(w_1) \le \Threshold$ is defined to guarantee that:  
\begin{equation}
\label{eq:upperbound}
 \text{  for all } z'_1 < V(w_1): \sup_{z_1 \ge \Threshold} \left(\Quality_1(z'_1) - \CostFn_1(z_1, z'_1)\right) < \inf_{z_1 \in [\Threshold, w_1]} \left(\Quality_1(\Threshold) - \CostFn_1(z_1, \Threshold) \right).   
\end{equation}
This means that: 
\[\max_{z'_1 \le \Threshold} \left(\Quality_1(z'_1) - \CostFn_1(z_1, z'_1) \right)= \max_{z'_1 \in [V(w_1), \Threshold]} \left(\Quality_1(z'_1) - \CostFn_1(z_1, z'_1) \right) \]
for all $z_1 \in [\Threshold, w_1]$. Throughout the analysis, we will use the fact that we can apply Berge's maximum theorem to $\max_{z'_1 \in [V(w_1), \Threshold]} \left(\Quality_1(z'_1) - \CostFn_1(z_1, z'_1) \right)$ whenever we have an upper bound on $z_1$. For ease of notation, we will use this general fact without explicitly rederiving the threshold $V(w_1)$ on each instance. 

\subsection{Auxiliary lemmas}
We prove a useful lemma about the structure of $A^{\TypeNoArg}(\Penalty)$.
\begin{lemma}
\label{lemma:apenaltyset}
Fix $\Penalty > 0$, $\Threshold < \infty$, $\Quality$, and finite costs $\CostFn$. If $\bar{A}^{\TypeNoArg}(\Penalty)$ denotes the closure of $A^{\TypeNoArg}(\Penalty)$, then:
\[\bar{A}^{\TypeNoArg}(\Penalty) \setminus A^{\TypeNoArg}(\Penalty) \subseteq \left\{\FractionNoArg^{\text{game}}\right\}.\]
\end{lemma}
\begin{proof}
Using the definition of $S(\Penalty)$, we know that:
\[A^{\TypeNoArg}(\Penalty) = \left\{\FractionNoArg \in [0,1] \mid \Combination{\TypeNoArg}_1(\FractionNoArg) > \Threshold, \Quality_1(\Combination{\TypeNoArg}_1(\FractionNoArg)) - \max_{z'_1 \le \Threshold}  \left(\Quality_1(z'_1) - \CostFn_1(\Combination{\TypeNoArg}_1(\FractionNoArg), z'_1) \right) \le \Penalty \right\}.\]

Consider any $\FractionNoArg \in [0,1]$ such that $\FractionNoArg \not\in A^{\TypeNoArg}(\Penalty) \cup \left\{\FractionNoArg^{\text{game}}\right\}$. It suffices to show that there exists an open ball around $\FractionNoArg$ that does not intersect $A^{\TypeNoArg}(\Penalty)$. By definition, we know that either $\Combination{\TypeNoArg}_1(\FractionNoArg) \le \Threshold$ or $\Quality_1(\Combination{\TypeNoArg}_1(\FractionNoArg)) - \max_{z'_1 \le \Threshold}  \left(\Quality_1(z'_1) - \CostFn_1(\Combination{\TypeNoArg}_1(\FractionNoArg), z'_1) \right) > \Penalty$. We divide into these two cases.

\paragraph{Case 1: $\Combination{\TypeNoArg}_1(\FractionNoArg) \le \Threshold$.} Since we know by assumption $\FractionNoArg \neq \FractionNoArg^{\text{game}}$, we know that $\Combination{\TypeNoArg}_1(\FractionNoArg) < \Threshold$. If we take a sufficiently small open ball around $B_{\epsilon}(\FractionNoArg)$, we know that every $\FractionNoArg' \in B_{\epsilon}(\FractionNoArg)$ satisfies $\Combination{\TypeNoArg}_1(\FractionNoArg') < \Threshold$. This means that $B_{\epsilon}(\FractionNoArg) \cap A^{\TypeNoArg}(\Penalty) = \emptyset$ as desired. 

\paragraph{Case 2: $\Quality_1(\Combination{\TypeNoArg}_1(\FractionNoArg)) - \max_{z'_1 \le \Threshold}  \left(\Quality_1(z'_1) - \CostFn_1(\Combination{\TypeNoArg}_1(\FractionNoArg), z'_1) \right) > \Penalty$.} We know that the function $\Quality_1(\Combination{\TypeNoArg}_1(\FractionNoArg)) - \max_{z'_1 \le \Threshold}  \left(\Quality_1(z'_1) - \CostFn_1(\Combination{\TypeNoArg}_1(\FractionNoArg), z'_1) \right)$ is continuous in $\FractionNoArg$. Thus, the pre-image of the open set $(\Penalty, \infty)$ is open. This means that the set 
\[\left\{\FractionNoArg \in [0,1] \mid \Quality_1(\Combination{\TypeNoArg}_1(\FractionNoArg)) - \max_{z'_1 \le \Threshold}  \left(\Quality_1(z'_1) - \CostFn_1(\Combination{\TypeNoArg}_1(\FractionNoArg), z'_1) \right) > \Penalty \right\}\] is open as desired.     
\end{proof}

We prove the following structural property of $G$.
\begin{lemma}
\label{lemma:gthreshold}
Fix finite costs $\CostFn$ and suppose that $\Quality_1$ is weakly increasing. Then, it holds that $G(z_1) = \Threshold$ for all $z_1 \ge \Threshold$.
\end{lemma}
\begin{proof}
This holds because $\CostFn_1(z_1, z'_1)$ is weakly decreasing in $z'_1$ and  $\Quality_1(z'_1)$ is weakly increasing in $z'_1$. 
\end{proof}

We prove the following structural property of $G$.
\begin{lemma}
\label{lemma:gthresholdmixed}
Fix finite costs $\CostFn$ such that $\frac{\partial^2 \CostFn_1 (z_1, z'_1)}{\partial z_1 \partial z'_1} \le 0$ for all $z'_1 \le z_1$. Then, it holds that $G(z_1) = \Threshold$ for all $z_1 \ge \Threshold$.
\end{lemma}
\begin{proof}
The case of $z_1 = \Threshold$ follows immediately from (A1). We consider $z_1 > \Threshold$ for the remainder of the analysis. 

It suffices to show that $\Quality_1(z'_1) - \CostFn_1(z_1, z'_1) < \Quality_1(\Threshold) - \CostFn_1(z_1, \Threshold)$ for all $z'_1 < \Threshold$, which we can rewrite as 
\[\Quality_1(z'_1) - \Quality_1(\Threshold) < \CostFn_1(z_1, z'_1) - \CostFn_1(z_1, \Threshold).\]  Note that by (A1), we know that $\Quality_1(z'_1) - \Quality_1(\Threshold) < \CostFn_1(\Threshold, z'_1)$. This means that it suffices to show that:
\[\CostFn_1(\Threshold, z'_1) \le \CostFn_1(z_1, z'_1) - \CostFn_1(z_1, \Threshold). \]
Using (A3), this can be rewritten as:
\[\CostFn_1(\Threshold, z'_1) - \CostFn_1(\Threshold, \Threshold) \le \CostFn_1(z_1, z'_1) - \CostFn_1(z_1, \Threshold), \]
which can be rewritten as:
\[ -\int_{z'_1}^{\Threshold} \nabla_2 \CostFn_1(\Threshold, y) dy \le -\int_{z'_1}^{\Threshold} \nabla_2 \CostFn_1(z_1, y) dy, \]
which can be rewritten as:
\[ \int_{z'_1}^{\Threshold} (\nabla_2 \CostFn_1(z_1, y)  - \nabla_2 \CostFn_1(\Threshold, y)) dy \le 0.\]
We know this holds by the assumption that $\frac{\partial^2 \CostFn_1(z_1, z'_1)}{\partial z_1 \partial z'_1} \le 0$.
\end{proof}

We prove that $\max_{z' \le \Threshold}(\Quality_1(z'_1) - \CostFn_1(z_1, z'_1))$ is weakly decreasing in $z_1$.
\begin{lemma}
\label{lemma:maxweaklydecreasing}
Fix finite costs $\CostFn$. For any $\Threshold$, it holds that $\max_{z' \le \Threshold}(\Quality_1(z'_1) - \CostFn_1(z_1, z'_1))$ is weakly decreasing in $z_1$ for $z_1 \ge \Threshold$.
\end{lemma}
\begin{proof}
For any $z'_1 \le z_1 \le z''_1$, it holds that $(\Quality_1(z'_1) - \CostFn_1(z_1, z'_1)) \ge (\Quality_1(z'_1) - \CostFn_1(z''_1, z'_1))$ by (A2). This means that $\max_{z' \le \Threshold}(\Quality_1(z'_1) - \CostFn_1(z_1, z'_1)) \ge \max_{z' \le \Threshold}(\Quality_1(z'_1) - \CostFn_1(z''_1, z'_1))$ as desired. 
\end{proof}

\section{Characterization of user workflow decisions}\label{appendix:characerization}

In this section, we characterize user workflow decisions by analyzing the utility-optimization program $\text{max}_{\ContentNoArg} \Utility{\TypeNoArg}(\FractionNoArg,\FinalContentNoArg; \Penalty, \Threshold, \Quality, \CostFn)$ (and showing it is well-defined).

\paragraph{Outline.} We perform this analysis for three cases: no detector (Appendix \ref{appendix:nodetector}), detector with infinite gaming costs (Appendix \ref{appendix:infinite}), and detector with finite gaming costs (Appendix \ref{appendix:finite}). We characterize properties of these user workflow decisions. For each case, we analyze the utility-optimization program in two stages: (1) for each fixed value of $\FractionNoArg$, we analyze the indirect utility $\text{max}_{\ContentNoArg} \Utility{\TypeNoArg}(\FractionNoArg,\FinalContentNoArg; \Penalty, \Threshold, \Quality, \CostFn)$, and (2) we analyze the value of $\FractionNoArg$ that maximizes $\max_{\FractionNoArg \in [0,1]} \text{max}_{\ContentNoArg} \Utility{\TypeNoArg}(\FractionNoArg,\FinalContentNoArg; \Penalty, \Threshold, \Quality, \CostFn)$.

\subsection{No detector}\label{appendix:nodetector}

The following result characterizes user workflow decisions for the case of no detector.
\begin{theorem}
\label{thm:nodetector}
Fix $\Quality$ and $\CostFn$. The optimization program $\text{max}_{\FractionNoArg \in [0,1], \ContentNoArg} \Utility{\TypeNoArg}(\FractionNoArg,\FinalContentNoArg; \emptyset, \Quality, \CostFn)$ is well-defined, and the optimal value is given by:
\[\text{max}_{\FractionNoArg \in [0,1], \ContentNoArg} \Utility{\TypeNoArg}(\FractionNoArg,\FinalContentNoArg; \emptyset, \Quality, \CostFn) = \max_{\FractionNoArg \in [0,1]} \left(\Quality(\Combination{\TypeNoArg}(\FractionNoArg)) - \ProductionCost{\TypeNoArg}\cdot (1 - \FractionNoArg) \right). \]
Moreover, it holds that $\Fraction{\TypeNoArg}(\emptyset; \Quality, \CostFn) = \FractionNoArg^*$, and $\FinalContent{\TypeNoArg}(\emptyset; \Quality, \CostFn) = \Combination{\TypeNoArg}(\FractionNoArg^*)$.
\end{theorem}

To prove Theorem \ref{thm:nodetector}, we first analyze the indirect utility function. We show this is well-defined and that users will never post-process their content. 
\begin{lemma}
\label{lemma:nodetectionnopostprocessing} 
Fix $\Quality$ and $\CostFn$. For any $\FractionNoArg \in [0,1]$, the optimization program for the indirect utility $\max_{\FinalContentNoArg \in \mathbb{R}^D}(\Utility{\TypeNoArg}(\FractionNoArg, \ContentNoArg;\emptyset, \Quality, \CostFn))$ is well-defined and has a unique optimum at $\FinalContentNoArg^* = \Combination{\TypeNoArg}(\FractionNoArg)$. The optimal value is given by: 
\[\max_{\FinalContentNoArg \in \mathbb{R}^D}(\Utility{\TypeNoArg}(\FractionNoArg, \ContentNoArg;\emptyset, \Quality, \CostFn)) = \Quality(\Combination{\TypeNoArg}(\FractionNoArg)) - \ProductionCost{\TypeNoArg} \cdot (1 - \FractionNoArg). \]
\end{lemma}
\begin{proof}
For any value of $\FinalContentNoArg$, we know that:
\begin{align*}
\Utility{\TypeNoArg}(\FractionNoArg,\FinalContentNoArg; \emptyset, \Quality, \CostFn) &= \Quality(\FinalContentNoArg) -\ProductionCost{\TypeNoArg} \cdot (1-\FractionNoArg) - \CostFn(\Combination{\TypeNoArg}(\FractionNoArg), \FinalContentNoArg) \\
&= -\ProductionCost{\TypeNoArg} \cdot (1-\FractionNoArg) 
+ \sum_{i=1}^D \left(\Quality_{i}(\FinalContentNoArg) - \CostFn_{i}(\Combination{\TypeNoArg}_{i}(\FractionNoArg), \FinalContentNoArg_{i})\right)\\
&\le_{(1)} -\ProductionCost{\TypeNoArg} \cdot (1-\FractionNoArg) 
+ \sum_{i=1}^D \Quality_i(\Combination{\TypeNoArg}_{i}(\FractionNoArg))\\
&=\Quality(\Combination{\TypeNoArg}(\FractionNoArg)) -\ProductionCost{\TypeNoArg} \cdot (1-\FractionNoArg) \\
&= \Utility{\TypeNoArg}(\FractionNoArg,\Combination{\TypeNoArg}(\FractionNoArg); \emptyset, \Quality, \CostFn).
\end{align*}
where (1) follows from assumption (A1). 

This implies that $\Combination{\TypeNoArg}(\FractionNoArg)$ is an optimizer of $\max_{\FinalContentNoArg}(\Utility{\TypeNoArg}(\FractionNoArg, \ContentNoArg;\emptyset, \Quality, \CostFn)$. We also see that the optimal value is given by:
\[\Utility{\TypeNoArg}(\FractionNoArg,\Combination{\TypeNoArg}(\FractionNoArg); \emptyset, \Quality, \CostFn) = \Quality(\Combination{\TypeNoArg}(\FractionNoArg)) -\ProductionCost{\TypeNoArg} \cdot (1-\FractionNoArg). \]

Now, we show that $\Combination{\TypeNoArg}(\FractionNoArg)$ is the unique optimizer. If $\FinalContentNoArg^*$ is an optimizer of $\max_{\FinalContentNoArg}(\Utility{\TypeNoArg}(\FractionNoArg, \ContentNoArg;\emptyset, \Quality, \CostFn)$ of $\max_{\FinalContentNoArg}(\Utility{\TypeNoArg}(\FractionNoArg, \FinalContentNoArg;\Penalty, \Threshold, \Quality, \CostFn)$, then we know that equality must hold in the above chain of equations. This means that $\sum_{i=1}^D \left(\Quality_{i}(\FinalContentNoArg^*) - \CostFn_{i}(\Combination{\TypeNoArg}_{i}(\FractionNoArg), \FinalContentNoArg^*_{i})\right) = \sum_{i=1}^D \Quality_i(\Combination{\TypeNoArg}_{i}(\FractionNoArg))$. Applying (A1), this means that $\FinalContentNoArg^*_i =  \Combination{\TypeNoArg}_i(\FractionNoArg)$ for all $1 \le i \le D$.

\end{proof}

We prove Theorem \ref{thm:nodetector}. 
\begin{proof}[Proof of Theorem \ref{thm:nodetector}]
Using Lemma \ref{lemma:nodetectionnopostprocessing}, we know that for any $\FractionNoArg \in [0,1]$, it holds that:
\[\text{max}_{\ContentNoArg} \Utility{\TypeNoArg}(\FractionNoArg,\FinalContentNoArg; \emptyset, \Quality, \CostFn) = \Quality(\Combination{\TypeNoArg}(\FractionNoArg)) - \ProductionCost{\TypeNoArg}\cdot (1 - \FractionNoArg).\]
This function is continuous in $\FractionNoArg$, and the optimization is over a compact set $[0,1]$, meaning that it achieves its optimum. 

We know that $\Fraction{\TypeNoArg}(\emptyset; \Quality, \CostFn) = \FractionNoArg^*$ since both are tiebroken in favor of lower values by definition. The fact that $\FinalContent{\TypeNoArg}(\emptyset; \Quality, \CostFn) = \Combination{\TypeNoArg}(\Fraction{\TypeNoArg}(\emptyset; \Quality, \CostFn))$ now follows from Lemma \ref{lemma:nodetectionnopostprocessing}.
\end{proof}

\subsection{Detector in the case of infinite gaming costs}\label{appendix:infinite}

The following result characterizes user workflow decisions in the case of infinite gaming costs.
\begin{theorem}
\label{thm:infinitecostdetector}
Fix $\Penalty > 0$, $\Threshold < \infty$, and $\Quality$. Suppose that gaming costs $\CostFn = \CostFn^{\infty}$ are infinite. The optimization program $\text{max}_{\FractionNoArg \in [0,1], \ContentNoArg} \Utility{\TypeNoArg}(\FractionNoArg,\FinalContentNoArg; \Penalty, \Threshold, \Quality, \CostFn^{\infty})$ is well-defined. Moreover, if $\Combination{\TypeNoArg}_1(\FractionNoArg^*) \le \Threshold$, it holds that $\Fraction{\TypeNoArg}(\Penalty, \Threshold; \Quality, \CostFn^{\infty}) = \FractionNoArg^*$; otherwise, it holds that 
$\Fraction{\TypeNoArg}(\Penalty, \Threshold; \Quality, \CostFn^{\infty}) \in \left\{\FractionNoArg^*, \FractionNoArg^{\text{game}}\right\}$. Finally, it holds that $\FinalContent{\TypeNoArg}(\Penalty, \Threshold; \Quality, \CostFn^{\infty}) = \Combination{\TypeNoArg}(\Fraction{\TypeNoArg}(\Penalty, \Threshold; \Quality, \CostFn^{\infty}))$. 
\end{theorem}

To prove Theorem \ref{thm:infinitecostdetector}, 
we analyze the indirect utility function. We show this is well-defined and that users will never post-process their content. 
\begin{lemma}
\label{lemma:infinitecostspostprocessing}
Fix $\Penalty > 0$, $\Threshold < \infty$, and $\Quality$. Suppose that gaming costs $\CostFn^{\infty}$ are infinite. For any $\FractionNoArg \in [0,1]$, the optimization program for the indirect utility $\text{max}_{\ContentNoArg} \Utility{\TypeNoArg}(\FractionNoArg,\FinalContentNoArg; \Penalty, \Threshold, \Quality, \CostFn^{\infty})$ is well-defined, and the optimum is uniquely achieved at $\FinalContentNoArg^* = \Combination{\TypeNoArg}(\FractionNoArg)$. The optimal  value is given by: 
\[
\text{max}_{\ContentNoArg} \Utility{\TypeNoArg}(\FractionNoArg,\FinalContentNoArg; \Penalty, \Threshold, \Quality, \CostFn^{\infty}) = 
\Quality(\Combination{\TypeNoArg}(\FractionNoArg)) - \Penalty \cdot 1[\Combination{\TypeNoArg}_1(\FractionNoArg) > \Threshold] - \ProductionCost{\TypeNoArg} \cdot (1-\FractionNoArg).\]
\end{lemma}
\begin{proof}
First, for any value of $\FractionNoArg$, we claim that $\max_{\FinalContentNoArg}(\Utility{\TypeNoArg}(\FractionNoArg, \FinalContentNoArg; \Penalty, \Threshold, \Quality, \CostFn^{\infty}))$ is uniquely maximized at $\FinalContentNoArg^* = \Combination{\TypeNoArg}(\FractionNoArg)$. For any value of $\FinalContentNoArg$, note that:
\[\Utility{\TypeNoArg}(\FractionNoArg,\FinalContentNoArg;\Penalty, \Threshold, \Quality, \CostFn^{\infty}) := \Quality(\FinalContentNoArg) - \Penalty \cdot 1[\FinalContentNoArg_1 > \Threshold] - \ProductionCost{\TypeNoArg} \cdot (1-\FractionNoArg) - \CostFn^{\infty}(\Combination{\TypeNoArg}(\FractionNoArg), \FinalContentNoArg).   \]
This expression is equal to $-\infty$ if $\FinalContentNoArg \neq \Combination{\TypeNoArg}(\FractionNoArg)$, and it is finite when $\FinalContentNoArg = \Combination{\TypeNoArg}(\FractionNoArg)$. This proves the desired claim. 

This implies that
\[
\text{max}_{\ContentNoArg} \Utility{\TypeNoArg}(\FractionNoArg,\FinalContentNoArg; \Penalty, \Threshold, \Quality, \CostFn^{\infty}) = 
\Quality(\Combination{\TypeNoArg}(\FractionNoArg)) - \Penalty \cdot 1[\Combination{\TypeNoArg}_1(\FractionNoArg) > \Threshold] - \ProductionCost{\TypeNoArg} \cdot (1-\FractionNoArg)\]
as desired. 
\end{proof}

We prove Theorem \ref{thm:infinitecostdetector}.
\begin{proof}[Proof of Theorem \ref{thm:infinitecostdetector}]
By Lemma \ref{lemma:infinitecostspostprocessing}, we know that for any $\FractionNoArg \in [0,1]$ it holds that:
\[\text{max}_{\ContentNoArg} \Utility{\TypeNoArg}(\FractionNoArg,\FinalContentNoArg; \Penalty, \Threshold, \Quality, \CostFn) = \Quality(\Combination{\TypeNoArg}(\FractionNoArg)) - \Penalty \cdot 1[\Combination{\TypeNoArg}_1(\FractionNoArg) > \Threshold] - \ProductionCost{\TypeNoArg} \cdot (1-\FractionNoArg). \]
To optimize over $\FractionNoArg$, we split into cases.

\paragraph{Case 1: $\HumanContent{\TypeNoArg}_1 < \LLMContent{\TypeNoArg}_1 \le  \Threshold$.} This means that $\Combination{\TypeNoArg}_1(\FractionNoArg)) \le \Threshold$ for all $\FractionNoArg \in [0,1]$. Then we know that $\text{max}_{\ContentNoArg} \Utility{\TypeNoArg}(\FractionNoArg,\FinalContentNoArg; \Penalty, \Threshold, \Quality, \CostFn^{\infty}) = \Quality(\Combination{\TypeNoArg}(\FractionNoArg))  - \ProductionCost{\TypeNoArg} \cdot (1 - \FractionNoArg)$ for all $\FractionNoArg \in [0,1]$. The
function $\Quality(\Combination{\TypeNoArg}(\FractionNoArg)) - \ProductionCost{\TypeNoArg} \cdot (1 - \FractionNoArg)$ is a continuous function of $\FractionNoArg$ on a compact set $[0,1]$, meaning that $\text{max}_{\FractionNoArg \in [0,1], \ContentNoArg} \Utility{\TypeNoArg}(\FractionNoArg,\FinalContentNoArg; \Penalty, \Threshold, \Quality, \CostFn^{\infty})$ is well-defined. Using the tie-breaking rule specified in Section \ref{subsec:userdecisions}, it holds that  $\Fraction{\TypeNoArg}(\Penalty, \Threshold; \Quality, \CostFn^{\infty}) = \FractionNoArg^*$. The fact that $\FinalContent{\TypeNoArg}(\Penalty, \Threshold; \Quality, \CostFn^{\infty}) = \Combination{\TypeNoArg}(\Fraction{\TypeNoArg}(\Penalty, \Threshold; \Quality, \CostFn^{\infty}))$ follows from Lemma \ref{lemma:infinitecostspostprocessing}. 

\paragraph{Case 2: $\Threshold < \HumanContent{\TypeNoArg}_1 < \LLMContent{\TypeNoArg}_1$.}
This means that $\Combination{\TypeNoArg}_1(\FractionNoArg)) > \Threshold$ for all $\FractionNoArg \in [0,1]$. Then we know that $\text{max}_{\ContentNoArg} \Utility{\TypeNoArg}(\FractionNoArg,\FinalContentNoArg; \Penalty, \Threshold,\Quality, \CostFn^{\infty}) = \Quality(\Combination{\TypeNoArg}(\FractionNoArg)) - \ProductionCost{\TypeNoArg} \cdot (1 - \FractionNoArg) - \Penalty$ for all $\FractionNoArg \in [0,1]$. By an analogous argument, we know that $\text{max}_{\FractionNoArg \in [0,1], \ContentNoArg} \Utility{\TypeNoArg}(\FractionNoArg,\FinalContentNoArg;\Penalty, \Threshold,\Quality, \CostFn^{\infty})$ is well-defined. The optimizer of $\Quality(\Combination{\TypeNoArg}(\FractionNoArg)) - \ProductionCost{\TypeNoArg} \cdot (1 - \FractionNoArg) - \Penalty$ is the same as the optimizer of $\Quality(\Combination{\TypeNoArg}(\FractionNoArg)) - \ProductionCost{\TypeNoArg} \cdot (1 - \FractionNoArg)$, since these functions are just off by a constant shift. Using the tie-breaking rule specified in Section \ref{subsec:userdecisions}, it thus holds that  $\Fraction{\TypeNoArg}(\Penalty, \Threshold; \Quality, \CostFn) = \FractionNoArg^*$. The fact that $\FinalContent{\TypeNoArg}(\Penalty, \Threshold; \Quality, \CostFn^{\infty}) = \Combination{\TypeNoArg}(\Fraction{\TypeNoArg}(\Penalty, \Threshold; \Quality, \CostFn^{\infty}))$ follows from Lemma \ref{lemma:infinitecostspostprocessing}. 

\paragraph{Case 3: $\HumanContent{\TypeNoArg}_1 \le \Threshold < \LLMContent{\TypeNoArg}_1$.} In this case, the value $\FractionNoArg^{\text{game}}$ denotes the unique value such that $\Combination{\TypeNoArg}_1(\FractionNoArg^{\text{game}})) = \Threshold$. We further know that $\Combination{\TypeNoArg}_1(\FractionNoArg)) < \Threshold$ for $\FractionNoArg < \FractionNoArg^{\text{game}} $ and $\Combination{\TypeNoArg}_1(\FractionNoArg)) > \Threshold$ for $\FractionNoArg > \FractionNoArg^{\text{game}} $. This means that:
\[\text{max}_{\ContentNoArg} \Utility{\TypeNoArg}(\FractionNoArg,\FinalContentNoArg; \Penalty, \Threshold, \Quality, \CostFn^{\infty}) = 
\begin{cases}
\Quality(\Combination{\TypeNoArg}(\FractionNoArg)) - \ProductionCost{\TypeNoArg} \cdot (1 - \FractionNoArg)  &\text{ if } 0 \le \FractionNoArg \le \FractionNoArg^{\text{game}} \\
\Quality(\Combination{\TypeNoArg}(\FractionNoArg)) - \ProductionCost{\TypeNoArg} \cdot (1 - \FractionNoArg) - \Penalty &\text{ if } \FractionNoArg^{\text{game}} < \FractionNoArg \le 1.
\end{cases}.\]

First, we show that $\text{max}_{\FractionNoArg \in [0,1], \ContentNoArg} \Utility{\TypeNoArg}(\FractionNoArg,\FinalContentNoArg; \Penalty, \Threshold, \Quality, \CostFn)$ is well-defined. A useful intermediate step is to show that  
\[\text{max}_{\FractionNoArg \in [0,1], \ContentNoArg} \Utility{\TypeNoArg}(\FractionNoArg,\FinalContentNoArg; \Penalty, \Threshold, \Quality, \CostFn) \]
\[= \max\left(\max_{\FractionNoArg \in [0, \FractionNoArg^{\text{game}}]} \left(\Quality(\Combination{\TypeNoArg}(\FractionNoArg)) - \ProductionCost{\TypeNoArg} \cdot (1 - \FractionNoArg)\right), \max_{\FractionNoArg \in [\FractionNoArg^{\text{game}}, 1]} \left(\Quality(\Combination{\TypeNoArg}(\FractionNoArg)) - \ProductionCost{\TypeNoArg} \cdot (1 - \FractionNoArg) - \Penalty\right) \right). \]
This almost follows from the equation above; the only aspect we need to verify is that we can replace $(\FractionNoArg^{\text{game}}, 1]$ by $[\FractionNoArg^{\text{game}}, 1]$ in the second term. If $\max_{\FractionNoArg \in [\FractionNoArg^{\text{game}}, 1]} \left(\Quality(\Combination{\TypeNoArg}(\FractionNoArg)) - \ProductionCost{\TypeNoArg} \cdot (1 - \FractionNoArg) - \Penalty\right)$ is achieved for $\FractionNoArg \in (\FractionNoArg^{\text{game}}, 1]$, then we are done. If not, then we know that: 
\begin{align*}
\sup_{\FractionNoArg \in (\FractionNoArg^{\text{game}}, 1]} \left(\Quality(\Combination{\TypeNoArg}(\FractionNoArg)) - \ProductionCost{\TypeNoArg} \cdot (1 - \FractionNoArg) - \Penalty \right) &= \Quality(\Combination{\TypeNoArg}(\FractionNoArg^{\text{game}})) - \ProductionCost{\TypeNoArg} \cdot (1 - \FractionNoArg^{\text{game}}) - \Penalty \\
&< \Quality(\Combination{\TypeNoArg}(\FractionNoArg^{\text{game}})) - \ProductionCost{\TypeNoArg} \cdot (1 - \FractionNoArg^{\text{game}}) \\ &\le \text{max}_{\FractionNoArg \in [0, \FractionNoArg^{\text{game}}]} \left(\Quality(\Combination{\TypeNoArg}(\FractionNoArg)) - \ProductionCost{\TypeNoArg} \cdot (1 - \FractionNoArg)\right),
\end{align*}
so the first branch dominates anyway. This shows the intermediate step. Using the intermediate step, we show that $\text{max}_{\FractionNoArg \in [0, \FractionNoArg^{\text{game}}]} \left(\Quality(\Combination{\TypeNoArg}(\FractionNoArg)) - \ProductionCost{\TypeNoArg} \cdot (1 - \FractionNoArg)\right)$ is well-defined. Note that the
functions $\Quality(\Combination{\TypeNoArg}(\FractionNoArg)) - \ProductionCost{\TypeNoArg} \cdot (1 - \FractionNoArg)$ and $\Quality(\Combination{\TypeNoArg}(\FractionNoArg)) - \ProductionCost{\TypeNoArg} \cdot (1 - \FractionNoArg) - \Penalty$ are continuous functions of $\FractionNoArg$, and in the intermediate step, we optimize these functions  over compact sets. 

We now claim that $\Fraction{\TypeNoArg}(\Penalty, \Threshold; \Quality, \CostFn) \in \left\{ \FractionNoArg^*, \FractionNoArg^{\text{game}}\right\}$ always, and $\Fraction{\TypeNoArg}(\Penalty, \Threshold; \Quality, \CostFn)  = \FractionNoArg^*$ if $\FractionNoArg^* \le \FractionNoArg^{\text{game}}$.  Using the characterization
\[\text{max}_{\ContentNoArg} \Utility{\TypeNoArg}(\FractionNoArg,\FinalContentNoArg; \Penalty, \Threshold, \Quality, \CostFn) = 
\begin{cases}
\Quality(\Combination{\TypeNoArg}(\FractionNoArg)) - \ProductionCost{\TypeNoArg} \cdot (1 - \FractionNoArg)  &\text{ if } 0 \le \FractionNoArg \le \FractionNoArg^{\text{game}} \\
\Quality(\Combination{\TypeNoArg}(\FractionNoArg)) - \ProductionCost{\TypeNoArg} \cdot (1 - \FractionNoArg) - \Penalty &\text{ if } \FractionNoArg^{\text{game}} < \FractionNoArg \le 1.
\end{cases},\]
we know that either (1) $\Fraction{\TypeNoArg}(\Penalty, \Threshold; \Quality, \CostFn)$ lies on the boundary $\left\{0, 1, \FractionNoArg^{\text{game}}\right\}$, or (2) $\Fraction{\TypeNoArg}(\Penalty, \Threshold; \Quality, \CostFn)$ does not lie on the boundary $\left\{0, 1, \FractionNoArg^{\text{game}}\right\}$ and the derivative of $\Quality(\Combination{\TypeNoArg}(\FractionNoArg)) - \ProductionCost{\TypeNoArg} \cdot (1 - \FractionNoArg) - \Penalty$ or $\Quality(\Combination{\TypeNoArg}(\FractionNoArg)) - \ProductionCost{\TypeNoArg} \cdot (1 - \FractionNoArg)$ is $0$ at $\Fraction{\TypeNoArg}(\Penalty, \Threshold; \Quality, \CostFn)$. We handle the cases separately.   
\begin{itemize}[leftmargin=*]
    \item For (1), to show that $\Fraction{\TypeNoArg}(\Penalty, \Threshold; \Quality, \CostFn) \in \left\{ \FractionNoArg^*, \FractionNoArg^{\text{game}}\right\}$, it suffices to show that $\Fraction{\TypeNoArg}(\Penalty, \Threshold; \Quality, \CostFn) \in \left\{ 0, 1\right\} \setminus \left\{\FractionNoArg^{\text{game}}\right\}$ implies that $\Fraction{\TypeNoArg}(\Penalty, \Threshold; \Quality, \CostFn) = \FractionNoArg^*$. If $\Fraction{\TypeNoArg}(\Penalty, \Threshold; \Quality, \CostFn) = 0$ and $\Fraction{\TypeNoArg}(\Penalty, \Threshold; \Quality, \CostFn) \neq \FractionNoArg^{\text{game}}$, then the right derivative of $\Quality(\Combination{\TypeNoArg}(\FractionNoArg)) - \ProductionCost{\TypeNoArg} \cdot (1 - \FractionNoArg)$ at $\FractionNoArg = 0$ must be non-positive, and by concavity, this means that $\Quality(\Combination{\TypeNoArg}(\FractionNoArg)) - \ProductionCost{\TypeNoArg} \cdot (1 - \FractionNoArg)$ is non-increasing, so $\FractionNoArg^* = 0$ as well. If $\Fraction{\TypeNoArg}(\Penalty, \Threshold; \Quality, \CostFn) = 1$ and $\Fraction{\TypeNoArg}(\Penalty, \Threshold; \Quality, \CostFn) \neq \FractionNoArg^{\text{game}}$, then the left derivative of $\Quality(\Combination{\TypeNoArg}(\FractionNoArg)) - \ProductionCost{\TypeNoArg} \cdot (1 - \FractionNoArg)$ at $\FractionNoArg = 1$ must be non-negative, and by concavity, this means that $\Quality(\Combination{\TypeNoArg}(\FractionNoArg)) - \ProductionCost{\TypeNoArg} \cdot (1 - \FractionNoArg)$ is non-decreasing. Moreover, by the tie-breaking rule, it must hold that the inequality $\Quality(\Combination{\TypeNoArg}(\FractionNoArg)) - \ProductionCost{\TypeNoArg} \cdot (1 - \FractionNoArg) < \Quality(\LLMContent{\TypeNoArg})$ is strict for sufficiently large $\FractionNoArg < 1$. Using concavity again, this means that $\Quality(\Combination{\TypeNoArg}(\FractionNoArg)) - \ProductionCost{\TypeNoArg} \cdot (1 - \FractionNoArg)$ is strictly increasing. Thus, it holds that $\FractionNoArg^* = 1$ as well. This proves the desired claim. 

    We further show that if $\FractionNoArg^* \le \FractionNoArg^{\text{game}}$, then it holds that $\Fraction{\TypeNoArg}(\Penalty, \Threshold; \Quality, \CostFn)  = \FractionNoArg^*$. This follows from the fact that $\max_{\FinalContentNoArg \in \mathbb{R}^D} (\Utility{\TypeNoArg}(\FractionNoArg, \FinalContentNoArg; \Penalty, \Threshold, \Quality, \CostFn))$ is point-wise weakly dominated by the function $\Quality(\Combination{\TypeNoArg}(\FractionNoArg)) - \ProductionCost{\TypeNoArg}\cdot (1-\FractionNoArg)$ with equality at $\FractionNoArg^*$, and tiebreaking favors lower values of the work allocation. 
    \item For (2), we use the concavity of $\Quality$, and the fact that the derivative of $\Quality(\Combination{\TypeNoArg}(\FractionNoArg)) - \ProductionCost{\TypeNoArg} \cdot (1 - \FractionNoArg) - \Penalty$ is the same as the derivative of $\Quality(\Combination{\TypeNoArg}(\FractionNoArg)) - \ProductionCost{\TypeNoArg} \cdot (1 - \FractionNoArg)$ everywhere. If $\Fraction{\TypeNoArg}(\Penalty, \Threshold; \Quality, \CostFn) = \FractionNoArg^*$, we are done. Assume for sake of contradiction that $\Fraction{\TypeNoArg}(\Penalty, \Threshold; \Quality, \CostFn) \neq \FractionNoArg^*$. Then, we would know that the derivative of $\Quality(\Combination{\TypeNoArg}(\FractionNoArg)) - \ProductionCost{\TypeNoArg} \cdot (1 - \FractionNoArg) - \Penalty$ is $0$ both at $\Fraction{\TypeNoArg}(\Penalty, \Threshold; \Quality, \CostFn)$ and at $\FractionNoArg^*$. However, by the concavity of $\Quality$, this means that the derivative of $\Quality(\Combination{\TypeNoArg}(\FractionNoArg)) - \ProductionCost{\TypeNoArg} \cdot (1 - \FractionNoArg) - \Penalty$ is $0$ along the full interval between $\Fraction{\TypeNoArg}(\Penalty, \Threshold; \Quality, \CostFn)$ and $\FractionNoArg^*$. By assumption, there is a local ball around $\Fraction{\TypeNoArg}(\Penalty, \Threshold; \Quality, \CostFn)$ that does not intersect $\left\{0, 1, \FractionNoArg^{\text{game}}\right\}$, which means that slightly reducing $\Fraction{\TypeNoArg}(\Penalty, \Threshold; \Quality, \CostFn)$ would yield the same utility but be lower, which is a contradiction given the tiebreaking rule. 
\end{itemize}
The fact that $\FinalContent{\TypeNoArg}(\Penalty, \Threshold; \Quality, \CostFn) = \Combination{\TypeNoArg}(\Fraction{\TypeNoArg}(\Penalty, \Threshold; \Quality, \CostFn))$ follows from Lemma \ref{lemma:infinitecostspostprocessing}. 
\end{proof}

\subsection{Detector in the case of finite gaming costs}\label{appendix:finite}

The following result characterizes user workflow decisions in the case of finite gaming costs. 
\begin{theorem}
\label{thm:detectorfinitecosts}
Fix $\Penalty > 0$, $\Threshold < \infty$, $\Quality$, and finite costs $\CostFn$. The  program $\text{max}_{\FractionNoArg \in [0,1], \ContentNoArg} \Utility{\TypeNoArg}(\FractionNoArg,\FinalContentNoArg; \Penalty, \Threshold, \Quality, \CostFn)$ is well-defined.  Moreover, if $\Combination{\TypeNoArg}_1(\FractionNoArg^*) \le \Threshold$, it holds that $\Fraction{\TypeNoArg}(\Penalty, \Threshold; \Quality, \CostFn) = \FractionNoArg^*$; otherwise, it holds that $\Fraction{\TypeNoArg}(\Penalty, \Threshold; \Quality, \CostFn) \in \left\{\FractionNoArg^*, \FractionNoArg^{\text{game}}\right\} \cup  A^{\TypeNoArg}(\Penalty)$. Finally, it holds that $\FinalContent{\TypeNoArg}_i(\Penalty, \Threshold; \Quality, \CostFn) = \Combination{\TypeNoArg}_i(\Fraction{\TypeNoArg}(\Penalty, \Threshold; \Quality, \CostFn))$ for $i \ge 2$, and it holds that: 
\[ 
\FinalContent{\TypeNoArg}_1(\Penalty, \Threshold; \Quality, \CostFn) = 
\begin{cases}
G(\Combination{\TypeNoArg}_1(\Fraction{\TypeNoArg}(\Penalty, \Threshold; \Quality, \CostFn))) &\text{ if  } \Fraction{\TypeNoArg}(\Penalty, \Threshold; \Quality, \CostFn) \in A^{\TypeNoArg}(\Penalty) \\
\Combination{\TypeNoArg}_1(\Fraction{\TypeNoArg}(\Penalty, \Threshold; \Quality, \CostFn)) & \text{ else. }
\end{cases}
\]
\end{theorem}

To prove Theorem \ref{thm:detectorfinitecosts}, 
we first analyze the indirect utility. We show the optimization program for the indirect utility is well-defined, users will never post-process their content along dimensions that are not used in the detection classifier, and users post-process in a structured manner. 

\begin{lemma}
\label{lemma:detectorfinitecostspostprocessing} 
Fix $\Penalty > 0$, $\Threshold < \infty$, $\Quality$, and finite costs $\CostFn$. For any $\FractionNoArg \in [0,1]$, the optimization program for the indirect utility $\max_{\FinalContentNoArg \in \mathbb{R}^D}(\Utility{\TypeNoArg}(\FractionNoArg, \ContentNoArg;\Penalty, \Threshold, \Quality, \CostFn))$ is well-defined, and the optimal value $\max_{\FinalContentNoArg \in \mathbb{R}^D}(\Utility{\TypeNoArg}(\FractionNoArg, \ContentNoArg;\Penalty, \Threshold, \Quality, \CostFn))$ is given by: 
\[\begin{cases}
\Quality(\Combination{\TypeNoArg}(\FractionNoArg)) - \ProductionCost{\TypeNoArg}\cdot (1-\FractionNoArg) &\text{ if } (\Combination{\TypeNoArg}_1(\FractionNoArg)) \le \Threshold \\
\sum_{i\ge 2} \left( \Quality_i(\Combination{\TypeNoArg}_i(\FractionNoArg))\right) + \max_{z'_1 \le \Threshold} \left(\Quality_1(z'_1) - \CostFn_1(\Combination{\TypeNoArg}_1(\FractionNoArg), z'_1)\right) - \ProductionCost{\TypeNoArg} \cdot (1-\FractionNoArg)  &\text{ if } (\Combination{\TypeNoArg}_1(\FractionNoArg)) \in S(\Penalty) \\
\Quality(\Combination{\TypeNoArg}(\FractionNoArg)) - \Penalty - \ProductionCost{\TypeNoArg} \cdot (1-\FractionNoArg) &\text{ else. } \\
\end{cases} \]
Moreover, any optimizer $\FinalContentNoArg^*$ satisfies $\FinalContentNoArg^*_i = \Combination{\TypeNoArg}_i(\FractionNoArg)$ for $i \ge 2$. Finally, the optimizer $\FinalContentNoArg^*$ with minimum first dimension (as per the tiebreaking rule in Section \ref{subsec:userdecisions}) satisfies:
\[ 
\FinalContentNoArg^*_1 = 
\begin{cases}
G(\Combination{\TypeNoArg}_1(\FractionNoArg)) &\text{ if  } \Combination{\TypeNoArg}_1(\FractionNoArg) \in S(\Penalty) \\
\Combination{\TypeNoArg}_1(\FractionNoArg) & \text{ else. }
\end{cases}
\]
\end{lemma}
\begin{proof}

Let $\tilde{z}$ denote the content where the 1st coordinate is given by $\tilde{\FinalContentNoArg}_1 = G(\Combination{\TypeNoArg}_1(\FractionNoArg))$ if $\Combination{\TypeNoArg}_1(\FractionNoArg) \in S(\Penalty)$ and $\tilde{\FinalContentNoArg}_1 = \Combination{\TypeNoArg}_1(\FractionNoArg)$ otherwise, and the other coordinates are given by $\tilde{\FinalContentNoArg}_{i} = \Combination{\TypeNoArg}_i(\FractionNoArg)$ for all $i \ge 2$. 
Given any value of $\FinalContentNoArg$, we know that: 
\begin{align*}
&\Utility{\TypeNoArg}(\FractionNoArg,\FinalContentNoArg; \Penalty, \Threshold, \Quality, \CostFn) \\
&= \Quality(\FinalContentNoArg) -\ProductionCost{\TypeNoArg} \cdot (1-\FractionNoArg) -  \Penalty \cdot 1[\FinalContentNoArg_1 > \Threshold] - \CostFn(\Combination{\TypeNoArg}(\FractionNoArg), \FinalContentNoArg) \\
&= -\ProductionCost{\TypeNoArg} \cdot (1-\FractionNoArg) + \Quality_1(\FinalContentNoArg_1) - \Penalty \cdot 1[\FinalContentNoArg_1 > \Threshold] -  \CostFn_{1}(\Combination{\TypeNoArg}_{1}(\FractionNoArg), \FinalContentNoArg_{1})
+ \sum_{i \ge 2} \left(\Quality_{i}(\FinalContentNoArg) - \CostFn_{i}(\Combination{\TypeNoArg}_{i}(\FractionNoArg), \FinalContentNoArg_{i})\right)\\
&\le_{(1)} -\ProductionCost{\TypeNoArg} \cdot (1-\FractionNoArg) + \Quality_1(\FinalContentNoArg_1) - \Penalty \cdot 1[\FinalContentNoArg_1 > \Threshold] -  \CostFn_{1}(\Combination{\TypeNoArg}_{1}(\FractionNoArg), \FinalContentNoArg_{1})
+ \sum_{i \ge 2} \left(\Quality_{i}(\Combination{\TypeNoArg}_i(\FractionNoArg)\right)\\
&= -\ProductionCost{\TypeNoArg} \cdot (1-\FractionNoArg) + \Quality_1(\FinalContentNoArg_1) - \Penalty \cdot 1[\FinalContentNoArg_1 > \Threshold] -  \CostFn_{1}(\Combination{\TypeNoArg}_{1}(\FractionNoArg), \FinalContentNoArg_{1})
+ \sum_{i \ge 2} \left(\Quality_{i}(\tilde{\FinalContentNoArg}_i)\right)\\
&\le_{(2)} -\ProductionCost{\TypeNoArg} \cdot (1-\FractionNoArg) + \Quality_1(\tilde{\FinalContentNoArg}_1) - \Penalty \cdot 1[\tilde{\FinalContentNoArg}_1 > \Threshold] -  \CostFn_{1}(\Combination{\TypeNoArg}_{1}(\FractionNoArg), \tilde{\FinalContentNoArg}_1)
+ \sum_{i \ge 2} \left(\Quality_{i}(\tilde{\FinalContentNoArg}_i)\right)\\
&=\Quality(\tilde{\FinalContentNoArg}) -\ProductionCost{\TypeNoArg} \cdot (1-\FractionNoArg)- \Penalty \cdot 1[\tilde{\FinalContentNoArg}_1 > \Threshold] - \CostFn(\Combination{\TypeNoArg}(\FractionNoArg), \tilde{\ContentNoArg}) \\
&= \Utility{\TypeNoArg}(\FractionNoArg,\tilde{\FinalContentNoArg}; \Penalty, \Threshold, \Quality, \CostFn)
\end{align*}
where (1) follows from (A1). We observe that (2) follows from the following two facts. 
The first is that if $\Combination{\TypeNoArg}_1(\FractionNoArg) \not\in S(\Penalty)$, then by (A1): 
\[\Quality_1(\FinalContentNoArg_1) - \Penalty \cdot 1[\FinalContentNoArg_1 > \Threshold] -  \CostFn_{1}(\Combination{\TypeNoArg}_{1}(\FractionNoArg), \FinalContentNoArg_1) \le \Quality_1(\Combination{\TypeNoArg}_{1}(\FractionNoArg)) - \Penalty \cdot 1[\Combination{\TypeNoArg}_{1}(\FractionNoArg) > \Threshold]. \]
The second is that if $\Combination{\TypeNoArg}_1(\FractionNoArg) \in S(\Penalty)$, then: 
\[\Quality_1(\FinalContentNoArg_1) - \Penalty \cdot 1[\FinalContentNoArg_1 > \Threshold] -  \CostFn_{1}(\Combination{\TypeNoArg}_{1}(\FractionNoArg), \FinalContentNoArg_1) \le \Quality_1(G(\Combination{\TypeNoArg}_{1}(\FractionNoArg))) -  \Penalty \cdot 1[G(\Combination{\TypeNoArg}_{1}(\FractionNoArg)) > \Threshold] - \CostFn_{1}(\Combination{\TypeNoArg}_{1}(\FractionNoArg), G(\Combination{\TypeNoArg}_{1}(\FractionNoArg))). \]
(2) follows from these two facts. 

This proves that $\tilde{\FinalContentNoArg}$ is an optimizer of $\max_{\FinalContentNoArg \in \mathbb{R}^D}(\Utility{\TypeNoArg}(\FractionNoArg, \ContentNoArg;\Penalty, \Threshold, \Quality, \CostFn))$. Using the definition of $G$, we see that the optimal value is given by:  
\[\begin{cases}
\Quality(\Combination{\TypeNoArg}(\FractionNoArg)) - \ProductionCost{\TypeNoArg}\cdot (1-\FractionNoArg) &\text{ if } (\Combination{\TypeNoArg}_1(\FractionNoArg)) \le \Threshold \\
\sum_{i\ge 2} \left( \Quality_i(\Combination{\TypeNoArg}_i(\FractionNoArg))\right) + \max_{z'_1 \le \Threshold} \left(\Quality_1(z'_1) - \CostFn_1(\Combination{\TypeNoArg}_1(\FractionNoArg), z'_1)\right) - \ProductionCost{\TypeNoArg} \cdot (1-\FractionNoArg)  &\text{ if } (\Combination{\TypeNoArg}_1(\FractionNoArg)) \in S(\Penalty) \\
\Quality(\Combination{\TypeNoArg}(\FractionNoArg)) - \Penalty - \ProductionCost{\TypeNoArg} \cdot (1-\FractionNoArg) &\text{ else. }  \\
\end{cases} \]
as desired. 

Now, consider any optimizer $\FinalContentNoArg^*$. Then we know that equality must hold in the above chain of equalities. 
For (1), this means that 
\[\sum_{i\ge 2} \left(\Quality_{i}(\FinalContentNoArg^*) - \CostFn_{i}(\Combination{\TypeNoArg}_{i}(\FractionNoArg), \FinalContentNoArg^*_{i})\right) = \sum_{i\ge 2} \Quality_{i}(\Combination{\TypeNoArg}_i(\FractionNoArg)).\] 
Applying (A1), this means that $\FinalContentNoArg^*_i = \Combination{\TypeNoArg}_i(\FractionNoArg)$ for all $i \ge 2$. For (2), we know that:
\begin{align*}
\Quality_1(\FinalContentNoArg^*_1) - \Penalty \cdot 1[\FinalContentNoArg^*_1 > \Threshold] -  \CostFn_{1}(\Combination{\TypeNoArg}_{1}(\FractionNoArg), \FinalContentNoArg^*_1) &= \Quality_1(\tilde{\FinalContentNoArg}_1) - \Penalty \cdot 1[\tilde{\FinalContentNoArg}_1 > \Threshold] -  \CostFn_{1}(\Combination{\TypeNoArg}_{1}(\FractionNoArg), \tilde{\FinalContentNoArg}_1) \\
&\ge \max_{\FinalContentNoArg_1} \left(\Quality_1(\FinalContentNoArg_1) - \Penalty \cdot 1[\FinalContentNoArg > \Threshold] - \CostFn_{1}(\Combination{\TypeNoArg}_{1}(\FractionNoArg), \FinalContentNoArg_1) \right), 
\end{align*}
where the last step follows from the analysis above that we showed for every value of $\FinalContentNoArg_1$. 

Let's now assume that $\FinalContentNoArg^*$ tiebreaks in favor of lower values. Based on the definitions of $G$ and $S(\Penalty)$, we know that $\tilde{\FinalContentNoArg}_1$ is the maximizer of $\max_{\FinalContentNoArg_1} \left(\Quality_1(\FinalContentNoArg_1) - \Penalty \cdot 1[\FinalContentNoArg > \Threshold] - \CostFn_{1}(\Combination{\TypeNoArg}_{1}(\FractionNoArg), \FinalContentNoArg_1) \right)$ with lowest value, meaning that  $\FinalContentNoArg^*_1 = \tilde{\FinalContentNoArg}_1$ as desired. 
\end{proof}

We prove Theorem \ref{thm:detectorfinitecosts}.
\begin{proof}[Proof of Theorem \ref{thm:detectorfinitecosts}]
Using Lemma \ref{lemma:detectorfinitecostspostprocessing} and applying a partial reparameterization, we know that the optimal value 
$\max_{\FinalContentNoArg \in \mathbb{R}^D}(\Utility{\TypeNoArg}(\FractionNoArg, \ContentNoArg; \Penalty, \Threshold,\Quality, \CostFn))$ is given by: 
\[\begin{cases}
\Quality(\Combination{\TypeNoArg}(\FractionNoArg)) - \ProductionCost{\TypeNoArg}\cdot (1-\FractionNoArg) &\text{ if } \Combination{\TypeNoArg}_1(\FractionNoArg) \le \Threshold \\
\sum_{i\ge 2} \left( \Quality_i(\Combination{\TypeNoArg}_i(\FractionNoArg))\right) + \max_{z'_1 \le \Threshold} \left(\Quality_1(z'_1) - \CostFn_1(\Combination{\TypeNoArg}_1(\FractionNoArg), z'_1)\right) - \ProductionCost{\TypeNoArg} \cdot (1-\FractionNoArg)  &\text{ if } \FractionNoArg \in A^{\TypeNoArg}(\Penalty) \\
\Quality(\Combination{\TypeNoArg}(\FractionNoArg)) - \Penalty - \ProductionCost{\TypeNoArg} \cdot (1-\FractionNoArg) &\text{ else. } \\
\end{cases} \]
Since the functions $\Quality$, $\CostFn$, and $\Combination{\TypeNoArg}$ are continuous, we know that this function is continuous in $\FractionNoArg$ (we use Berge's maximum theorem to ensure continuity of $\max_{z'_1 \le \Threshold} \left(\Quality_1(z'_1) - \CostFn_1(\Combination{\TypeNoArg}_1(\FractionNoArg), z'_1)\right)$).  The optimization is over a compact set $[0,1]$, meaning that it achieves its optima. 

The characterization of $\FinalContent{\TypeNoArg}(\Penalty, \Threshold; \Quality, \CostFn)$ follows from Lemma \ref{lemma:detectorfinitecostspostprocessing} coupled with the fact that $A^{\TypeNoArg}(\Penalty)$ is a reparameterization of $S(\Penalty)$.  

If $\Combination{\TypeNoArg}_1(\FractionNoArg^*) \le \Threshold$, we show that $\Fraction{\TypeNoArg}(\Penalty, \Threshold; \Quality, \CostFn) = \FractionNoArg^*$. This follows from the fact that \\
$\max_{\FinalContentNoArg \in \mathbb{R}^D} (\Utility{\TypeNoArg}(\FractionNoArg, \FinalContentNoArg; \Penalty, \Threshold, \Quality, \CostFn))$ is point-wise dominated by the function $\Quality(\Combination{\TypeNoArg}(\FractionNoArg)) - \ProductionCost{\TypeNoArg}\cdot (1-\FractionNoArg)$ with equality at $\FractionNoArg^*$, and tiebreaking favors lower values of the work allocation. 

For the remainder of the analysis, we can assume that $\Combination{\TypeNoArg}_1(\FractionNoArg^*) > \Threshold$. We wish to show that $\Fraction{\TypeNoArg}(\Penalty, \Threshold; \Quality, \CostFn) \in A^{\TypeNoArg}(\Penalty) \cup \left\{ \FractionNoArg^{\text{game}},\FractionNoArg^*  \right\}$. It suffices to show any $\FractionNoArg \in [0,1]$ such that $\FractionNoArg \not\in A^{\TypeNoArg}(\Penalty) \cup \left\{ \FractionNoArg^{\text{game}},\FractionNoArg^*  \right\}$ is not equal to $\Fraction{\TypeNoArg}(\Penalty, \Threshold; \Quality, \CostFn)$. We split into cases.

\paragraph{Case 1: $\FractionNoArg \in (0,1)$ such that $\FractionNoArg \not\in A^{\TypeNoArg}(\Penalty) \cup \left\{ \FractionNoArg^{\text{game}},\FractionNoArg^*  \right\}$.} By Lemma \ref{lemma:apenaltyset}, we know that the set $A^{\TypeNoArg}(\Penalty) \cup \left\{ \FractionNoArg^{\text{game}} \right\}$ is closed. This means that there is a ball around $\FractionNoArg$ which is fully contained in either the first branch ($\Combination{\TypeNoArg}_1(\FractionNoArg) \le \Threshold$) or the third branch  ($\Combination{\TypeNoArg}_1(\FractionNoArg) > \Threshold$, $\FractionNoArg \not\in A^{\TypeNoArg}(\Penalty)$). The function $\max_{\FinalContentNoArg \in \mathbb{R}^D}(\Utility{\TypeNoArg}(\FractionNoArg, \ContentNoArg;\Penalty, \Threshold, \Quality, \CostFn))$ is differentiable on these branches. We now use the concavity of $\Quality$, and the fact that the derivative of $\Quality(\Combination{\TypeNoArg}(\FractionNoArg)) - \ProductionCost{\TypeNoArg} \cdot (1 - \FractionNoArg) - \Penalty$ is the same as the derivative of $\Quality(\Combination{\TypeNoArg}(\FractionNoArg)) - \ProductionCost{\TypeNoArg} \cdot (1 - \FractionNoArg)$ everywhere. If the derivative is nonzero, then we can find a local improvement, meaning that $\FractionNoArg \neq \Fraction{\TypeNoArg}(\Penalty, \Threshold; \Quality, \CostFn)$ as desired. Suppose the derivative is zero. By assumption, we know that $\FractionNoArg \neq \FractionNoArg^*$. However, by the concavity of $\Quality$, this means that the derivative of $\Quality(\Combination{\TypeNoArg}(\FractionNoArg)) - \ProductionCost{\TypeNoArg} \cdot (1 - \FractionNoArg) - \Penalty$ is $0$ along the full interval between $\FractionNoArg$ and $\FractionNoArg^*$. However, this is a contradiction given the tiebreaking rule.

\paragraph{Case 2: $\FractionNoArg =0$ and $\FractionNoArg \not\in A^{\TypeNoArg}(\Penalty) \cup \left\{\FractionNoArg^{\text{game}}, \FractionNoArg^*\right\}$.} By Lemma \ref{lemma:apenaltyset}, we know that the set $A^{\TypeNoArg}(\Penalty) \cup \left\{ \FractionNoArg^{\text{game}}  \right\}$ is closed. This means that there exists $\epsilon > 0$ such that around $\FractionNoArg + \epsilon'$ for $0 \le \epsilon' < \epsilon$ which is fully contained in either the first branch ($\Combination{\TypeNoArg}_1(\FractionNoArg) \le \Threshold$) or the third branch  ($\Combination{\TypeNoArg}_1(\FractionNoArg) > \Threshold$, $\FractionNoArg \not\in A^{\TypeNoArg}(\Penalty)$). The function $\max_{\FinalContentNoArg \in \mathbb{R}^D}(\Utility{\TypeNoArg}(\FractionNoArg, \ContentNoArg;\Penalty, \Threshold, \Quality, \CostFn))$ is differentiable on these branches. We now use the concavity of $\Quality$, and the fact that the derivative of $\Quality(\Combination{\TypeNoArg}(\FractionNoArg)) - \ProductionCost{\TypeNoArg} \cdot (1 - \FractionNoArg) - \Penalty$ is the same as the derivative of $\Quality(\Combination{\TypeNoArg}(\FractionNoArg)) - \ProductionCost{\TypeNoArg} \cdot (1 - \FractionNoArg)$ everywhere. If the derivative is positive, then we can find a local improvement, meaning that $\FractionNoArg \neq \Fraction{\TypeNoArg}(\Penalty, \Threshold; \Quality, \CostFn)$ as desired. Suppose the derivative is nonpositive. By assumption, we know that $0 \neq \FractionNoArg^*$. However, by the concavity of $\Quality$, this means that the derivative of $\Quality(\Combination{\TypeNoArg}(\FractionNoArg)) - \ProductionCost{\TypeNoArg} \cdot (1 - \FractionNoArg) - \Penalty$ is nonpositive along the full interval between $[0, \FractionNoArg^*]$. However, this would mean that $\FractionNoArg^* = 0$, due to the tiebreaking rule, which is a contradiction. 

\paragraph{Case 3: $\FractionNoArg =1$ and $\FractionNoArg \not\in A^{\TypeNoArg}(\Penalty) \cup \left\{\FractionNoArg^{\text{game}}, \FractionNoArg^*\right\}$.} By Lemma \ref{lemma:apenaltyset}, we know that the set $A^{\TypeNoArg}(\Penalty) \cup \left\{ \FractionNoArg^{\text{game}}  \right\}$ is closed. This means that there exists $\epsilon > 0$ such that around $\FractionNoArg - \epsilon'$ for $0 \le \epsilon' < \epsilon$ which is fully contained in either the first branch ($\Combination{\TypeNoArg}_1(\FractionNoArg) \le \Threshold$) or the third branch  ($\Combination{\TypeNoArg}_1(\FractionNoArg) > \Threshold$, $\FractionNoArg \not\in A^{\TypeNoArg}(\Penalty)$). The function $\max_{\FinalContentNoArg \in \mathbb{R}^D}(\Utility{\TypeNoArg}(\FractionNoArg, \ContentNoArg;\Penalty, \Threshold, \Quality, \CostFn))$ is differentiable on these branches. We now use the concavity of $\Quality$, and the fact that the derivative of $\Quality(\Combination{\TypeNoArg}(\FractionNoArg)) - \ProductionCost{\TypeNoArg} \cdot (1 - \FractionNoArg) - \Penalty$ is the same as the derivative of $\Quality(\Combination{\TypeNoArg}(\FractionNoArg)) - \ProductionCost{\TypeNoArg} \cdot (1 - \FractionNoArg)$ everywhere. If the derivative is negative, then we can find a local improvement, meaning that $\FractionNoArg \neq \Fraction{\TypeNoArg}(\Penalty, \Threshold; \Quality, \CostFn)$ as desired. Suppose the derivative is nonnegative. By assumption, we know that $1 \neq \FractionNoArg^*$. However, by the concavity of $\Quality$, this means that the derivative of $\Quality(\Combination{\TypeNoArg}(\FractionNoArg)) - \ProductionCost{\TypeNoArg} \cdot (1 - \FractionNoArg) - \Penalty$ is nonnegative along the full interval between $[\FractionNoArg^*, 1]$. However, the tiebreaking rule would mean that $\Fraction{\TypeNoArg}(\Penalty, \Threshold; \Quality, \CostFn)  = \FractionNoArg^* \neq 1$, which is a contradiction. 

\end{proof}

\section{Proofs for Section \ref{subsec:llmusage}}

\subsection{Proof of Theorem \ref{thm:usageincrease}}\label{appendix:proofthmincrease}

We prove Theorem \ref{thm:usageincrease}, and we also state and prove a variant for 2-dimensional settings that matches the presented proof intuition more closely. 

\begin{proof}[Proof of Theorem \ref{thm:usageincrease}]
Suppose that \eqref{eq:conditionincrease} is satisfied. 

First, we claim that $\Quality_1$ is weakly decreasing. Assume for sake of contradiction that $\Quality_1$ is weakly increasing (recall that we assumed one of the two holds). This means that:
\begin{align*}
 \max_{z'_1 \le \Threshold}(\Quality_1(z'_1) - \CostFn_1(z^u_1, z'_1)) - \max_{z'_1 \le \Threshold}(\Quality_1(z'_1) - \CostFn_1(z^l_1, z'_1)) &=_{(A)}  (\Quality_1(\Threshold) - \CostFn_1(z^u_1, \Threshold)) - (\Quality_1(\Threshold) - \CostFn_1(z^l_1, \Threshold)) \\
 &= \CostFn_1(z^l_1, \Threshold) - \CostFn_1(z^u_1, \Threshold) \\
 &<_{(B)} 0 \\
 &\le_{(C)} Q'_1(z^l_1) (z^u_1 - z^l_1).
\end{align*}
where (A) follows from Lemma \ref{lemma:gthreshold}, (B) follows from assumption (A2), and (C) follows from the fact that $Q_1$ is weakly increasing. This violates \eqref{eq:conditionincrease}, which is a contradiction.

Let $\TypeNoArg$ be such that $\HumanContent{\TypeNoArg}_1 = z^l_1$, $\LLMContent{\TypeNoArg}_1 = z^u_1$, $\HumanContent{\TypeNoArg}_i = \LLMContent{\TypeNoArg}_i$ for all $i \ge 2$. Let 
$\ProductionCost{\TypeNoArg} = - Q'_1(z^l_1) \cdot (z^u_1 - z^l_1)$, which we know is nonnegative since $Q_1$ is weakly decreasing. Let 
\[\bar{\Penalty} = \sup_{z_1 \in [z^l_1, z^u_1]} \left(Q_1(z_1) - \max_{z'_1 \le \Threshold} (\Quality_1(z'_1) - \CostFn_1(z_1, z'_1)\right)\] which we know is finite because it is the supremum of a continuous function (continuous by Berge's maximum theorem) on a compact set. 

We claim that $\FractionNoArg^* = 0$. By concavity, it suffices to show that the derivative at $\FractionNoArg = 0$ is nonpositive. Note that: 
\[\Quality(\Combination{\TypeNoArg}(\FractionNoArg)) - \ProductionCost{\TypeNoArg}(1 - \FractionNoArg) = \Quality_1(\Combination{\TypeNoArg}_1(\FractionNoArg)) - \ProductionCost{\TypeNoArg}(1 - \FractionNoArg) + \sum_{i=2}^D \Quality_i(\HumanContent{\TypeNoArg}_i). \]
Taking a derivative, we obtain:
\[\Quality'_1(\HumanContent{\TypeNoArg}_1) (\LLMContent{\TypeNoArg}_1 - \HumanContent{\TypeNoArg}_1) + \ProductionCost{\TypeNoArg} = 0. \]
This coupled with the tiebreaking rule means that $\FractionNoArg^* = 0$.  

We claim that for $\Penalty > \bar{\beta}$, it holds that $\FractionNoArg \in A^{\TypeNoArg}(\Penalty)$ for all $\FractionNoArg \in [0,1]$. We know that:
\begin{align*}
Q_1(\Combination{\TypeNoArg}_1(\FractionNoArg)) - \max_{z'_1 \le \Threshold} (\Quality_1(z'_1) - \CostFn_1(\Combination{\TypeNoArg}_1(\FractionNoArg), z'_1)) &\le  \sup_{z_1 \in [z^l_1, z^u_1]} \left(Q_1(z_1) - \max_{z'_1 \le \Threshold} (\Quality_1(z'_1) - \CostFn_1(z_1, z'_1)\right) \\
&= \bar{\Penalty} \\
&< \Penalty,   
\end{align*}
as desired.

Now, we claim that $\max_{z}(\Utility{\TypeNoArg}(1, \ContentNoArg; \Penalty, \Threshold, \Quality, \CostFn)) > \max_{z}(\Utility{\TypeNoArg}(0, \ContentNoArg; \Penalty, \Threshold, \Quality, \CostFn))$. We know that: 
\begin{align*}
&\max_{z}(\Utility{\TypeNoArg}(1, \ContentNoArg; \Penalty, \Threshold, \Quality, \CostFn)) - \max_{z}(\Utility{\TypeNoArg}(0, \ContentNoArg; \Penalty, \Threshold, \Quality, \CostFn))  \\
&=_{(A)} \left(\max_{z'_1 \le \Threshold}(\Quality_1(z'_1) - \CostFn_1(z^u_1, z'_1)) \right) - \left(\max_{z'_1 \le \Threshold}(\Quality_1(z'_1) - \CostFn_1(z^l_1, z'_1)) \right) + \ProductionCost{\TypeNoArg} \\
&= \left(\max_{z'_1 \le \Threshold}(\Quality_1(z'_1) - \CostFn_1(z^u_1, z'_1)) \right) - \left(\max_{z'_1 \le \Threshold}(\Quality_1(z'_1) - \CostFn_1(z^l_1, z'_1)) \right) - Q'_1(z^l_1) \cdot (z^u_1 - z^l_1)  \\
&>_{(B)} 0,
\end{align*}
where (A) uses Lemma \ref{lemma:detectorfinitecostspostprocessing}, (B)
uses \eqref{eq:conditionincrease}. 

By Theorem \ref{thm:nodetector}, we know that  $\Fraction{\TypeNoArg}(\emptyset; \Quality, \CostFn) = \FractionNoArg^*$. By the above, we know that $\Fraction{\TypeNoArg}(\Penalty, \Threshold; \Quality, \CostFn) \neq 0$. This means that $\Fraction{\TypeNoArg}(\Penalty, \Threshold; \Quality, \CostFn) > 0 = \FractionNoArg^* = \Fraction{\TypeNoArg}(\emptyset; \Quality, \CostFn)$ as desired. 

\end{proof}

The proof intuition in the main body and depicted in Figure \ref{fig:llmusage} actually shows a slightly more complex construction, which relies on slightly stronger assumptions. We state this alternate result below. We believe this construction conveys more practical intuition, even though the underlying result is less general. 
\begin{theorem}
\label{thm:alternateusageincrease}
Consider the same setup as Theorem \ref{thm:usageincrease}. Let's additionally assume that $D\ge 2$ and $Q_{2}$ is non-constant. If gaming costs $\CostFn$ are finite and there exists $z_1^u > z_1^l > \Threshold$ satisfying \eqref{eq:conditionincrease}, then there exists a penalty threshold $\bar{\Penalty} > 0$ and a type $\TypeNoArg$ of the form $\Quality_{2}(\LLMContent{\TypeNoArg}_2)  > \Quality_{2}(\HumanContent{\TypeNoArg}_2)$ and $\Quality_{1}(\LLMContent{\TypeNoArg}_1) < \Quality_{1}(\HumanContent{\TypeNoArg}_1)$ such that 
\[\Fraction{\TypeNoArg}(\Penalty, \Threshold; \Quality, \CostFn) > \Fraction{\TypeNoArg}(\emptyset; \Quality, \CostFn)  \text{ for all } \Penalty > \bar{\Penalty}.\] 
\end{theorem}
\begin{proof}
First, we claim that $\Quality_1$ is weakly decreasing. Assume for sake of contradiction that $\Quality_1$ is weakly increasing (recall that we assumed one of the two holds). This means that:
\begin{align*}
 \max_{z'_1 \le \Threshold}(\Quality_1(z'_1) - \CostFn_1(z^u_1, z'_1)) - \max_{z'_1 \le \Threshold}(\Quality_1(z'_1) - \CostFn_1(z^l_1, z'_1)) &=_{(A)}  (\Quality_1(\Threshold) - \CostFn_1(z^u_1, \Threshold)) - (\Quality_1(\Threshold) - \CostFn_1(z^l_1, \Threshold)) \\
 &= \CostFn_1(z^l_1, \Threshold) - \CostFn_1(z^u_1, \Threshold) \\
 &<_{(B)} 0 \\
 &\le_{(C)} Q'_1(z^l_1) (z^u_1 - z^l_1).
\end{align*}
where (A) follows from Lemma \ref{lemma:gthreshold}, (B) follows from assumption (A2), and (C) follows from the fact that $Q_1$ is weakly increasing. This violates \eqref{eq:conditionincrease}, which is a contradiction.  

We will let $\epsilon_1 > 0$ and $\epsilon_2 >0$ be sufficiently small parameters that we will set later. Let $\TypeNoArg$ be such that $\HumanContent{\TypeNoArg}_1 = z^l_1$, $\LLMContent{\TypeNoArg}_1 = z^u_1$, $\HumanContent{\TypeNoArg}_i = \LLMContent{\TypeNoArg}_i$ for all $i > 2$. Furthermore, let $\TypeNoArg$ be such that $Q'_2(\HumanContent{\TypeNoArg}_2) \neq 0$ (this exists because $Q_2$ is non-constant by assumption). 
Let 
\[\bar{\beta} = \sup_{z_1 \in [z^l_1, z^u_1]} \left(Q_1(z_1) - \max_{z'_1 \le \Threshold} (\Quality_1(z'_1) - \CostFn_1(z_1, z'_1)\right)\] which we know is finite because it is the maximum of a continuous function (continuous by Berge's maximum theorem) on a compact set. 

Let $\LLMContent{\TypeNoArg}_2 = \HumanContent{\TypeNoArg}_2 + \epsilon_1 \cdot \text{sign}(Q'_2(\HumanContent{\TypeNoArg}_2))$. (If $\epsilon_1$ is positive and falls below some positive threshold $\bar{\epsilon}_1$, we know that $\Quality_2(\LLMContent{\TypeNoArg}_2) > \Quality_2(\HumanContent{\TypeNoArg}_2)$.) Let 
$\ProductionCost{\TypeNoArg} = - Q'_1(z^l_1) \cdot (z^u_1 - z^l_1) - \epsilon_2$, which we know is nonnegative since $Q_1$ is weakly decreasing.

If $|\Quality'_2(\HumanContent{\TypeNoArg}_2)| \cdot \epsilon_1 - \epsilon_2 \le 0$, we claim that $\FractionNoArg^* = 0$. By concavity, it suffices to show that the derivative at $\FractionNoArg = 0$ is nonpositive. Note that: 
\[\Quality(\Combination{\TypeNoArg}(\FractionNoArg)) - \ProductionCost{\TypeNoArg}(1 - \FractionNoArg) = \Quality_1(\Combination{\TypeNoArg}_1(\FractionNoArg)) + \Quality_2(\HumanContent{\TypeNoArg}_2) - \ProductionCost{\TypeNoArg}(1 - \FractionNoArg) +   \sum_{i > 2} \Quality_i(\HumanContent{\TypeNoArg}_i). \]
Taking a derivative, we obtain:
\begin{align*}
 \Quality'_1(\HumanContent{\TypeNoArg}_1) (\LLMContent{\TypeNoArg}_1 - \HumanContent{\TypeNoArg}_1) + \Quality'_2(\HumanContent{\TypeNoArg}_2) \cdot (\LLMContent{\TypeNoArg}_2 - \HumanContent{\TypeNoArg}_2) + \ProductionCost{\TypeNoArg} &= |\Quality'_2(\HumanContent{\TypeNoArg}_2)| \cdot \epsilon_1 - \epsilon_2 \le 0.   
\end{align*}
This coupled with the tiebreaking rule means that $\FractionNoArg^* = 0$.  

We claim that for $\Penalty > \bar{\beta}$, it holds that $\FractionNoArg \in A^{\TypeNoArg}(\Penalty)$ for all $\FractionNoArg \in [0,1]$. We know that:
\begin{align*}
Q_1(\Combination{\TypeNoArg}_1(\FractionNoArg)) - \max_{z'_1 \le \Threshold} (\Quality_1(z'_1) - \CostFn_1(\Combination{\TypeNoArg}_1(\FractionNoArg), z'_1)) &\le  \sup_{z_1 \in [z^l_1, z^u_1]} \left(Q_1(z_1) - \max_{z'_1 \le \Threshold} (\Quality_1(z'_1) - \CostFn_1(z_1, z'_1)\right) \\
&= \bar{\Penalty} \\
&< \Penalty,   
\end{align*}
as desired. 
 
Now, we claim that $\max_{z}(\Utility{\TypeNoArg}(1, \ContentNoArg; \Penalty, \Threshold, \Quality, \CostFn)) > \max_{z}(\Utility{\TypeNoArg}(0, \ContentNoArg; \Penalty, \Threshold,  \Quality, \CostFn))$. For $\epsilon_2$ below some positive threshold  $\bar{\epsilon}_2$, we know that: 
\begin{align*}
&\max_{z}(\Utility{\TypeNoArg}(1, \ContentNoArg; \Penalty, \Threshold, \Quality, \CostFn)) - \max_{z}(\Utility{\TypeNoArg}(0, \ContentNoArg; \Penalty, \Threshold,  \Quality, \CostFn))  \\
&=_{(A)} \left(\max_{z'_1 \le \Threshold}(\Quality_1(z'_1) - \CostFn_1(z^u_1, z'_1)) \right) - \left(\max_{z'_1 \le \Threshold}(\Quality_1(z'_1) - \CostFn_1(z^l_1, z'_1)) \right) + \ProductionCost{\TypeNoArg} + \Quality_2(\LLMContent{\TypeNoArg}_2) - \Quality_2(\HumanContent{\TypeNoArg}_2)   \\
&>_{(B)} \left(\max_{z'_1 \le \Threshold}(\Quality_1(z'_1) - \CostFn_1(z^u_1, z'_1)) \right) - \left(\max_{z'_1 \le \Threshold}(\Quality_1(z'_1) - \CostFn_1(z^l_1, z'_1)) \right) - Q'_1(z^l_1) \cdot (z^u_1 - z^l_1) - \epsilon_2  \\
&>_{(C)} 0,
\end{align*}
where (A) uses Lemma \ref{lemma:detectorfinitecostspostprocessing}, and (B) uses that $\epsilon_1$ is sufficiently small, and (C)
uses that $\epsilon_2$ is sufficiently small and uses \eqref{eq:conditionincrease}.

By Theorem \ref{thm:nodetector}, we know that  $\Fraction{\TypeNoArg}(\emptyset; \Quality, \CostFn) = \FractionNoArg^*$. By the above, we know that $\Fraction{\TypeNoArg}(\Penalty, \Threshold; \Quality, \CostFn) \neq 0$. This means that $\Fraction{\TypeNoArg}(\Penalty, \Threshold; \Quality, \CostFn) > 0 = \FractionNoArg^* = \Fraction{\TypeNoArg}(\emptyset; \Quality, \CostFn)$ as desired. 

We can satisfy the conditions on $\epsilon_1$ as follows. Let $\epsilon_2 = \bar{\epsilon}_2/2$, and let $\epsilon_1 = \min(\bar{\epsilon}_1/2, \epsilon_2 / |\Quality'_2(\HumanContent{\TypeNoArg}_2)|)$.

\end{proof}

\subsection{Statement and Proof of Proposition \ref{proposition:nonmonotoneusage}}\label{appendix:propconstruction}

\begin{proposition}
\label{proposition:nonmonotoneusage}
Fix $k_1 \in \mathbb{R}$, $k_2 > e^{-2}$, and $0 < k_3 < 4 k_2$, and $k_4 > 0$. Let $\Threshold = 0$, $\Quality_1(z_1) = k_1 - k_2 z_1$, and let 
\[\CostFn_1(z_1, z'_1) = k_2(z_1 - z'_1) + \frac{k_3}{1 + e^{z_1}} (1 - e^{-(z_1 - z'_1)}) + k_4 (z_1 - z'_1)^2.\] Then the assumptions in Section \ref{sec:model} are satisfied. For any $z^u_1 >  z^l_1 > \ln(1+\sqrt{2})$, if $k_4$ is sufficiently small, the condition \eqref{eq:conditionincrease} holds. 
\end{proposition}
\begin{proof}

First, we see that $\Quality_1$ is weakly decreasing and continuously differentiable. 

For (A1), note that:
\[\CostFn_1(z_1, z'_1) =  k_2(z_1 - z'_1) + \frac{k_3}{1 + e^{z_1}} (1 - e^{-(z_1 - z'_1)}) + k_4 (z_1 - z'_1)^2 \ge \Quality_1(z'_1) - \Quality_1(z_1)\]
as desired. Equality holds only when $z'_1 = z_1$. 

For (A2), we take a derivative with respect to $z_1$ to obtain:
\[k_2  -  \frac{k_3 \cdot e^{z_1}}{(1 + e^{z_1})^2} +  \frac{k_3(2e^{z'_1} + e^{-z_1+z'_1})}{(1 + e^{z_1})^2} +  2 k_4 (z_1 - z'_1) \ge k_2 - \frac{k_3}{4} > 0.  \]
We take a derivative with respect to $z'_1$ to obtain: 
\[-k_2 - \frac{k_3}{1 + e^{z_1}} e^{-(z_1 - z'_1)}) + 2 k_4 (z'_1 - z_1) < 0. \]

For (A3), we see that:
\[\CostFn_1(z_1, z'_1) =  k_2(z_1 - z'_1) + \frac{k_3}{1 + e^{z_1}} (1 - e^{-(z_1 - z'_1)}) + k_4 (z_1 - z'_1)^2 \ge 0\] 
since each term is nonnegative because $z_1 \ge z'_1$. Equality is achieved if and only if $z_1 = z'_1$. 

For (A4), note that $k_2 (z_1 - z'_1) \rightarrow \infty$ as $z_1 \rightarrow \infty$, and the other terms are nonnegative.  

For (A5), we see that 
\begin{align*}
 \Quality_1(z'_1) - \CostFn_1(z_1, z'_1) &= k_1 - k_2 z'_1 - k_2(z_1 - z'_1) - \frac{k_3}{1 + e^{z_1}} (1 - e^{-(z_1 - z'_1)}) - k_4 (z_1 - z'_1)^2  \\
 &= k_1 - k_2z_1 - \frac{k_3}{1 + e^{z_1}} (1 - e^{-(z_1 - z'_1)}) - k_4 (z_1 - z'_1)^2.
\end{align*}
Note that for any $\ContentNoArg''_i$, it holds that the utility $\sup_{\ContentNoArg_i \ge \ContentNoArg''_i} (- k_4 (z_1 - z'_1)^2)  \rightarrow -\infty$ as $\ContentNoArg'_i \rightarrow -\infty$. 

Now, observe that:
\begin{align*}
\max_{z'_1 \le \Threshold} \left(Q_1(z'_1) - \CostFn_1(z_1, z'_1) \right) &= \max_{z'_1 \le \Threshold} \left( k_1 - k_2z_1 - \frac{k_3}{1 + e^{z_1}} (1 - e^{-(z_1 - z'_1)}) - k_4 (z_1 - z'_1)^2 \right) \\
&= k_1 - k_2z_1 - \min_{z'_1 \le \Threshold} \left(\frac{k_3}{1 + e^{z_1}} (1 - e^{-(z_1 - z'_1)}) + k_4 (z_1 - z'_1)^2 \right).  
\end{align*}

Note that $\frac{k_3}{1 + e^{z_1}} (1 - e^{-(z_1 - z'_1)}) + k_4 (z_1 - z'_1)^2$ is decreasing with $z'_1$, so it is minimized at $z'_1 = \Threshold = 0$. Thus we obtain:
\[\max_{z'_1 \le \Threshold} \left(Q_1(z'_1) - \CostFn_1(z_1, z'_1) \right) = k_1 - k_2z_1 - \left(\frac{k_3}{1 + e^{z_1}} (1 - e^{-z_1}) + k_4 z_1^2 \right). \]

Now, observe that:
\begin{align*}
 & \max_{z'_1 \le \Threshold} \left(Q_1(z'_1) - \CostFn_1(z^u_1, z'_1) \right) - \max_{z'_1 \le \Threshold} \left(Q_1(z'_1) - \CostFn_1(z^l_1, z'_1) \right) \\
 &= - k_2(z^u_1 - z^l_1) - \frac{k_3}{1 + e^{z^u_1}} (1 - e^{-z^u_1}) - \frac{k_3}{1 + e^{z^l_1}} (1 - e^{-z^l_1}) + k_4 (z_1^l)^2 - k_4 (z_1^u)^2  \\
 &= Q'_1(z^l_1) (z^u_1 - z^l_1) - \frac{k_3 (1 - e^{-z^u_1})}{1 + e^{z^u_1}}  + \frac{k_3 (1 - e^{-z^l_1})}{1 + e^{z^l_1}}  + k_4 (z_1^l)^2 - k_4 (z_1^u)^2
\end{align*}
Let $f(x) = \frac{k_3 (1 - e^{-x)}}{1 + e^{x}}$. By taking a derivative, we can see that $f'(x) < 0$ for $x \ge \ln(1+\sqrt{2})$. This means that for $z^u_1 >  z^l_1 > \ln(1+\sqrt{2})$, it holds that $ - \frac{k_3 (1 - e^{-z^u_1})}{1 + e^{z^u_1}}  + \frac{k_3 (1 - e^{-z^l_1})}{1 + e^{z^l_1}} > 0$. We can take $k_4$ sufficiently small so that the overall expression is positive. 
\end{proof}

\subsection{Proof of Theorem \ref{thm:usagedecrease}}\label{appendix:thmdecrease}

To prove Theorem \ref{thm:usagedecrease}, we first prove a series of lemmas. The first case is infinite gaming costs. 

\begin{lemma}
\label{lemma:infinitellmusage}
Fix $\Penalty > 0$, $\Threshold < \infty$, and $\Quality$. Suppose that gaming costs $\CostFn = \CostFn^{\infty}$ are infinite. If $\Combination{\TypeNoArg}_1(\FractionNoArg^*) > \Threshold$, then it holds that:
\[ \Fraction{\TypeNoArg}(\Penalty, \Threshold; \Quality, \CostFn) \le \FractionNoArg^*.\]
\end{lemma}
\begin{proof}
By Theorem \ref{thm:infinitecostdetector}, we know that $\Fraction{\TypeNoArg}(\Penalty, \Threshold; \Quality, \CostFn) \in \left\{\FractionNoArg^*, \FractionNoArg^{\text{game}} \right\}$. If $\Fraction{\TypeNoArg}(\Penalty, \Threshold; \Quality, \CostFn) = \FractionNoArg^*$, then we are done. If $\Fraction{\TypeNoArg}(\Penalty, \Threshold; \Quality, \CostFn) = \FractionNoArg^{\text{game}}$, then we know that $\FractionNoArg^{\text{game}} < \infty$ and this case would imply that $\FractionNoArg^* \ge \FractionNoArg^{\text{game}} = \Fraction{\TypeNoArg}(\Penalty, \Threshold; \Quality, \CostFn)$ as desired. 

\end{proof}

The second case is that $\max_{z'_1 \le \Threshold} \left(\Quality_1(z'_1) - \CostFn_1(z_1, z'_1)\right) - Q_1(z_1)$ is weakly decreasing in $z_1$ for all $z_1 > \Threshold$. 
\begin{lemma}
\label{lemma:increasingllmusage}
Fix $\Penalty > 0$, $\Threshold < \infty$, and finite costs $\CostFn$. Suppose that $\max_{z'_1 \le \Threshold} \left(\Quality_1(z'_1) - \CostFn_1(z_1, z'_1)\right) - Q_1(z_1)$ is weakly decreasing in $z_1$ for all $z_1 > \Threshold$.   
If $\Combination{\TypeNoArg}_1(\FractionNoArg^*) > \Threshold$, then it holds that:
\[ \Fraction{\TypeNoArg}(\Penalty, \Threshold; \Quality, \CostFn) \le \FractionNoArg^*.\]
\end{lemma}
\begin{proof}
 By Theorem \ref{thm:detectorfinitecosts}, 
we know that $\Fraction{\TypeNoArg}(\Penalty, \Threshold; \Quality, \CostFn) \in \left\{\FractionNoArg^*, \FractionNoArg^{\text{game}} \right\} \cup A^{\TypeNoArg}(\Penalty)$. If $\Fraction{\TypeNoArg}(\Penalty, \Threshold; \Quality, \CostFn) = \FractionNoArg^{\text{game}}$, then we know that $\FractionNoArg^{\text{game}} < \infty$ and that thus $\Combination{\TypeNoArg}_1(\FractionNoArg^*) > \Threshold = \Combination{\TypeNoArg}_1(\FractionNoArg^{\text{game}})$, which means that $\FractionNoArg^* \ge \FractionNoArg^{\text{game}} = \Fraction{\TypeNoArg}(\Penalty, \Threshold; \Quality, \CostFn)$ as desired. 

Otherwise, we know that $\Fraction{\TypeNoArg}(\Penalty, \Threshold; \Quality, \CostFn) \in A^{\TypeNoArg}(\Penalty)$. Assume for sake of contradiction that $\Fraction{\TypeNoArg}(\Penalty, \Threshold; \Quality, \CostFn) > \FractionNoArg^*$. 

We first claim that $A' := A^{\TypeNoArg}(\Penalty) \cap \left\{ \FractionNoArg \in [0,1] \mid \FractionNoArg \ge \FractionNoArg^* \right\}$ is nonempty and closed, meaning that $\min A'$ exists. The fact that it is nonempty follows from the fact that $\Fraction{\TypeNoArg}(\Penalty, \Threshold; \Quality, \CostFn) > \FractionNoArg^*$. Now, we show it is closed. Using Lemma \ref{lemma:apenaltyset}, we know that $\bar{A}^{\TypeNoArg}(\Penalty) \setminus A^{\TypeNoArg}(\Penalty) \subseteq \left\{\FractionNoArg^{\text{game}}  \right\}$. Since the intersection of two closed sets is closed, we know that $\bar{A}^{\TypeNoArg}(\Penalty) \cap \left\{ \FractionNoArg \in [0,1] \mid \FractionNoArg \ge \FractionNoArg^* \right\}$ is closed. Since $\Combination{\TypeNoArg}_1(\FractionNoArg^*) > \Threshold$, we know that  $\FractionNoArg^{\text{game}}  \not\in \bar{A}^{\TypeNoArg}(\Penalty) \cap \left\{ \FractionNoArg \in [0,1] \mid \FractionNoArg \ge \FractionNoArg^* \right\}$, which means that 
\[A' = A^{\TypeNoArg}(\Penalty) \cap \left\{ \FractionNoArg \in [0,1] \mid \FractionNoArg \ge \FractionNoArg^* \right\} = \bar{A}^{\TypeNoArg}(\Penalty) \cap \left\{ \FractionNoArg \in [0,1] \mid \FractionNoArg \ge \FractionNoArg^* \right\},\] so $A'$ is closed. 

First, we claim that $\Fraction{\TypeNoArg}(\Penalty, \Threshold; \Quality, \CostFn) = \min A'$.  We know that:  
\begin{align*}
  &\max_{z} \Utility{\TypeNoArg}(\Fraction{\TypeNoArg}(\Penalty, \Threshold; \Quality, \CostFn), \ContentNoArg; \Penalty, \Threshold, \Quality, \CostFn) \\ &=_{(A)} \sum_{i\ge 2} \left( \Quality_i(\Combination{\TypeNoArg}_i(\Fraction{\TypeNoArg}(\Penalty, \Threshold; \Quality, \CostFn)))\right) \\
  &+ \max_{z'_1 \le \Threshold} \left(\Quality_1(z'_1) - \CostFn_1((\Combination{\TypeNoArg}_1(\Fraction{\TypeNoArg}(\Penalty, \Threshold; \Quality, \CostFn))), z'_1)\right) - \ProductionCost{\TypeNoArg} \cdot (1-\Fraction{\TypeNoArg}(\Penalty, \Threshold; \Quality, \CostFn))  \\
  &= \left( \Quality(\Combination{\TypeNoArg}(\Fraction{\TypeNoArg}(\Penalty, \Threshold; \Quality, \CostFn))) - \ProductionCost{\TypeNoArg} \cdot (1-\Fraction{\TypeNoArg}(\Penalty, \Threshold; \Quality, \CostFn))\right)  \\ &+\max_{z'_1 \le \Threshold} \left(\Quality_1(z'_1) - \CostFn_1((\Combination{\TypeNoArg}_1(\Fraction{\TypeNoArg}(\Penalty, \Threshold; \Quality, \CostFn))), z'_1)\right) - \Quality_1(\Combination{\TypeNoArg}_1(\Fraction{\TypeNoArg}(\Penalty, \Threshold; \Quality, \CostFn))) \\
  &\le_{(B)}  \left( \Quality(\Combination{\TypeNoArg}(\min A')) - \ProductionCost{\TypeNoArg} \cdot (1-\min A')\right) \\
  &+\max_{z'_1 \le \Threshold} \left(\Quality_1(z'_1) - \CostFn_1((\Combination{\TypeNoArg}_1(\Fraction{\TypeNoArg}(\Penalty, \Threshold; \Quality, \CostFn))), z'_1)\right) - \Quality_1(\Combination{\TypeNoArg}_1(\Fraction{\TypeNoArg}(\Penalty, \Threshold; \Quality, \CostFn))) \\ 
  &= \max_{z} \Utility{\TypeNoArg}(\min A', \ContentNoArg; \Penalty, \Threshold, \Quality, \CostFn) - \left(\max_{z'_1 \le \Threshold} \left(\Quality_1(z'_1) - \CostFn_1((\Combination{\TypeNoArg}_1(\min A')), z'_1)\right) - \Quality_1(\Combination{\TypeNoArg}_1(\min A')) \right) \\
  &+\max_{z'_1 \le \Threshold} \left(\Quality_1(z'_1) - \CostFn_1((\Combination{\TypeNoArg}_1(\Fraction{\TypeNoArg}(\Penalty, \Threshold; \Quality, \CostFn))), z'_1)\right) - \Quality_1(\Combination{\TypeNoArg}_1(\Fraction{\TypeNoArg}(\Penalty, \Threshold; \Quality, \CostFn)))
\end{align*}
where (A) uses  Lemma \ref{lemma:detectorfinitecostspostprocessing}, and (B) uses the fact that 
$\Fraction{\TypeNoArg}(\Penalty, \Threshold; \Quality, \CostFn) \ge \min A' \ge \FractionNoArg^*$ and $\Quality(\Combination{\TypeNoArg}(\FractionNoArg)) - \ProductionCost{\TypeNoArg}(1 - \FractionNoArg)$ is concave and maximized at $\FractionNoArg^*$. Because of the tiebreaking rule, it now suffices to show that $\max_{z'_1 \le \Threshold} \left(\Quality_1(z'_1) - \CostFn_1((\Combination{\TypeNoArg}_1(\Fraction{\TypeNoArg}(\Penalty, \Threshold; \Quality, \CostFn))), z'_1)\right) - \Quality_1(\Combination{\TypeNoArg}_1(\Fraction{\TypeNoArg}(\Penalty, \Threshold; \Quality, \CostFn)))$ is upper bounded by \\
$\max_{z'_1 \le \Threshold} \left(\Quality_1(z'_1) - \CostFn_1(\Combination{\TypeNoArg}_1(\min A'), z'_1)\right) - \Quality_1(\Combination{\TypeNoArg}_1(\min A'))$. This follows from the assumption that $\max_{z'_1 \le \Threshold} \left(\Quality_1(z'_1) - \CostFn_1(z_1, z'_1)\right) - Q_1(z_1)$ is weakly decreasing in $z_1$ for all $z_1 > \Threshold$. 

Next, we claim that $\max_{z} \Utility{\TypeNoArg}(\min A', \ContentNoArg; \Penalty, \Threshold, \Quality, \CostFn) \le \max_{z} \Utility{\TypeNoArg}(\FractionNoArg^*, \ContentNoArg; \Penalty, \Threshold, \Quality, \CostFn)$. By assumption, we know that $\min A' \neq \FractionNoArg^*$ and so $((\min A') - \epsilon) \not\in A^{\TypeNoArg}(\Penalty)$ for sufficiently small $\epsilon > 0$. We know that: 
\[ A^{\TypeNoArg}(\Penalty) = \left\{ \FractionNoArg \in[0,1] \mid \Combination{\TypeNoArg}_1(\FractionNoArg) >  \Threshold, \Quality_1(\Combination{\TypeNoArg}(\FractionNoArg)) - \max_{z'_1 \le \Threshold} \left(\Quality_1(z'_1) -  \CostFn_1(\Combination{\TypeNoArg}_1(\FractionNoArg), z'_1) \right) \le \Penalty \right\}.\]
Using that $\Combination{\TypeNoArg}_1(\min A')) > \Combination{\TypeNoArg}_1(\FractionNoArg^*) > \Threshold$ by assumption, we know that $\Combination{\TypeNoArg}_1((\min A') - \epsilon)  > \Threshold$ for sufficiently small $\epsilon > 0$. 
Since $\Quality_1(\Combination{\TypeNoArg}(\FractionNoArg)) - \max_{z'_1 \le \Threshold} \left(\Quality_1(z'_1) -  \CostFn_1(\Combination{\TypeNoArg}(\FractionNoArg), z'_1) \right)$ is continuous in $\FractionNoArg$ (the continuity of the max follows from Berge's maximum theorem), this altogether means that $\Quality_1(\Combination{\TypeNoArg}(\min A')) - \max_{z'_1 \le \Threshold} \left(\Quality_1(z'_1) -  \CostFn_1(\Combination{\TypeNoArg}(\min A'), z'_1) \right) = \Penalty$. Thus, using Lemma \ref{lemma:detectorfinitecostspostprocessing}, we know that:
\begin{align*}
  &\max_{z} \Utility{\TypeNoArg}(\min A', \ContentNoArg; \Penalty, \Threshold, \Quality, \CostFn) \\ &=  \left( \Quality(\Combination{\TypeNoArg}(\min A')) - \ProductionCost{\TypeNoArg} \cdot (1-\min A')\right) - \Quality_1(\Combination{\TypeNoArg}_1(\min A')) + \max_{z'_1 \le \Threshold} \left(\Quality_1(z'_1)  -  \CostFn_1(\Combination{\TypeNoArg}_1(\min A'), z'_1) \right) \\
  &= \left( \Quality(\Combination{\TypeNoArg}(\min A')) - \ProductionCost{\TypeNoArg} \cdot (1-\min A')\right) - \Penalty \\
  &\le_{(C)} \left( \Quality(\Combination{\TypeNoArg}(\FractionNoArg^*)) - \ProductionCost{\TypeNoArg} \cdot (1-\FractionNoArg^*)\right) - \Penalty \\
  &=_{(D)} \max_{z} \Utility{\TypeNoArg}(\FractionNoArg^*, \ContentNoArg; \Penalty, \Threshold, \Quality, \CostFn)
\end{align*}
where (C) uses the fact that 
$\min A' > \FractionNoArg^*$ and $\Quality(\Combination{\TypeNoArg}(\FractionNoArg)) - \ProductionCost{\TypeNoArg}(1 - \FractionNoArg)$ is concave and maximized at $\FractionNoArg^*$, and (D) uses the fact that $\FractionNoArg^* \not\in A^{\TypeNoArg}(\Penalty)$ since $\min A' \neq \FractionNoArg^*$. However, based on the tiebreaking rule, this would mean that $\Fraction{\TypeNoArg}(\Penalty, \Threshold; \Quality, \CostFn) = \FractionNoArg^*$, which is a contradiction. 
\end{proof}

We prove Theorem \ref{thm:usagedecrease}. 
\begin{proof}[Proof of Theorem \ref{thm:usagedecrease}]

Note that the condition that there does not exist $z^u_1 > z^l_1  > \Threshold$ such that \eqref{eq:condition} holds is equivalent to the condition that $\max_{z'_1 \le \Threshold} \left(\Quality_1(z'_1) - \CostFn_1(z_1, z'_1)\right) - Q_1(z_1)$ is weakly decreasing in $z_1$ for all $z_1 > \Threshold$. It suffices to prove that if gaming costs $\CostFn$ are infinite or $\max_{z'_1 \le \Threshold} \left(\Quality_1(z'_1) - \CostFn_1(z_1, z'_1)\right) - Q_1(z_1)$ is weakly decreasing in $z_1$ for all $z_1 > \Threshold$, then the presence of the detector weakly decreases LLM usage for all types $\TypeNoArg$ and all penalties $\Penalty > 0$. We split into cases based on the value of $\Combination{\TypeNoArg}_1(\FractionNoArg^*)$.

\paragraph{Case 1: $\Combination{\TypeNoArg}_1(\FractionNoArg^*) \le \Threshold$.} In this case, we know that $\Fraction{\TypeNoArg}(\Penalty, \Threshold; \Quality, \CostFn) = \FractionNoArg^*$ in the case of finite gaming costs (by Theorem \ref{thm:detectorfinitecosts}) and infinite gaming costs (by Theorem \ref{thm:infinitecostdetector}). This, coupled with Theorem \ref{thm:nodetector}, implies that: \[\Fraction{\TypeNoArg}(\Penalty, \Threshold; \Quality, \CostFn) = \FractionNoArg^* = \Fraction{\TypeNoArg}(\emptyset; \Quality, \CostFn) \]
as desired. 

\paragraph{Case 2: $\Combination{\TypeNoArg}_1(\FractionNoArg^*) > \Threshold$.} We know that:
\[\Fraction{\TypeNoArg}(\Penalty, \Threshold; \Quality, \CostFn) \le_{(A)} \FractionNoArg^* =_{(B)} \Fraction{\TypeNoArg}(\emptyset; \Quality, \CostFn) \]
where (A) follows from Lemma \ref{lemma:infinitellmusage} for the case of infinite costs and Lemma \ref{lemma:increasingllmusage} for the case where $\max_{z'_1 \le \Threshold} \left(\Quality_1(z'_1) - \CostFn_1(z_1, z'_1)\right) - Q_1(z_1)$ is weakly decreasing in $z_1$ for all $z_1 > \Threshold$, and (B) follows from Theorem \ref{thm:nodetector}.

\end{proof}

\subsection{Proof of Corollary \ref{cor:monotoneusage}}\label{appendix:cordecrease}

Besides Theorem \ref{thm:usagedecrease}, another key lemma is the following structural property of $\max_{z'_1 \le \Threshold} (\Quality_1(z'_1) - \CostFn_1(z_1, z'_1)) - \Quality_1(z_1)$. 
\begin{lemma}
\label{lemma:maxshifteddecreasing}
Fix $\Quality$ and finite costs $\CostFn$. Suppose that either (1) $\Quality_1$ is increasing, or (2) $\frac{\partial^2 \CostFn_1 (z_1, z'_1)}{\partial z_1 \partial z'_1} \le 0$ for all $z'_1 \le z_1$. Then, it holds that $\max_{z'_1 \le \Threshold} (\Quality_1(z'_1) - C(z_1, z'_1)) - \Quality_1(z_1)$ is weakly decreasing in $z_1$ for $z_1 > \Threshold$.
\end{lemma}
\begin{proof}
We split into two cases: (1) $\Quality_1$ is weakly increasing, and (2) $\frac{\partial^2 \CostFn (z_1, z'_1)}{\partial z_1 \partial z'_1} \le 0$ for all $z'_1 \le z_1$

\paragraph{Case 1: $\Quality_1$ is weakly increasing.} 
By Lemma \ref{lemma:gthreshold}, we know that $G(z_1) = \Threshold$ for all $z_1 \ge \Threshold$. This means that for $z''_1 \le z_1$, it holds that:
    \begin{align*}
    &\max_{z'_1 \le \Threshold} \left(\Quality_1(z'_1) - \CostFn_1(z_1, z'_1)\right) - \Quality_1(z_1) \\
    &\le_{(A)} \max_{z'_1 \le \Threshold} \left(\Quality_1(z'_1) - \CostFn_1(z''_1, z'_1)\right) - \Quality_1(z_1) \\
     &\le_{(B)} \max_{z'_1 \le \Threshold} \left(\Quality_1(z'_1) - \CostFn_1(z''_1, z'_1)\right) - \Quality_1(z''_1)
    \end{align*}
    where (A) follows from Lemma \ref{lemma:maxweaklydecreasing}, and (B) follows from the fact that $\Quality_1$ is weakly increasing. 

\paragraph{Case 2: $\frac{\partial^2 \CostFn_1 (z_1, z'_1)}{\partial z_1 \partial z'_1} \le 0$ for all $z'_1 \le z_1$.} By Lemma \ref{lemma:gthresholdmixed}, we know that $G(z_1) = \Threshold$ for all $z_1 \ge \Threshold$. This means that it suffices to show that $\Quality_1(\Threshold) -\CostFn_1(z_1, \Threshold) - \Quality_1(z_1)$ is weakly decreasing in $z_1$. This is differentiable and we take the derivative to obtain $-\nabla_1 \CostFn_1(x, \Threshold) -\Quality'_1(x)$. Since $\Quality_1(x) - \CostFn_1(y, x) - \Quality_1(y) < 0$ for $x < y$ by (A1) and since $\CostFn_1(x, x) = 0$ for all $x$, we know that the function $y \mapsto \Quality_1(x) - \CostFn_1(y, x) - \Quality_1(y)$ is uniquely maximized at $y = x$ in the domain $y \ge x$. Since we know that $\CostFn_1$ is continuously differentiable up to the boundary, this means that $-\nabla_1 \CostFn_1(x, x) - \Quality'_1(x) \le 0$. Now, we use the mixed partial condition to conclude that:
\[-\nabla_1 \CostFn_1(x, \Threshold) - \Quality'_1(x) \le -\nabla_1 \CostFn_1(x, x) - \Quality'_1(x) \le 0\]
as desired. 

\end{proof}

Corollary \ref{cor:monotoneusage} now follows from Theorem \ref{thm:usagedecrease} and Lemma \ref{lemma:maxshifteddecreasing}.
\begin{proof}[Proof of Corollary \ref{cor:monotoneusage}]
If gaming costs $\CostFn = \CostFn^{\infty}$ are infinite, then the result follows directly from Theorem \ref{thm:usagedecrease}. If gaming costs $\CostFn$ satisfy the mixed partial condition $\frac{\partial^2 \CostFn_1 (z_1, z'_1)}{\partial z_1 \partial z'_1} \le 0$ for all $z'_1 \le z_1$ or if  $\Quality_1$ is weakly increasing, then by Lemma \ref{lemma:maxshifteddecreasing}, we know that $\max_{z'_1 \le \Threshold} (\Quality_1(z'_1) - \CostFn_1(z_1, z'_1)) - \Quality_1(z_1)$ is weakly decreasing in $z_1$ for $z_1 > \Threshold$. This means that there does not exist $z^u_1 > z^l_1  > \Threshold$ such that \eqref{eq:condition} holds. By Theorem \ref{thm:usagedecrease}, this means that the presence of the detector weakly decreases LLM usage. 
\end{proof}

\section{Additional material and Proofs for Section \ref{subsec:quality}}\label{appendix:quality}

\subsection{Additional result}

We show an analogue of Theorem \ref{thm:qualitydecrease} for quality increases and quality staying constant. 
\begin{theorem}
\label{thm:qualitygeneral}
Consider the setup of Theorem \ref{thm:qualitydecrease}. The presence of the detector increases quality for some user type $\TypeNoArg_I$: that is,  $\Quality( \FinalContent{\TypeNoArg_I}(\Penalty, \Threshold; \Quality, \CostFn)) > \Quality( \FinalContent{\TypeNoArg_I}(\emptyset; \Quality, \CostFn))$. It also induces no change in quality for some user type $\TypeNoArg_N$: that is, $\Quality( \FinalContent{\TypeNoArg_N}(\Penalty, \Threshold; \Quality, \CostFn)) = \Quality( \FinalContent{\TypeNoArg_N}(\emptyset; \Quality, \CostFn))$. 
\end{theorem}
Together, Theorem \ref{thm:qualitydecrease} and Theorem \ref{thm:qualitygeneral} illustrate that LLM detection can have a highly heterogeneous impact on output quality across users. This casts doubt about whether LLM detection is an effective intervention to improve output quality. 

\subsection{Proofs}
To prove Theorem \ref{thm:qualitydecrease} and Theorem \ref{thm:qualitygeneral}, we handle the infinite cost and finite cost cases separately.

\subsubsection{Infinite costs}

We split into sublemmas for quality increase, quality decrease, and constant quality. 

\begin{lemma}
\label{lemma:qualityinfiniteincrease}
Fix $\Penalty > 0$, $\Threshold < \infty$, $\Quality$, and infinite costs $\CostFn = \CostFn^{\infty}$. Suppose also  $D \ge 2$, that there exists a dimension $2 \le i' \le D$ such that $\Quality_{i'}$ is non-constant. Then, the presence of the detector increases quality for some user type $\TypeNoArg_I$,
\[ \Quality( \FinalContent{\TypeNoArg_I}(\Penalty, \Threshold; \Quality, \CostFn)) > \Quality( \FinalContent{\TypeNoArg_I}(\emptyset; \Quality, \CostFn)).\]
\end{lemma}
\begin{proof}
We let $\HumanContent{\TypeNoArg_I}_i = \LLMContent{\TypeNoArg_I}_i$ for all $i \neq 1, i'$. Since $\Quality_{i'}$ is non-constant and twice-differentiable concave, for every value $\epsilon$, there exists $b_{i'}(\epsilon)$ and $a_{i'}(\epsilon)$ such that $\sup_{z_i \in \text{conv}(a_{i'}(\epsilon), b_{i'}(\epsilon)))} \left|Q''_{i'}(z_i) \right| \cdot (b_{i'}(\epsilon) - a_{i'}(\epsilon))^2 < \epsilon$ and $\Quality_{i'}(a_{i'}(\epsilon)) \in (\Quality_{i'}(b_{i'}(\epsilon)), \Quality_{i'}(b_{i'}(\epsilon))+ \epsilon)$.

 Let $\TypeNoArg_I$ be such that $\HumanContent{\TypeNoArg_I}_1 = \Threshold - 1$, and let $\LLMContent{\TypeNoArg_I}_1 = \Threshold + \epsilon_1$ for an $\epsilon_1 > 0$ that we will set later. We let 
$\LLMContent{\TypeNoArg_I}_{i'} = b_{i'}(\epsilon_2)$ and we let $\HumanContent{\TypeNoArg_I}_{i'} = b_{i'}(\epsilon_2) + (a_{i'}(\epsilon_2) - b_{i'}(\epsilon_2)) \cdot \frac{1 + \epsilon_1}{\epsilon_1}$ for $\epsilon_2 > 0$ we will set later. We will construct $\ProductionCost{\TypeNoArg_I}$ later. 

\paragraph{Utility without the detector.}
For $\ProductionCost{\TypeNoArg_I} > -\Quality'_1(\Threshold + \epsilon_1) (1+ \epsilon_1) -\Quality'_{i'}(b_{i'}(\epsilon_2)) (b_{i'}(\epsilon_2) - a_{i'}(\epsilon_2)) \cdot \frac{1 + \epsilon_1}{\epsilon_1}$, we  claim that $\FractionNoArg^* = 1$. By concavity, it suffices to show that the derivative at $\FractionNoArg = 1$ is positive. 
By the definition of $\FractionNoArg^*$
\[\Quality(\Combination{\TypeNoArg_I}(\FractionNoArg)) - \ProductionCost{\TypeNoArg_I}(1- \FractionNoArg) = \Quality_1(\Combination{\TypeNoArg_I}_1(\FractionNoArg)) + \Quality_{i'}(\Combination{\TypeNoArg_I}_{i'}(\FractionNoArg)) - \ProductionCost{\TypeNoArg_I}(1- \FractionNoArg)  +  \sum_{i \neq i', 1} \Quality_i(\LLMContent{\TypeNoArg}). \]
Taking a derivative, we obtain:
\begin{align*}
&\Quality'_1(\LLMContent{\TypeNoArg_I}_1) (\LLMContent{\TypeNoArg_I}_1 - \HumanContent{\TypeNoArg_I}_1) + \Quality'_{i'}(\LLMContent{\TypeNoArg_I}_{i'}) (\LLMContent{\TypeNoArg_I}_{i'} - \HumanContent{\TypeNoArg_I}_{i'}) + \ProductionCost{\TypeNoArg_I}  \\
&= \Quality'_1(\Threshold + \epsilon_1) (1+ \epsilon_1) + \Quality'_{i'}(b_{i'}(\epsilon_2)) (b_{i'}(\epsilon_2) - a_{i'}(\epsilon_2)) \cdot \frac{1 + \epsilon_1}{\epsilon_1} + \ProductionCost{\TypeNoArg_I}  \\
&> 0
\end{align*}
as desired.

\paragraph{Quality analysis.}
If $\Quality_{i'}(a_{i'}(\epsilon_2)) - \Quality_{i'}(b_{i'}(\epsilon_2)) + \Quality_1(\Threshold) - \Quality_1(\Threshold + \epsilon_1)  > 0$, then we claim that $\Quality(\Combination{\TypeNoArg_I}(\FractionNoArg^{\text{game}}) > \Quality(\Combination{\TypeNoArg_I}(1)$. Note that $\FractionNoArg^{\text{game}} = \frac{1}{1+\epsilon_1}$. We know that:
\[\Quality(\Combination{\TypeNoArg_I}(\FractionNoArg^{\text{game}})) - \Quality(\Combination{\TypeNoArg_I}(1)) = \Quality_1(\Threshold) + \Quality_{i'}(a_{i'}(\epsilon_2)) - \Quality_1(\Threshold + \epsilon_1) - \Quality_{i'}(b_{i'}(\epsilon_2)) > 0. \]

\paragraph{Utility under detection.}
If $\Quality_{i'}(a_{i'}(\epsilon_2)) - \Quality_{i'}(b_{i'}(\epsilon_2)) + \Quality_1(\Threshold) - \Quality_1(\Threshold + \epsilon_1) - \ProductionCost{\TypeNoArg_I} \cdot \frac{\epsilon_1}{1+\epsilon_1} + \Penalty > 0$, then we claim that  $\max_{z}\Utility{\TypeNoArg_I}(\FractionNoArg^{\text{game}}, z; \Penalty, \Threshold,\Quality,\CostFn) > \max_{z}\Utility{\TypeNoArg_I}(1, z; \Penalty, \Threshold,\Quality,\CostFn)$. 
By Lemma \ref{lemma:infinitecostspostprocessing}, we also know that:
\begin{align*}
 &\max_{z}\Utility{\TypeNoArg_I}(\FractionNoArg^{\text{game}}, z; \Penalty, \Threshold,\Quality,\CostFn) - \max_{z}\Utility{\TypeNoArg_I}(1, z; \Penalty, \Threshold,\Quality,\CostFn) \\
 &= \Quality(\Combination{\TypeNoArg_I}(\FractionNoArg^{\text{game}})) + \Penalty - \Quality(\Combination{\TypeNoArg_I}(1)) - \ProductionCost{\TypeNoArg_I} (1 - \FractionNoArg^{\text{game}}) \\
 &= \Penalty + \Quality_{i'}(a_{i'}(\epsilon_2)) - \Quality_{i'}(b_{i'}(\epsilon_2)) + \Quality_1(\Threshold) - \Quality_1(\Threshold + \epsilon_1) - \ProductionCost{\TypeNoArg_I} \cdot \frac{\epsilon_1}{1+\epsilon_1} \\
 &> 0.
\end{align*}

\paragraph{Guarantees under conditions.}
Suppose that $\ProductionCost{\TypeNoArg_I} > -\Quality'_1(\Threshold + \epsilon_1) (1+ \epsilon_1) -\Quality'_{i'}(b_{i'}(\epsilon_2)) (b_{i'}(\epsilon_2) - a_{i'}(\epsilon_2)) \cdot \frac{1 + \epsilon_1}{\epsilon_1}$, $\Quality_{i'}(a_{i'}(\epsilon_2)) - \Quality_{i'}(b_{i'}(\epsilon_2)) + \Quality_1(\Threshold) - \Quality_1(\Threshold + \epsilon_1) > 0$, and $\Quality_{i'}(a_{i'}(\epsilon_2)) - \Quality_{i'}(b_{i'}(\epsilon_2)) + \Quality_1(\Threshold) - \Quality_1(\Threshold + \epsilon_1) - \ProductionCost{\TypeNoArg_I} \cdot \frac{\epsilon_1}{1+\epsilon_1} + \Penalty > 0$. By Theorem \ref{thm:infinitecostdetector}, we know that $\Fraction{\TypeNoArg_I}(\Penalty, \Threshold; \Quality, \CostFn^{\infty}) \in \left\{1, \FractionNoArg^{\text{game}}\right\}$, and the above arguments show that $\Fraction{\TypeNoArg_I}(\Penalty, \Threshold; \Quality, \CostFn^{\infty}) = \FractionNoArg^{\text{game}}$. Moreover, Theorem \ref{thm:infinitecostdetector} also tells us that $\FinalContent{\TypeNoArg_I}(\Penalty, \Threshold; \Quality, \CostFn^{\infty}) = \Combination{\TypeNoArg_I}(\FractionNoArg^{\text{game}})$. Moreover, by Theorem \ref{thm:nodetector}, we know that $\Fraction{\TypeNoArg_I}(\emptyset; \Quality, \CostFn^{\infty}) = \FractionNoArg^*$, , and the above arguments show that $\FractionNoArg^* = 1$. Moreover, Theorem \ref{thm:nodetector} also tells us
that $\FinalContent{\TypeNoArg_I}(\emptyset; \Quality, \CostFn^{\infty}) = \Combination{\TypeNoArg_I}(\FractionNoArg^*)$. Putting this together, and combining with the above arguments, we see that 
\[\Quality(\FinalContent{\TypeNoArg_I}(\Penalty, \Threshold; \Quality, \CostFn^{\infty})) = \Quality(\Combination{\TypeNoArg_I}(\Fraction{\TypeNoArg_I}(\Penalty, \Threshold; \Quality, \CostFn^{\infty})))  >  \Quality(\Combination{\TypeNoArg_I}(\Fraction{\TypeNoArg_I}(\emptyset; \Quality, \CostFn^{\infty}))) = \Quality(\FinalContent{\TypeNoArg_I}(\emptyset; \Quality, \CostFn^{\infty}))\]
as desired.

\paragraph{Condition analysis.}
It now suffices to construct $\epsilon_1, \epsilon_2$ and $\ProductionCost{\TypeNoArg_I}$ that satisfy 
\[\Quality_{i'}(a_{i'}(\epsilon_2)) - \Quality_{i'}(b_{i'}(\epsilon_2)) + \Quality_1(\Threshold) - \Quality_1(\Threshold + \epsilon_1) - \ProductionCost{\TypeNoArg_I} \cdot \frac{\epsilon_1}{1+\epsilon_1} + \Penalty > 0\]
and
\[\Quality_{i'}(a_{i'}(\epsilon_2)) - \Quality_{i'}(b_{i'}(\epsilon_2)) + \Quality_1(\Threshold) - \Quality_1(\Threshold + \epsilon_1)  > 0\]
and 
\[\ProductionCost{\TypeNoArg_I} > -\Quality'_1(\Threshold + \epsilon_1) (1+ \epsilon_1) -\Quality'_{i'}(b_{i'}(\epsilon_2)) (b_{i'}(\epsilon_2) - a_{i'}(\epsilon_2)) \cdot \frac{1 + \epsilon_1}{\epsilon_1}.\] 
We can rewrite these as:
\[\ProductionCost{\TypeNoArg_I} \cdot \frac{\epsilon_1}{1+\epsilon_1} < \Penalty + \Quality_{i'}(a_{i'}(\epsilon_2)) - \Quality_{i'}(b_{i'}(\epsilon_2)) + \Quality_1(\Threshold) - \Quality_1(\Threshold + \epsilon_1) \]
and
\[\Quality_{i'}(a_{i'}(\epsilon_2)) - \Quality_{i'}(b_{i'}(\epsilon_2))> \Quality_1(\Threshold + \epsilon_1)-  \Quality_1(\Threshold) \]
and 
\[\ProductionCost{\TypeNoArg_I} \cdot \frac{\epsilon_1}{1+\epsilon_1} > -\epsilon_1 \cdot \Quality'_1(\Threshold + \epsilon_1) + \Quality'_{i'}(b_{i'}(\epsilon_2)) (a_{i'}(\epsilon_2) - b_{i'}(\epsilon_2)).\] 
We can construct $\ProductionCost{\TypeNoArg_I} \ge 0$ as long as:
\[\Penalty + \Quality_{i'}(a_{i'}(\epsilon_2)) - \Quality_{i'}(b_{i'}(\epsilon_2)) + \Quality_1(\Threshold) - \Quality_1(\Threshold + \epsilon_1) > -\epsilon_1 \cdot \Quality'_1(\Threshold + \epsilon_1) + \Quality'_{i'}(b_{i'}(\epsilon_2)) (a_{i'}(\epsilon_2) - b_{i'}(\epsilon_2)).\] 
and 
\[\Quality_{i'}(a_{i'}(\epsilon_2)) - \Quality_{i'}(b_{i'}(\epsilon_2))> \Quality_1(\Threshold + \epsilon_1)-  \Quality_1(\Threshold). \]
This is equivalent to:
\[\Quality_{i'}(a_{i'}(\epsilon_2)) - \Quality_{i'}(b_{i'}(\epsilon_2)) > \Quality_1(\Threshold + \epsilon_1) - \Quality_1(\Threshold) \]
and 
\[\Penalty + \Quality_{i'}(a_{i'}(\epsilon_2)) - \Quality_{i'}(b_{i'}(\epsilon_2)) - \Quality'_{i'}(b_{i'}(\epsilon_2)) (a_{i'}(\epsilon_2) - b_{i'}(\epsilon_2))  > \Quality_1(\Threshold + \epsilon_1) - \Quality_1(\Threshold) -\epsilon_1 \cdot \Quality'_1(\Threshold + \epsilon_1).  \]

Since $\Quality_{i'}$ is concave, we know that:
\begin{align*}
&\Quality_{i'}(a_{i'}(\epsilon_2)) - \Quality_{i'}(b_{i'}(\epsilon_2)) - \Quality'_{i'}(b_{i'}(\epsilon_2)) (a_{i'}(\epsilon_2) - b_{i'}(\epsilon_2)) \\
&= \Quality_{i'}(a_{i'}(\epsilon_2)) - \Quality_{i'}(b_{i'}(\epsilon_2))
+ \Quality'_{i'}(a_{i'}(\epsilon_2)) (b_{i'}(\epsilon_2) - a_{i'}(\epsilon_2)) \\
& - \Quality'_{i'}(b_{i'}(\epsilon_2)) (a_{i'}(\epsilon_2) - b_{i'}(\epsilon_2)) + \Quality'_{i'}(a_{i'}(\epsilon_2)) (a_{i'}(\epsilon_2) - b_{i'}(\epsilon_2)) \\
&\ge  - 0.5 \sup_{z_i \in \text{conv}(a_{i'}(\epsilon_2), b_{i'}(\epsilon_2)))} \left|Q''_{i'}(z_i) \right| \cdot (b_{i'}(\epsilon_2) - a_{i'}(\epsilon_2))^2.  
\end{align*}
Thus it suffices to have that:
\[\Quality_{i'}(a_{i'}(\epsilon_2)) - \Quality_{i'}(b_{i'}(\epsilon_2)) > \Quality_1(\Threshold + \epsilon_1) - \Quality_1(\Threshold) \]
and 
\[\Penalty - 0.5 \sup_{z_i \in \text{conv}(a_{i'}(\epsilon_2), b_{i'}(\epsilon_2)))} \left|Q''_{i'}(z_i) \right| \cdot (b_{i'}(\epsilon_2) - a_{i'}(\epsilon_2))^2  > \Quality_1(\Threshold + \epsilon_1) - \Quality_1(\Threshold) -\epsilon_1 \cdot \Quality'_1(\Threshold + \epsilon_1).  \]

\paragraph{Parameter settings.} We can set $\epsilon_2 = \Penalty$, which means that $\sup_{z_i \in \text{conv}(a_{i'}(\epsilon_2), b_{i'}(\epsilon_2)))} \left|Q''_{i'}(z_i) \right| \cdot (b_{i'}(\epsilon_2) - a_{i'}(\epsilon_2))^2 < \Penalty$ and $\Quality_{i'}(a_{i'}(\epsilon_2)) > \Quality_{i'}(b_{i'}(\epsilon_2))$. As $\epsilon_1 \rightarrow 0$, we know that $\Quality_1(\Threshold + \epsilon_1) - \Quality_1(\Threshold) \rightarrow 0$ and  $\Quality_1(\Threshold + \epsilon_1) - \Quality_1(\Threshold) -\epsilon_1 \cdot \Quality'_1(\Threshold + \epsilon_1) \rightarrow 0$. Thus, we can  choose $\epsilon_1$ sufficiently small to satisfy the desired expressions. 

\end{proof}

\begin{lemma}
\label{lemma:qualityinfinitedecrease}
Fix $\Penalty > 0$, $\Threshold < \infty$, $\Quality$, and infinite costs $\CostFn = \CostFn^{\infty}$. Suppose also  $D \ge 2$, that there exists a dimension $2 \le i' \le D$ such that $\Quality_{i'}$ is non-constant. Then, the presence of the detector decreases quality for some user type $\TypeNoArg_D$ such that 
\[ \Quality( \FinalContent{\TypeNoArg_D}(\Penalty, \Threshold; \Quality, \CostFn)) < \Quality( \FinalContent{\TypeNoArg_D}(\emptyset; \Quality, \CostFn)).\]
\end{lemma}
\begin{proof}
We let $\HumanContent{\TypeNoArg_D}_i = \LLMContent{\TypeNoArg_D}_i$ for all $i \neq 1, i'$. Since $\Quality_{i'}$ is non-constant and twice-differentiable concave, for every value $\epsilon$, there exists $b_{i'}(\epsilon)$ and $a_{i'}(\epsilon)$ such that $\sup_{z_i \in \text{conv}(a_{i'}(\epsilon), b_{i'}(\epsilon)))} \left|Q''_{i'}(z_i) \right| \cdot (b_{i'}(\epsilon) - a_{i'}(\epsilon))^2 < \epsilon$ and $\Quality_{i'}(a_{i'}(\epsilon)) \in (\Quality_{i'}(b_{i'}(\epsilon)), \Quality_{i'}(b_{i'}(\epsilon))+ \epsilon)$.

Let $\TypeNoArg_D$ be such that $\HumanContent{\TypeNoArg_D}_1 = \Threshold - 1$, and let $\LLMContent{\TypeNoArg_D}_1 = \Threshold + \epsilon_1$ for an $\epsilon_1 > 0$ that we will set later. We let 
$\LLMContent{\TypeNoArg_D}_{i'} = a_{i'}(\epsilon_2)$ and we let $\HumanContent{\TypeNoArg_D}_{i'} = a_{i'}(\epsilon_2) + (b_{i'}(\epsilon_2) - a_{i'}(\epsilon_2)) \cdot \frac{1 + \epsilon_1}{\epsilon_1}$ for $\epsilon_2 > 0$ we will set later. We will construct $\ProductionCost{\TypeNoArg_D}$ later. 

\paragraph{Utility without the detector.}
For $\ProductionCost{\TypeNoArg_D} > -\Quality'_1(\Threshold + \epsilon_1) (1+ \epsilon_1) -\Quality'_{i'}(a_{i'}(\epsilon_2)) (a_{i'}(\epsilon_2) - b_{i'}(\epsilon_2)) \cdot \frac{1 + \epsilon_1}{\epsilon_1}$, we  claim that $\FractionNoArg^* = 1$. By concavity, it suffices to show that the derivative at $\FractionNoArg = 1$ is positive. 
By the definition of $\FractionNoArg^*$
\[\Quality(\Combination{\TypeNoArg_D}(\FractionNoArg)) - \ProductionCost{\TypeNoArg_D}(1- \FractionNoArg) = \Quality_1(\Combination{\TypeNoArg_D}_1(\FractionNoArg)) + \Quality_{i'}(\Combination{\TypeNoArg_D}_{i'}(\FractionNoArg)) - \ProductionCost{\TypeNoArg_D}(1- \FractionNoArg)  +  \sum_{i \neq i', 1} \Quality_i(\LLMContent{\TypeNoArg}). \]
Taking a derivative, we obtain:
\begin{align*}
&\Quality'_1(\LLMContent{\TypeNoArg_D}_1) (\LLMContent{\TypeNoArg_D}_1 - \HumanContent{\TypeNoArg_D}_1) + \Quality'_{i'}(\LLMContent{\TypeNoArg_D}_{i'}) (\LLMContent{\TypeNoArg_D}_{i'} - \HumanContent{\TypeNoArg_D}_{i'}) + \ProductionCost{\TypeNoArg_D}  \\
&= \Quality'_1(\Threshold + \epsilon_1) (1+ \epsilon_1) + \Quality'_{i'}(a_{i'}(\epsilon_2)) (a_{i'}(\epsilon_2) - b_{i'}(\epsilon_2)) \cdot \frac{1 + \epsilon_1}{\epsilon_1} + \ProductionCost{\TypeNoArg_D}  \\
&> 0
\end{align*}
as desired.

\paragraph{Quality analysis.} If $\Quality_{i'}(b_{i'}(\epsilon_2)) - \Quality_{i'}(a_{i'}(\epsilon_2)) + \Quality_1(\Threshold) - \Quality_1(\Threshold + \epsilon_1)  < 0$, then we claim that $\Quality(\Combination{\TypeNoArg_D}(\FractionNoArg^{\text{game}})) < \Quality(\Combination{\TypeNoArg_D}(1)$. Note that $\FractionNoArg^{\text{game}} = \frac{1}{1+\epsilon_1}$. We know that:
\[\Quality(\Combination{\TypeNoArg_D}(\FractionNoArg^{\text{game}})) - \Quality(\Combination{\TypeNoArg_D}(1)) = \Quality_1(\Threshold) + \Quality_{i'}(b_{i'}(\epsilon_2)) - \Quality_1(\Threshold + \epsilon_1) - \Quality_{i'}(a_{i'}(\epsilon_2)) < 0. \]

\paragraph{Utility under detection.}
If $\Quality_{i'}(b_{i'}(\epsilon_2)) - \Quality_{i'}(a_{i'}(\epsilon_2)) + \Quality_1(\Threshold) - \Quality_1(\Threshold + \epsilon_1) - \ProductionCost{\TypeNoArg_D} \cdot \frac{\epsilon_1}{1+\epsilon_1} + \Penalty > 0$, then we claim that  $\max_{z}\Utility{\TypeNoArg_D}(\FractionNoArg^{\text{game}}, z; \Penalty, \Threshold,\Quality,\CostFn) > \max_{z}\Utility{\TypeNoArg_D}(1, z; \Penalty, \Threshold,\Quality,\CostFn)$. 
By Lemma \ref{lemma:infinitecostspostprocessing}, we also know that:
\begin{align*}
 &\max_{z}\Utility{\TypeNoArg_D}(\FractionNoArg^{\text{game}}, z; \Penalty, \Threshold,\Quality,\CostFn) - \max_{z}\Utility{\TypeNoArg_D}(1, z; \Penalty, \Threshold,\Quality,\CostFn) \\
 &= \Quality(\Combination{\TypeNoArg_D}(\FractionNoArg^{\text{game}})) + \Penalty - \Quality(\Combination{\TypeNoArg_D}(1)) - \ProductionCost{\TypeNoArg_D} (1 - \FractionNoArg^{\text{game}}) \\
 &= \Penalty + \Quality_{i'}(b_{i'}(\epsilon_2)) - \Quality_{i'}(a_{i'}(\epsilon_2)) + \Quality_1(\Threshold) - \Quality_1(\Threshold + \epsilon_1) - \ProductionCost{\TypeNoArg_D} \cdot \frac{\epsilon_1}{1+\epsilon_1} \\
 &> 0.
\end{align*}

\paragraph{Guarantees under conditions.}
Suppose that $\ProductionCost{\TypeNoArg_D} > -\Quality'_1(\Threshold + \epsilon_1) (1+ \epsilon_1) -\Quality'_{i'}(a_{i'}(\epsilon_2)) (a_{i'}(\epsilon_2) - b_{i'}(\epsilon_2)) \cdot \frac{1 + \epsilon_1}{\epsilon_1}$, $\Quality_{i'}(b_{i'}(\epsilon_2)) - \Quality_{i'}(a_{i'}(\epsilon_2)) + \Quality_1(\Threshold) - \Quality_1(\Threshold + \epsilon_1) < 0$, and $\Quality_{i'}(b_{i'}(\epsilon_2)) - \Quality_{i'}(a_{i'}(\epsilon_2)) + \Quality_1(\Threshold) - \Quality_1(\Threshold + \epsilon_1) - \ProductionCost{\TypeNoArg_D} \cdot \frac{\epsilon_1}{1+\epsilon_1} + \Penalty > 0$. By Theorem \ref{thm:infinitecostdetector}, we know that $\Fraction{\TypeNoArg_D}(\Penalty, \Threshold; \Quality, \CostFn^{\infty}) \in \left\{1, \FractionNoArg^{\text{game}}\right\}$, and the above arguments show that $\Fraction{\TypeNoArg_D}(\Penalty, \Threshold; \Quality, \CostFn^{\infty}) = \FractionNoArg^{\text{game}}$. Moreover, Theorem \ref{thm:infinitecostdetector} also tells us that $\FinalContent{\TypeNoArg_D}(\Penalty, \Threshold; \Quality, \CostFn^{\infty}) = \Combination{\TypeNoArg_D}(\FractionNoArg^{\text{game}})$. Moreover, by Theorem \ref{thm:nodetector}, we know that $\Fraction{\TypeNoArg_D}(\emptyset; \Quality, \CostFn^{\infty}) = \FractionNoArg^*$, , and the above arguments show that $\FractionNoArg^* = 1$. Moreover, Theorem \ref{thm:nodetector} also tells us
that $\FinalContent{\TypeNoArg_D}(\emptyset; \Quality, \CostFn^{\infty}) = \Combination{\TypeNoArg_D}(\FractionNoArg^*)$. Putting this together, and combining with the above arguments, we see that 
\[\Quality(\FinalContent{\TypeNoArg_D}(\Penalty, \Threshold; \Quality, \CostFn^{\infty})) = \Quality(\Combination{\TypeNoArg_D}(\Fraction{\TypeNoArg_D}(\Penalty, \Threshold; \Quality, \CostFn^{\infty}))  <  \Quality(\Combination{\TypeNoArg_D}(\Fraction{\TypeNoArg_D}(\emptyset; \Quality, \CostFn^{\infty})) = \Quality(\FinalContent{\TypeNoArg_D}(\emptyset; \Quality, \CostFn^{\infty}))\]
as desired.

\paragraph{Condition analysis.}
It now suffices to construct $\epsilon_1, \epsilon_2$ and $\ProductionCost{\TypeNoArg_D}$ that satisfy 
\[\Quality_{i'}(b_{i'}(\epsilon_2)) - \Quality_{i'}(a_{i'}(\epsilon_2)) + \Quality_1(\Threshold) - \Quality_1(\Threshold + \epsilon_1) - \ProductionCost{\TypeNoArg_D} \cdot \frac{\epsilon_1}{1+\epsilon_1} + \Penalty > 0\]
and
\[\Quality_{i'}(b_{i'}(\epsilon_2)) - \Quality_{i'}(a_{i'}(\epsilon_2)) + \Quality_1(\Threshold) - \Quality_1(\Threshold + \epsilon_1) < 0\]
and 
\[\ProductionCost{\TypeNoArg_D} > -\Quality'_1(\Threshold + \epsilon_1) (1+ \epsilon_1) -\Quality'_{i'}(a_{i'}(\epsilon_2)) (a_{i'}(\epsilon_2) - b_{i'}(\epsilon_2)) \cdot \frac{1 + \epsilon_1}{\epsilon_1}.\] 
We can rewrite these as:
\[\ProductionCost{\TypeNoArg_D} \cdot \frac{\epsilon_1}{1+\epsilon_1} < \Penalty + \Quality_{i'}(b_{i'}(\epsilon_2)) - \Quality_{i'}(a_{i'}(\epsilon_2)) + \Quality_1(\Threshold) - \Quality_1(\Threshold + \epsilon_1) \]
and
\[\Quality_{i'}(b_{i'}(\epsilon_2)) - \Quality_{i'}(a_{i'}(\epsilon_2)) < \Quality_1(\Threshold + \epsilon_1)-  \Quality_1(\Threshold) \]
and 
\[\ProductionCost{\TypeNoArg_D} \cdot \frac{\epsilon_1}{1+\epsilon_1} > -\epsilon_1 \cdot \Quality'_1(\Threshold + \epsilon_1) + \Quality'_{i'}(a_{i'}(\epsilon_2)) (b_{i'}(\epsilon_2) - a_{i'}(\epsilon_2)).\] 
 We can construct $\ProductionCost{\TypeNoArg_D} \ge 0$ as long as:
\[\Penalty + \Quality_{i'}(b_{i'}(\epsilon_2)) - \Quality_{i'}(a_{i'}(\epsilon_2)) + \Quality_1(\Threshold) - \Quality_1(\Threshold + \epsilon_1) > -\epsilon_1 \cdot \Quality'_1(\Threshold + \epsilon_1) + \Quality'_{i'}(a_{i'}(\epsilon_2)) (b_{i'}(\epsilon_2) - a_{i'}(\epsilon_2)).\] 
and 
\[-\Penalty +  \Quality_1(\Threshold + \epsilon_1)-  \Quality_1(\Threshold) < \Quality_{i'}(b_{i'}(\epsilon_2)) - \Quality_{i'}(a_{i'}(\epsilon_2)) < \Quality_1(\Threshold + \epsilon_1)-  \Quality_1(\Threshold). \]
This is equivalent to:
\[\Penalty -  \left(\Quality_1(\Threshold + \epsilon_1)-  \Quality_1(\Threshold)\right) > \Quality_{i'}(a_{i'}(\epsilon_2)) - \Quality_{i'}(b_{i'}(\epsilon_2)) > -\left(\Quality_1(\Threshold + \epsilon_1)-  \Quality_1(\Threshold)\right). \]
and 
\[\Penalty + \Quality_{i'}(b_{i'}(\epsilon_2)) - \Quality_{i'}(a_{i'}(\epsilon_2)) - \Quality'_{i'}(a_{i'}(\epsilon_2)) (b_{i'}(\epsilon_2) - a_{i'}(\epsilon_2))  > \Quality_1(\Threshold + \epsilon_1) - \Quality_1(\Threshold) -\epsilon_1 \cdot \Quality'_1(\Threshold + \epsilon_1).  \]

Since $\Quality_{i'}$ is concave, we know that:
\begin{align*}
&\Quality_{i'}(b_{i'}(\epsilon_2)) - \Quality_{i'}(a_{i'}(\epsilon_2)) - \Quality'_{i'}(a_{i'}(\epsilon_2)) (b_{i'}(\epsilon_2) - a_{i'}(\epsilon_2)) \\
&= \Quality_{i'}(b_{i'}(\epsilon_2)) - \Quality_{i'}(a_{i'}(\epsilon_2)) - \Quality'_{i'}(b_{i'}(\epsilon_2)) (b_{i'}(\epsilon_2) - a_{i'}(\epsilon_2))\\
&- \Quality'_{i'}(a_{i'}(\epsilon_2)) (b_{i'}(\epsilon_2) - a_{i'}(\epsilon_2)) + \Quality'_{i'}(b_{i'}(\epsilon_2)) (b_{i'}(\epsilon_2) - a_{i'}(\epsilon_2)) \\
&\ge - 0.5 \sup_{z_i \in \text{conv}(a_{i'}(\epsilon_2), b_{i'}(\epsilon_2)))} \left|Q''_{i'}(z_i) \right| \cdot (b_{i'}(\epsilon_2) - a_{i'}(\epsilon_2))^2.  
\end{align*}

\paragraph{Parameter settings.}
We can set $\epsilon_2 = \Penalty / 2$, which means that $\sup_{z_i \in \text{conv}(a_{i'}(\epsilon_2), b_{i'}(\epsilon_2)))} \left|Q''_{i'}(z_i) \right| \cdot (b_{i'}(\epsilon_2) - a_{i'}(\epsilon_2))^2 < \Penalty/2$ and $\Quality_{i'}(a_{i'}(\epsilon_2)) \in  (\Quality_{i'}(b_{i'}(\epsilon_2)), \Quality_{i'}(b_{i'}(\epsilon_2)) + \Penalty/2)$. As $\epsilon_1 \rightarrow 0$, we know that $\Quality_1(\Threshold + \epsilon_1) - \Quality_1(\Threshold) \rightarrow 0$ and  $\Quality_1(\Threshold + \epsilon_1) - \Quality_1(\Threshold) -\epsilon_1 \cdot \Quality'_1(\Threshold + \epsilon_1) \rightarrow 0$. Thus, we can  choose $\epsilon_1$ sufficiently small to satisfy the desired expressions. 

\end{proof}

\begin{lemma}
\label{lemma:qualityinfiniteconstant}
Fix $\Penalty > 0$, $\Threshold < \infty$, $\Quality$, and infinite costs $\CostFn = \CostFn^{\infty}$. Then, the presence of the detector keeps quality the same for some user type $\TypeNoArg_N$ 
\[ \Quality( \FinalContent{\TypeNoArg_N}(\Penalty, \Threshold; \Quality, \CostFn)) = \Quality( \FinalContent{\TypeNoArg_N}(\emptyset; \Quality, \CostFn)).\]
\end{lemma}
\begin{proof}
Let $\TypeNoArg_N$ be such that $\HumanContent{\TypeNoArg_N}_1 <  \LLMContent{\TypeNoArg_N}_1 < \Threshold$, and let $\HumanContent{\TypeNoArg_N}_{i} = \LLMContent{\TypeNoArg_N}_{i}$ for all $i \ge 2$. By Theorem \ref{thm:infinitecostdetector}, we know that $\Fraction{\TypeNoArg_N}(\Penalty, \Threshold; \Quality, \CostFn^{\infty}) = \FractionNoArg^*$ and $\FinalContent{\TypeNoArg_N}(\Penalty, \Threshold; \Quality, \CostFn^{\infty}) = \Combination{\TypeNoArg_N}(\FractionNoArg^*)$. By Theorem \ref{thm:nodetector}, we know that $\Fraction{\TypeNoArg_N}(\emptyset; \Quality, \CostFn^{\infty}) = \FractionNoArg^*$ and $\FinalContent{\TypeNoArg_N}(\emptyset; \Quality, \CostFn^{\infty}) = \Combination{\TypeNoArg_N}(\FractionNoArg^*)$. This implies that $\Quality(\FinalContent{\TypeNoArg_N}(\Penalty, \Threshold; \Quality, \CostFn^{\infty})) = \Quality(\FinalContent{\TypeNoArg_N}(\emptyset; \Quality, \CostFn^{\infty}))$ as desired. 
\end{proof}

\subsubsection{Finite costs}

We split into sublemmas for quality increase, quality decrease, and constant quality.

\begin{lemma}
\label{lemma:qualityfiniteincrease}
 Fix $\Penalty > 0$, $\Threshold < \infty$, $\Quality$, and finite costs $\CostFn$. Suppose that the second or third condition in Corollary \ref{cor:monotoneusage} holds. Suppose also that $\nabla_1 (\CostFn_1(\Threshold, \Threshold)) > \max(0, -\Quality'_1(\Threshold))$.  Suppose also  $D \ge 2$, that there exists a dimension $2 \le i' \le D$ such that $\Quality_{i'}$ is non-constant. Then, the presence of the detector increases quality for some user type $\TypeNoArg_I$,
\[ \Quality( \FinalContent{\TypeNoArg_I}(\Penalty, \Threshold; \Quality, \CostFn)) > \Quality( \FinalContent{\TypeNoArg_I}(\emptyset; \Quality, \CostFn)).\] 
\end{lemma}
\begin{proof}

We let $\HumanContent{\TypeNoArg_I}_i = \LLMContent{\TypeNoArg_I}_i$ for all $i \neq 1, i'$. By applying the assumptions in the lemma along with Lemma \ref{lemma:gthresholdmixed}
in the mixed-partial case and Lemma \ref{lemma:gthreshold} in the weakly-increasing $Q_1$ case, we know that $G(z_1) = \Threshold$ for all $z_1 \ge \Threshold$.

Define the type such that 
$\HumanContent{\TypeNoArg_I}_1 = \Threshold - 1$, and let $\LLMContent{\TypeNoArg_I}_1 = \Threshold + \epsilon_1$ for an $\epsilon_1 > 0$ that we will set later. We will also later construct $0 < \epsilon_3 < \epsilon_1$ and $\ProductionCost{\TypeNoArg_I}$. Let $\bar{\FractionNoArg}_3 = \frac{1}{1+ \epsilon_3}$. Let $\kappa = \nabla_1 (\CostFn_1(\Threshold, \Threshold))$. 

We will use the following fact. For any $0 \le \epsilon < \epsilon_1$, let $\alpha = \frac{1}{1+\epsilon}$. For any $\epsilon_4 > 0$, if  $\epsilon_1$ is sufficiently small, and letting $\kappa = \nabla_1 (\CostFn_1(\Threshold, \Threshold))$, it holds that: 
\begin{align*}
&\min\left(\left(\Quality_1(\Combination{\TypeNoArg_I}_1(\FractionNoArg)) - \max_{z'_1 \le \Threshold} \left(\Quality_1(z'_1) - \CostFn_1(\Combination{\TypeNoArg_I}_1(\FractionNoArg), z'_1) \right)  \right), \Penalty \right) \\ 
 &=_{(B)} \min\left(\left(\Quality_1(\Combination{\TypeNoArg_I}_1(\FractionNoArg)) - \Quality_1(\Threshold) + \CostFn_1(\Combination{\TypeNoArg_I}_1(\FractionNoArg), \Threshold)  \right), \Penalty \right) \\
 &=_{(C)} \Quality_1(\Combination{\TypeNoArg_I}_1(\FractionNoArg)) - \Quality_1(\Threshold) + \CostFn_1(\Combination{\TypeNoArg_I}_1(\FractionNoArg), \Threshold) \\
 &= \Quality_1(\Combination{\TypeNoArg_I}_1(\FractionNoArg)) - \Quality_1(\Threshold) + \CostFn_1(\Combination{\TypeNoArg_I}_1(\FractionNoArg), \Threshold) - \CostFn_1(\Threshold, \Threshold) \\
 &\ge \Quality_1\left(\Threshold + \frac{\epsilon_1 - \epsilon}{1 + \epsilon}\right) - \Quality_1(\Threshold) + \left(\inf_{x \in [\Threshold, \Combination{\TypeNoArg_I}_1(\FractionNoArg)]} (\nabla_1 \CostFn_1(x, \Threshold))  \right) \cdot \left(\Combination{\TypeNoArg_I}_1(\FractionNoArg) - \Threshold \right) \\ 
 &\ge_{(D)} \Quality_1\left(\Threshold + \frac{\epsilon_1 - \epsilon}{1 + \epsilon}\right) - \Quality_1(\Threshold) + (\kappa - \epsilon_4) \cdot  \left(\Combination{\TypeNoArg_I}_1(\FractionNoArg) - \Threshold) \right) \\ 
&= \Quality_1\left(\Threshold + \frac{\epsilon_1 - \epsilon}{1 + \epsilon}\right) - \Quality_1(\Threshold) + (\kappa - \epsilon_4) \cdot \frac{\epsilon_1 - \epsilon}{1 + \epsilon}.  
\end{align*}
where (B) uses the fact that $G(z_1) = \Threshold$ for all $z_1 \ge \Threshold$, and (C) uses that $\epsilon_1$ is sufficiently small coupled with the continuity of $\Quality_1$ and assumption (A3) and the continuity of $\CostFn_1$ in its first argument, and (D) also uses that $\epsilon_1$ is sufficiently small coupled with the assumption that  $\CostFn$
is twice continuously differentiable. 

Since \(Q_{i'}\) is non-constant and concave, there exists an interval \(I\)
on which \(Q'_{i'}\) has constant nonzero sign. Thus, for all sufficiently
small \(\epsilon_2>0\), we can choose
\(a_{i'}(\epsilon_2),b_{i'}(\epsilon_2)\in I\) such that
\[
|a_{i'}(\epsilon_2)-b_{i'}(\epsilon_2)|=\epsilon_2,
\qquad
Q_{i'}(a_{i'}(\epsilon_2))>Q_{i'}(b_{i'}(\epsilon_2)).
\]
Moreover, there exist constants \(0<k_1<k_2<\infty\) and \(k_3<\infty\) such that,
for all sufficiently small \(\epsilon_2\),
\[
k_1\le |Q'_{i'}(z)|\le k_2,\qquad |Q''_{i'}(z)|\le k_3
\]
for all \(z\in \operatorname{conv}(a_{i'}(\epsilon_2),b_{i'}(\epsilon_2))\).

We let 
$\LLMContent{\TypeNoArg_I}_{i'} = b_{i'}(\epsilon_2)$ and we let $\HumanContent{\TypeNoArg_I}_{i'} = b_{i'}(\epsilon_2) + (a_{i'}(\epsilon_2) - b_{i'}(\epsilon_2)) \cdot \frac{1 + \epsilon_1}{\epsilon_1}$ for $\epsilon_2 > 0$.

\paragraph{Utility with no detection.} For $\ProductionCost{\TypeNoArg_I} > -\Quality'_1(\Threshold + \epsilon_1) (1+ \epsilon_1) -\Quality'_{i'}(b_{i'}(\epsilon_2)) (b_{i'}(\epsilon_2) - a_{i'}(\epsilon_2)) \cdot \frac{1 + \epsilon_1}{\epsilon_1}$, we  claim that $\FractionNoArg^* = 1$. By concavity, it suffices to show that the derivative at $\FractionNoArg = 1$ is positive. 
By the definition of $\FractionNoArg^*$
\[\Quality(\Combination{\TypeNoArg_I}(\FractionNoArg)) - \ProductionCost{\TypeNoArg_I}(1- \FractionNoArg) = \Quality_1(\Combination{\TypeNoArg_I}_1(\FractionNoArg)) + \Quality_{i'}(\Combination{\TypeNoArg_I}_{i'}(\FractionNoArg)) - \ProductionCost{\TypeNoArg_I}(1- \FractionNoArg)  +  \sum_{i \neq i', 1} \Quality_i(\LLMContent{\TypeNoArg}). \]
Taking a derivative, we obtain:
\begin{align*}
&\Quality'_1(\LLMContent{\TypeNoArg_I}_1) (\LLMContent{\TypeNoArg_I}_1 - \HumanContent{\TypeNoArg_I}_1) + \Quality'_{i'}(\LLMContent{\TypeNoArg_I}_{i'}) (\LLMContent{\TypeNoArg_I}_{i'} - \HumanContent{\TypeNoArg_I}_{i'}) + \ProductionCost{\TypeNoArg_I}  \\
&= \Quality'_1(\Threshold + \epsilon_1) (1+ \epsilon_1) + \Quality'_{i'}(b_{i'}(\epsilon_2)) (b_{i'}(\epsilon_2) - a_{i'}(\epsilon_2)) \cdot \frac{1 + \epsilon_1}{\epsilon_1} + \ProductionCost{\TypeNoArg_I}  \\
&> 0
\end{align*}
as desired.

\paragraph{Quality analysis.}
If $\Quality_1(\Threshold) - \Quality_1(\Threshold + \epsilon_1) + \Quality_{i'}\left(b_{i'}(\epsilon_2) + \frac{(1-\bar{\FractionNoArg}_3)(1+\epsilon_1)}{\epsilon_1} \cdot (a_{i'}(\epsilon_2) - b_{i'}(\epsilon_2)) \right) - \Quality_{i'}(b_{i'}(\epsilon_2)) > 0$, then we claim that $\Quality(\FinalContent{\TypeNoArg_I}(\FractionNoArg)) > \Quality(\Combination{\TypeNoArg_I}(1))$ for all $\FractionNoArg \in \left\{\FractionNoArg^{\text{game}}\right\} \cup ([\FractionNoArg^{\text{game}}, \bar{\FractionNoArg}_3] \cap A^{\TypeNoArg}(\Penalty))$. By Lemma \ref{lemma:detectorfinitecostspostprocessing}, we know that:
\begin{align*}
&\Quality(\FinalContent{\TypeNoArg_I}(\FractionNoArg)) - \Quality(\Combination{\TypeNoArg_I}(1)) \\ &= \Quality_1(\text{argmax}_{z'_1 \le \Threshold}\left(\Quality_1(z'_1) - \CostFn_1(\Combination{\TypeNoArg_I}(\FractionNoArg), z'_1 \right))) - \Quality_1(\Threshold + \epsilon_1) \\
&+ \Quality_{i'}\left(b_{i'}(\epsilon_2) + \frac{(1-\FractionNoArg)(1+\epsilon_1)}{\epsilon_1} \cdot (a_{i'}(\epsilon_2) - b_{i'}(\epsilon_2)) \right) - \Quality_{i'}(b_{i'}(\epsilon_2)) \\
&\ge_{(B)} \Quality_1(\Threshold) - \Quality_1(\Threshold + \epsilon_1) + \Quality_{i'}\left(b_{i'}(\epsilon_2) + \frac{(1-\bar{\FractionNoArg}_3)(1+\epsilon_1)}{\epsilon_1} \cdot (a_{i'}(\epsilon_2) - b_{i'}(\epsilon_2)) \right) - \Quality_{i'}(b_{i'}(\epsilon_2)) \\
&> 0
\end{align*}
where (B) uses the fact that $G(z_1) = \Threshold$ for all $z_1 \ge \Threshold$.  

\paragraph{Utility with detection.}
If $\Quality_{i'}(a_{i'}(\epsilon_2))  - \Quality_{i'}\left(b_{i'}(\epsilon_2) + \frac{\epsilon (1+\epsilon_1)}{\epsilon_1(1+\epsilon)} \cdot (a_{i'}(\epsilon_2) - b_{i'}(\epsilon_2)) \right) -  \ProductionCost{\TypeNoArg_I} \left(\frac{1}{1+\epsilon} - \FractionNoArg^{\text{game}}\right) + (\kappa - \epsilon_4) \cdot \frac{\epsilon_1 - \epsilon}{1 + \epsilon} > 0$ for all $0 < \epsilon < \epsilon_3$, then we claim that 
$\max_{z}\Utility{\TypeNoArg_I}(\FractionNoArg^{\text{game}}, z; \Penalty, \Threshold,\Quality,\CostFn) \ge \max_{z}\Utility{\TypeNoArg_I}(\FractionNoArg, z; \Penalty, \Threshold,\Quality,\CostFn)$ for all $\FractionNoArg \ge \bar{\FractionNoArg}_3$.  
By Lemma \ref{lemma:detectorfinitecostspostprocessing}, we also know that if we write $\FractionNoArg = \frac{1}{1+\epsilon}$, then:
\begin{align*}
 &\max_{z}\Utility{\TypeNoArg_I}(\FractionNoArg^{\text{game}}, z; \Penalty, \Threshold,\Quality,\CostFn) - \max_{z}\Utility{\TypeNoArg_I}(\FractionNoArg, z; \Penalty, \Threshold,\Quality,\CostFn) \\
 &= \Quality(\Combination{\TypeNoArg_I}(\FractionNoArg^{\text{game}}))  + \min\left(\Quality_1(\Combination{\TypeNoArg_I}_1(\FractionNoArg)) - \max_{z'_1 \le \Threshold} (Q_1(z'_1) - \CostFn_1(\Combination{\TypeNoArg_I}(\FractionNoArg), z'_1)), \Penalty\right) \\
 &- \Quality(\Combination{\TypeNoArg_I}(\FractionNoArg)) - \ProductionCost{\TypeNoArg_I} (\FractionNoArg - \FractionNoArg^{\text{game}}) \\
&\ge \Quality_{i'}(a_{i'}(\epsilon_2))  - \Quality_{i'}\left(b_{i'}(\epsilon_2) + \frac{(1-\alpha)(1+\epsilon_1)}{\epsilon_1} \cdot (a_{i'}(\epsilon_2) - b_{i'}(\epsilon_2)) \right) + \Quality_{1}(\Threshold) - \Quality_1\left(\Threshold + \frac{\epsilon_1-\epsilon}{1+\epsilon} \right) \\
&+  \min\left(\Quality_1(\Combination{\TypeNoArg_I}_1(\FractionNoArg)) - \max_{z'_1 \le \Threshold} (Q_1(z'_1) - \CostFn_1(\Combination{\TypeNoArg_I}(\FractionNoArg), z'_1)), \Penalty\right) -  \ProductionCost{\TypeNoArg_I} (\FractionNoArg - \FractionNoArg^{\text{game}}) \\
 &\ge \Quality_{i'}(a_{i'}(\epsilon_2))  - \Quality_{i'}\left(b_{i'}(\epsilon_2) + \frac{\epsilon(1+\epsilon_1)}{\epsilon_1(1+\epsilon)} \cdot (a_{i'}(\epsilon_2) - b_{i'}(\epsilon_2)) \right) -  \ProductionCost{\TypeNoArg_I} \left(\frac{1}{1+\epsilon} - \FractionNoArg^{\text{game}}\right) \\
 &+ (\kappa - \epsilon_4) \cdot \frac{\epsilon_1 - \epsilon}{1 + \epsilon} \\
 &\ge 0,
\end{align*}
where the last inequality holds by assumption for $0 < \epsilon < \epsilon_3$ and by continuity for $\epsilon = 0$ and $\epsilon = \epsilon_3$. 

\paragraph{Guarantees under conditions.}
Suppose that $\ProductionCost{\TypeNoArg_I} > -\Quality'_1(\Threshold + \epsilon_1) (1+ \epsilon_1) -\Quality'_{i'}(b_{i'}(\epsilon_2)) (b_{i'}(\epsilon_2) - a_{i'}(\epsilon_2)) \cdot \frac{1 + \epsilon_1}{\epsilon_1}$, $\Quality_1(\Threshold) - \Quality_1(\Threshold + \epsilon_1) + \Quality_{i'}\left(b_{i'}(\epsilon_2) + \frac{(1-\bar{\FractionNoArg}_3)(1+\epsilon_1)}{\epsilon_1} \cdot (a_{i'}(\epsilon_2) - b_{i'}(\epsilon_2)) \right) - \Quality_{i'}(b_{i'}(\epsilon_2)) > 0$, and \\
$\Quality_{i'}(a_{i'}(\epsilon_2))  - \Quality_{i'}\left(b_{i'}(\epsilon_2) + \frac{\epsilon (1+\epsilon_1)}{\epsilon_1(1+\epsilon)} \cdot (a_{i'}(\epsilon_2) - b_{i'}(\epsilon_2)) \right) -  \ProductionCost{\TypeNoArg_I} \left(\frac{1}{1+\epsilon} - \FractionNoArg^{\text{game}}\right) + (\kappa - \epsilon_4) \cdot \frac{\epsilon_1 - \epsilon}{1 + \epsilon} > 0$. By Theorem \ref{thm:detectorfinitecosts}, we know that $\Fraction{\TypeNoArg_I}(\Penalty, \Threshold; \Quality, \CostFn) \in \left\{1, \FractionNoArg^{\text{game}}\right\} \cup A^{\TypeNoArg}(\Penalty)$, and the above arguments coupled with the tiebreaking rule show that $\Fraction{\TypeNoArg_I}(\Penalty, \Threshold; \Quality, \CostFn) \in \left\{\FractionNoArg^{\text{game}}\right\} \cup ([\FractionNoArg^{\text{game}}, \bar{\FractionNoArg}_3] \cap A^{\TypeNoArg}(\Penalty))$.  
Moreover, by Theorem \ref{thm:nodetector}, we know that $\Fraction{\TypeNoArg_I}(\emptyset; \Quality, \CostFn) = \FractionNoArg^*$, and the above arguments show that $\FractionNoArg^* = 1$. Moreover, Theorem \ref{thm:nodetector} also tells us
that $\FinalContent{\TypeNoArg_I}(\emptyset; \Quality, \CostFn) = \Combination{\TypeNoArg_I}(\FractionNoArg^*)$. Putting this together, and combining with the above arguments, we see that 
\[\Quality(\FinalContent{\TypeNoArg_I}(\Penalty, \Threshold; \Quality, \CostFn)) >  \Quality(\Combination{\TypeNoArg_I}(\Fraction{\TypeNoArg_I}(\emptyset; \Quality, \CostFn))) = \Quality(\FinalContent{\TypeNoArg_I}(\emptyset; \Quality, \CostFn))\]
as desired.

\paragraph{Condition analysis.} It now suffices to construct $\epsilon_1, \epsilon_2, \epsilon_3$ and $\ProductionCost{\TypeNoArg_I}$ that satisfy the following for all $0 \le \epsilon < \epsilon_3$:
\[\ProductionCost{\TypeNoArg_I} > -\Quality'_1(\Threshold + \epsilon_1) (1+ \epsilon_1) -\Quality'_{i'}(b_{i'}(\epsilon_2)) (b_{i'}(\epsilon_2) - a_{i'}(\epsilon_2)) \cdot \frac{1 + \epsilon_1}{\epsilon_1},\]
\[\Quality_1(\Threshold) - \Quality_1(\Threshold + \epsilon_1) + \Quality_{i'}\left(b_{i'}(\epsilon_2) + \frac{(1-\bar{\FractionNoArg}_3)(1+\epsilon_1)}{\epsilon_1} \cdot (a_{i'}(\epsilon_2) - b_{i'}(\epsilon_2)) \right) - \Quality_{i'}(b_{i'}(\epsilon_2)) > 0\]
\[\Quality_{i'}(a_{i'}(\epsilon_2))  - \Quality_{i'}\left(b_{i'}(\epsilon_2) + \frac{\epsilon (1+\epsilon_1)}{\epsilon_1(1+\epsilon)} \cdot (a_{i'}(\epsilon_2) - b_{i'}(\epsilon_2)) \right)  \]
\[-  \ProductionCost{\TypeNoArg_I} \left(\frac{1}{1+\epsilon} - \FractionNoArg^{\text{game}}\right) + (\kappa - \epsilon_4) \cdot \frac{\epsilon_1 - \epsilon}{1 + \epsilon} > 0\]

This is equivalent to 
\[\ProductionCost{\TypeNoArg_I} \cdot  \frac{\epsilon_1 - \epsilon}{(1 + \epsilon_1)(1+\epsilon)} > -\Quality'_1(\Threshold + \epsilon_1) \cdot \frac{\epsilon_1 - \epsilon}{(1+\epsilon)}  + \Quality'_{i'}(b_{i'}(\epsilon_2)) (a_{i'}(\epsilon_2) - b_{i'}(\epsilon_2)) \cdot \frac{\epsilon_1 - \epsilon}{\epsilon_1(1+\epsilon)} ,\]
\[\Quality_1(\Threshold) - \Quality_1(\Threshold + \epsilon_1) + \Quality_{i'}\left(b_{i'}(\epsilon_2) + \frac{\epsilon_3(1+\epsilon_1)}{\epsilon_1 (1+\epsilon_3)} \cdot (a_{i'}(\epsilon_2) - b_{i'}(\epsilon_2)) \right) - \Quality_{i'}(b_{i'}(\epsilon_2)) > 0\]
\[\Quality_{i'}(a_{i'}(\epsilon_2))  - \Quality_{i'}\left(b_{i'}(\epsilon_2) + \frac{\epsilon (1+\epsilon_1)}{\epsilon_1(1+\epsilon)} \cdot (a_{i'}(\epsilon_2) - b_{i'}(\epsilon_2)) \right)  \]
\[-  \ProductionCost{\TypeNoArg_I} \cdot  \frac{\epsilon_1 - \epsilon}{(1 + \epsilon_1)(1+\epsilon)} + (\kappa - \epsilon_4) \cdot \frac{\epsilon_1 - \epsilon}{1 + \epsilon} > 0\]

We can construct $\ProductionCost{\TypeNoArg_I} \ge 0$ as long as: 
\[-\Quality'_1(\Threshold + \epsilon_1) \cdot \frac{\epsilon_1 - \epsilon}{(1+\epsilon)} + \Quality'_{i'}(b_{i'}(\epsilon_2)) (a_{i'}(\epsilon_2) - b_{i'}(\epsilon_2)) \cdot \frac{\epsilon_1 - \epsilon}{\epsilon_1(1+\epsilon)} \]
\[< \Quality_{i'}(a_{i'}(\epsilon_2))  - \Quality_{i'}\left(b_{i'}(\epsilon_2) + \frac{\epsilon(1+\epsilon_1)}{\epsilon_1 (1+\epsilon)} \cdot (a_{i'}(\epsilon_2) - b_{i'}(\epsilon_2)) \right) + (\kappa - \epsilon_4) \cdot \frac{\epsilon_1 - \epsilon}{1 + \epsilon}.\]
\[\Quality_{i'}\left(b_{i'}(\epsilon_2) + \frac{\epsilon_3(1+\epsilon_1)}{\epsilon_1 (1+\epsilon_3)} \cdot (a_{i'}(\epsilon_2) - b_{i'}(\epsilon_2)) \right) - \Quality_{i'}(b_{i'}(\epsilon_2)) > \Quality_1(\Threshold + \epsilon_1) - \Quality_1(\Threshold)\]
\[\Quality_{i'}(a_{i'}(\epsilon_2))  - \Quality_{i'}\left(b_{i'}(\epsilon_2) + \frac{\epsilon (1+\epsilon_1)}{\epsilon_1(1+\epsilon)} \cdot (a_{i'}(\epsilon_2) - b_{i'}(\epsilon_2)) \right)  + (\kappa - \epsilon_4) \cdot \frac{\epsilon_1 - \epsilon}{1 + \epsilon} > 0.\]

We can rewrite the first inequality as:
\[- \Quality'_1(\Threshold + \epsilon_1) \cdot \frac{\epsilon_1 - \epsilon}{(1+\epsilon)} \]
\[< \Quality_{i'}(a_{i'}(\epsilon_2))  - \Quality_{i'}\left(a_{i'}(\epsilon_2) + \frac{\epsilon_1 - \epsilon}{\epsilon_1 (1+\epsilon)} \cdot (b_{i'}(\epsilon_2) - a_{i'}(\epsilon_2)) \right) \]
\[- \Quality'_{i'}(b_{i'}(\epsilon_2)) (a_{i'}(\epsilon_2) - b_{i'}(\epsilon_2)) \cdot \frac{\epsilon_1 - \epsilon}{\epsilon_1(1+\epsilon)} + (\kappa - \epsilon_4) \cdot \frac{\epsilon_1 - \epsilon}{1 + \epsilon}.\]
Using concavity, the right-hand-side can be lower bounded by:
\begin{align*}
  &\Quality_{i'}(a_{i'}(\epsilon_2))  - \Quality_{i'}\left(a_{i'}(\epsilon_2) + \frac{\epsilon_1 - \epsilon}{\epsilon_1 (1+\epsilon)} \cdot (b_{i'}(\epsilon_2) - a_{i'}(\epsilon_2)) \right) \\
  &- \Quality'_{i'}(b_{i'}(\epsilon_2)) (a_{i'}(\epsilon_2) - b_{i'}(\epsilon_2)) \cdot \frac{\epsilon_1 - \epsilon}{\epsilon_1(1+\epsilon)} + (\kappa - \epsilon_4) \cdot \frac{\epsilon_1 - \epsilon}{1 + \epsilon} \\
  &\ge \Quality'_{i'}(a_{i'}(\epsilon_2)) \cdot \frac{\epsilon_1 - \epsilon}{\epsilon_1 (1+\epsilon)} \cdot (a_{i'}(\epsilon_2) - b_{i'}(\epsilon_2)) \\
  &- \Quality'_{i'}(b_{i'}(\epsilon_2)) (a_{i'}(\epsilon_2) - b_{i'}(\epsilon_2)) \cdot \frac{\epsilon_1 - \epsilon}{\epsilon_1(1+\epsilon)} + (\kappa - \epsilon_4) \cdot \frac{\epsilon_1 - \epsilon}{1 + \epsilon} \\
  &= (\Quality'_{i'}(a_{i'}(\epsilon_2)) - \Quality'_{i'}(b_{i'}(\epsilon_2))) \cdot (a_{i'}(\epsilon_2) - b_{i'}(\epsilon_2)) \cdot \frac{\epsilon_1 - \epsilon}{\epsilon_1(1+\epsilon)} + (\kappa - \epsilon_4) \cdot \frac{\epsilon_1 - \epsilon}{1 + \epsilon} \\
  &\ge -\sup_{z_i \in \text{conv}(a_{i'}(\epsilon_2), b_{i'}(\epsilon_2))} |Q''_{i'}(z_i)| \cdot (b_{i'}(\epsilon_2) - a_{i'}(\epsilon_2))^2 \cdot \frac{\epsilon_1 - \epsilon}{\epsilon_1(1+\epsilon)} + (\kappa - \epsilon_4) \cdot \frac{\epsilon_1 - \epsilon}{1 + \epsilon} .
\end{align*}
The first inequality is thus implied by:
\[(\kappa - \epsilon_4) \cdot \epsilon_1 + \epsilon_1 \Quality'_1(\Threshold + \epsilon_1) > \sup_{z_i \in \text{conv}(a_{i'}(\epsilon_2), b_{i'}(\epsilon_2))} |Q''_{i'}(z_i)| \cdot (b_{i'}(\epsilon_2) - a_{i'}(\epsilon_2))^2. \]
The second inequality is implied by:
\[\Quality'_{i'}\left(b_{i'}(\epsilon_2) + \frac{\epsilon_3(1+\epsilon_1)}{\epsilon_1 (1+\epsilon_3)} \cdot (a_{i'}(\epsilon_2) - b_{i'}(\epsilon_2)) \right) \left(\frac{\epsilon_3(1+\epsilon_1)}{\epsilon_1 (1+\epsilon_3)} \right) \left(a_{i'}(\epsilon_2) - b_{i'}(\epsilon_2) \right) > \epsilon_1 \Quality'_1(\Threshold). \]
The third inequality is equivalent to:
\[\Quality_{i'}(a_{i'}(\epsilon_2))  - \Quality_{i'}\left(a_{i'}(\epsilon_2) + \frac{\epsilon_1 - \epsilon}{\epsilon_1(1+\epsilon)} \cdot (b_{i'}(\epsilon_2) - a_{i'}(\epsilon_2)) \right)  + (\kappa - \epsilon_4) \cdot \frac{\epsilon_1 - \epsilon}{1 + \epsilon} > 0.\]

This is implied by: 
\[-\Quality'_{i'}(a_{i'}(\epsilon_2))  \cdot \frac{\epsilon_1 - \epsilon}{\epsilon_1(1+\epsilon)} \cdot (b_{i'}(\epsilon_2) - a_{i'}(\epsilon_2) + (\kappa - \epsilon_4) \cdot \frac{\epsilon_1 - \epsilon}{1 + \epsilon} > 0.\]
This is equivalent to:
\[\Quality'_{i'}(a_{i'}(\epsilon_2))  \cdot (a_{i'}(\epsilon_2) - b_{i'}(\epsilon_2)) + (\kappa - \epsilon_4) \cdot \epsilon_1 > 0.\]

\paragraph{Parameter settings.} Our goal is to construct parameters that satisfy the following three conditions: 
\begin{equation}
\label{eqref:condition1increase}
(\kappa - \epsilon_4) \cdot \epsilon_1 + \epsilon_1 \Quality'_1(\Threshold + \epsilon_1) > \sup_{z_i \in \text{conv}(a_{i'}(\epsilon_2), b_{i'}(\epsilon_2))} |Q''_{i'}(z_i)| \cdot (b_{i'}(\epsilon_2) - a_{i'}(\epsilon_2))^2
\end{equation}
\begin{equation}
\label{eqref:condition2increase}
\Quality'_{i'}\left(b_{i'}(\epsilon_2) + \frac{\epsilon_3(1+\epsilon_1)}{\epsilon_1 (1+\epsilon_3)} \cdot (a_{i'}(\epsilon_2) - b_{i'}(\epsilon_2)) \right) \left(\frac{\epsilon_3(1+\epsilon_1)}{\epsilon_1 (1+\epsilon_3)} \right) \left(a_{i'}(\epsilon_2) - b_{i'}(\epsilon_2) \right) > \epsilon_1 \Quality'_1(\Threshold)
\end{equation}
\begin{equation}
\label{eqref:condition3increase}
\Quality'_{i'}(a_{i'}(\epsilon_2))  \cdot (a_{i'}(\epsilon_2) - b_{i'}(\epsilon_2)) + (\kappa - \epsilon_4) \cdot \epsilon_1 > 0.
\end{equation}

We can take $\epsilon_4 < \min((\kappa + \Quality'_1(\Threshold))/4, \kappa/2)$. Then we satisfy \eqref{eqref:condition3increase}. 
If we take $\epsilon_2$ such that $\left(a_{i'}(\epsilon_2) - b_{i'}(\epsilon_2) \right) = o(\sqrt{(\kappa + Q'(\Threshold)) \epsilon_1})$ and make $\epsilon_1$ sufficiently small, this satisfies \eqref{eqref:condition1increase}. If we take $\epsilon_2$ such that $\left(a_{i'}(\epsilon_2) - b_{i'}(\epsilon_2) \right) = \Omega(\epsilon_1)$, take $\epsilon_3$ sufficiently close to $\epsilon_1$, and take $\epsilon_1$ sufficiently small, then we satisfy \eqref{eqref:condition2increase}.  
We can achieve both conditions on $\epsilon_2$ by making it scale with $\epsilon_1^{3/4}$ and making $\epsilon_1$ sufficiently small.  

\end{proof}

\begin{lemma}
\label{lemma:qualityfinitedecrease}
 Fix $\Penalty > 0$, $\Threshold < \infty$, $\Quality$, and finite costs $\CostFn$. Suppose that the second or third condition in Corollary \ref{cor:monotoneusage} holds. Suppose also that $\nabla_1 (\CostFn_1(\Threshold, \Threshold)) > \max(0, -\Quality'_1(\Threshold))$.  Suppose also  $D \ge 2$, that there exists a dimension $2 \le i' \le D$ such that $\Quality_{i'}$ is non-constant. Then, the presence of the detector 
decreases quality for some user type $\TypeNoArg_D$ such that 
\[ \Quality( \FinalContent{\TypeNoArg_D}(\Penalty, \Threshold; \Quality, \CostFn)) < \Quality( \FinalContent{\TypeNoArg_D}(\emptyset; \Quality, \CostFn)).\]  
\end{lemma}
\begin{proof}

We let $\HumanContent{\TypeNoArg_D}_i = \LLMContent{\TypeNoArg_D}_i$ for all $i \neq 1, i'$. By applying the assumptions in the lemma along with Lemma \ref{lemma:gthresholdmixed}
in the mixed-partial case and Lemma \ref{lemma:gthreshold} in the weakly-increasing $Q_1$ case, we know that $G(z_1) = \Threshold$ for all $z_1 \ge \Threshold$.

Let $\HumanContent{\TypeNoArg_D}_1 = \Threshold - 1$, and let $\LLMContent{\TypeNoArg_D}_1 = \Threshold + \epsilon_1$ for an $\epsilon_1 > 0$ that we will set later. We will also later construct $0 < \epsilon_3 < \epsilon_1$ and $\ProductionCost{\TypeNoArg_D}$. Let $\bar{\FractionNoArg}_3 = \frac{1}{1+ \epsilon_3}$. Let $\kappa = \nabla_1 (\CostFn_1(\Threshold, \Threshold))$. 

We will use the following fact. For any $0 \le \epsilon < \epsilon_1$, let $\alpha = \frac{1}{1+\epsilon}$. 
For any $\epsilon_4 > 0$, if  $\epsilon_1$ is sufficiently small, and letting $\kappa = \nabla_1 (\CostFn_1(\Threshold, \Threshold))$, it holds that: 
\begin{align*}
&\min\left(\left(\Quality_1(\Combination{\TypeNoArg_D}_1(\FractionNoArg)) - \max_{z'_1 \le \Threshold} \left(\Quality_1(z'_1) - \CostFn_1(\Combination{\TypeNoArg_D}_1(\FractionNoArg), z'_1) \right)  \right), \Penalty \right) \\ 
 &=_{(B)} \min\left(\left(\Quality_1(\Combination{\TypeNoArg_D}_1(\FractionNoArg)) - \Quality_1(\Threshold) + \CostFn_1(\Combination{\TypeNoArg_D}_1(\FractionNoArg), \Threshold)  \right), \Penalty \right) \\
 &=_{(C)} \Quality_1(\Combination{\TypeNoArg_D}_1(\FractionNoArg)) - \Quality_1(\Threshold) + \CostFn_1(\Combination{\TypeNoArg_D}_1(\FractionNoArg), \Threshold) \\
 &= \Quality_1(\Combination{\TypeNoArg_D}_1(\FractionNoArg)) - \Quality_1(\Threshold) + \CostFn_1(\Combination{\TypeNoArg_D}_1(\FractionNoArg), \Threshold) - \CostFn_1(\Threshold, \Threshold) \\
 &\ge \Quality_1\left(\Threshold + \frac{\epsilon_1 - \epsilon}{1 + \epsilon}\right) - \Quality_1(\Threshold) + \left(\inf_{x \in [\Threshold, \Combination{\TypeNoArg_D}_1(\FractionNoArg)]} (\nabla_1 \CostFn_1(x, \Threshold))  \right) \cdot \left(\Combination{\TypeNoArg_D}_1(\FractionNoArg) - \Threshold \right) \\ 
 &\ge_{(D)} \Quality_1\left(\Threshold + \frac{\epsilon_1 - \epsilon}{1 + \epsilon}\right) - \Quality_1(\Threshold) + (\kappa - \epsilon_4) \cdot  \left(\Combination{\TypeNoArg_D}_1(\FractionNoArg) - \Threshold) \right) \\ 
&= \Quality_1\left(\Threshold + \frac{\epsilon_1 - \epsilon}{1 + \epsilon}\right) - \Quality_1(\Threshold) + (\kappa - \epsilon_4) \cdot \frac{\epsilon_1 - \epsilon}{1 + \epsilon}.  
\end{align*}
where (B) uses the fact that $G(z_1) = \Threshold$ for all $z_1 \ge \Threshold$, and (C) uses that $\epsilon_1$ is sufficiently small along with the continuity of $\Quality_1$, assumption (A3), and the continuity of $\CostFn_1$ in its first argument, and (D) also uses that $\epsilon_1$ is sufficiently small along with the assumption that  $\CostFn$
is twice continuously differentiable. 

Since $Q_{i'}$ is non-constant and concave, there exists an interval \(I\)
on which \(Q'_{i'}\) has constant nonzero sign. Thus, for all sufficiently
small \(\epsilon_2>0\), we can choose \(a_{i'}(\epsilon_2),b_{i'}(\epsilon_2)\in I\)
such that
\[
|a_{i'}(\epsilon_2)-b_{i'}(\epsilon_2)|=\epsilon_2,
\qquad
Q_{i'}(a_{i'}(\epsilon_2))>Q_{i'}(b_{i'}(\epsilon_2)),
\]
and
\[
Q'_{i'}(a_{i'}(\epsilon_2))(a_{i'}(\epsilon_2)-b_{i'}(\epsilon_2))>0.
\]
Moreover, there exist constants \(0<k_1<k_2<\infty\) and \(k_3<\infty\) such that,
for all sufficiently small \(\epsilon_2\),
\[
k_1 \le |Q'_{i'}(z)| \le k_2,\qquad |Q''_{i'}(z)|\le k_3
\]
for all \(z\in \operatorname{conv}(a_{i'}(\epsilon_2),b_{i'}(\epsilon_2))\).
Moreover, we choose the parametrization so that the map
\[
\epsilon_2 \mapsto 
Q'_{i'}(a_{i'}(\epsilon_2))(a_{i'}(\epsilon_2)-b_{i'}(\epsilon_2))
\]
extends continuously to \(\epsilon_2=0\), has value \(0\) at \(\epsilon_2=0\), and is strictly positive for all sufficiently small \(\epsilon_2>0\).

Let $\LLMContent{\TypeNoArg_D}_{i'} = a_{i'}(\epsilon_2)$ and we let $\HumanContent{\TypeNoArg_D}_{i'} = a_{i'}(\epsilon_2) + (b_{i'}(\epsilon_2) - a_{i'}(\epsilon_2)) \cdot \frac{1 + \epsilon_1}{\epsilon_1}$ for $\epsilon_2 > 0$.

\paragraph{Utility with no detector.}
For $\ProductionCost{\TypeNoArg_D} > -\Quality'_1(\Threshold + \epsilon_1) (1+ \epsilon_1) -\Quality'_{i'}(a_{i'}(\epsilon_2)) (a_{i'}(\epsilon_2) - b_{i'}(\epsilon_2)) \cdot \frac{1 + \epsilon_1}{\epsilon_1}$, we  claim that $\FractionNoArg^* = 1$. By concavity, it suffices to show that the derivative at $\FractionNoArg = 1$ is positive. 
By the definition of $\FractionNoArg^*$
\[\Quality(\Combination{\TypeNoArg_D}(\FractionNoArg)) - \ProductionCost{\TypeNoArg_D}(1- \FractionNoArg) = \Quality_1(\Combination{\TypeNoArg_D}_1(\FractionNoArg)) + \Quality_{i'}(\Combination{\TypeNoArg_D}_{i'}(\FractionNoArg)) - \ProductionCost{\TypeNoArg_D}(1- \FractionNoArg)  +  \sum_{i \neq i', 1} \Quality_i(\LLMContent{\TypeNoArg}). \]
Taking a derivative, we obtain:
\begin{align*}
&\Quality'_1(\LLMContent{\TypeNoArg_D}_1) (\LLMContent{\TypeNoArg_D}_1 - \HumanContent{\TypeNoArg_D}_1) + \Quality'_{i'}(\LLMContent{\TypeNoArg_D}_{i'}) (\LLMContent{\TypeNoArg_D}_{i'} - \HumanContent{\TypeNoArg_D}_{i'}) + \ProductionCost{\TypeNoArg_D}  \\
&= \Quality'_1(\Threshold + \epsilon_1) (1+ \epsilon_1) + \Quality'_{i'}(a_{i'}(\epsilon_2)) (a_{i'}(\epsilon_2) - b_{i'}(\epsilon_2)) \cdot \frac{1 + \epsilon_1}{\epsilon_1} + \ProductionCost{\TypeNoArg_D}  \\
&> 0
\end{align*}
as desired.

\paragraph{Quality analysis.}
If $\Quality_1(\Threshold) - \Quality_1(\Threshold + \epsilon_1) + \Quality_{i'}\left(a_{i'}(\epsilon_2) + \frac{(1-\bar{\FractionNoArg}_3)(1+\epsilon_1)}{\epsilon_1} \cdot (b_{i'}(\epsilon_2) - a_{i'}(\epsilon_2)) \right) - \Quality_{i'}(a_{i'}(\epsilon_2)) < 0$, then we claim that $\Quality(\FinalContent{\TypeNoArg_D}(\FractionNoArg)) < \Quality(\Combination{\TypeNoArg_D}(1))$ for all $\FractionNoArg \in \left\{\FractionNoArg^{\text{game}}\right\} \cup ([\FractionNoArg^{\text{game}}, \bar{\FractionNoArg}_3] \cap A^{\TypeNoArg}(\Penalty))$. By Lemma \ref{lemma:detectorfinitecostspostprocessing}, we know that:
\begin{align*}
&\Quality(\FinalContent{\TypeNoArg_D}(\FractionNoArg)) - \Quality(\Combination{\TypeNoArg_D}(1)) \\ &= \Quality_1(\text{argmax}_{z'_1 \le \Threshold}\left(\Quality_1(z'_1) - \CostFn_1(\Combination{\TypeNoArg_D}(\FractionNoArg), z'_1 \right))) - \Quality_1(\Threshold + \epsilon_1) \\
&+ \Quality_{i'}\left(a_{i'}(\epsilon_2) + \frac{(1-\FractionNoArg)(1+\epsilon_1)}{\epsilon_1} \cdot (b_{i'}(\epsilon_2) - a_{i'}(\epsilon_2)) \right) - \Quality_{i'}(a_{i'}(\epsilon_2)) \\
&\le_{(B)} \Quality_1(\Threshold) - \Quality_1(\Threshold + \epsilon_1) + \Quality_{i'}\left(a_{i'}(\epsilon_2) + \frac{(1-\bar{\FractionNoArg}_3)(1+\epsilon_1)}{\epsilon_1} \cdot (b_{i'}(\epsilon_2) - a_{i'}(\epsilon_2)) \right) - \Quality_{i'}(a_{i'}(\epsilon_2)) \\
&< 0
\end{align*}
where (B) uses the fact that $G(z_1) = \Threshold$ for all $z_1 \ge \Threshold$. 

\paragraph{Utility under detection.}
If $\Quality_{i'}(b_{i'}(\epsilon_2))  - \Quality_{i'}\left(a_{i'}(\epsilon_2) + \frac{\epsilon (1+\epsilon_1)}{\epsilon_1(1+\epsilon)}  \cdot (b_{i'}(\epsilon_2) - a_{i'}(\epsilon_2)) \right) + (\kappa - \epsilon_4) \cdot \left(\frac{\epsilon_1 - \epsilon}{1 + \epsilon} \right) -  \ProductionCost{\TypeNoArg_D} \left(\frac{1}{1+\epsilon} - \FractionNoArg^{\text{game}}\right)  > 0$ for all $0 < \epsilon < \epsilon_3$, then we claim that 
$\max_{z}\Utility{\TypeNoArg_D}(\FractionNoArg^{\text{game}}, z; \Penalty, \Threshold,\Quality,\CostFn) \ge \max_{z}\Utility{\TypeNoArg_D}(\FractionNoArg, z; \Penalty, \Threshold,\Quality,\CostFn)$ for all $\FractionNoArg \ge \bar{\FractionNoArg}_3$.  
If we let $\alpha = \frac{1}{1+\epsilon}$, then 
by Lemma \ref{lemma:detectorfinitecostspostprocessing}, we know that:
\begin{align*}
 &\max_{z}\Utility{\TypeNoArg_D}(\FractionNoArg^{\text{game}}, z; \Penalty, \Threshold,\Quality,\CostFn) - \max_{z}\Utility{\TypeNoArg_D}(\FractionNoArg, z; \Penalty, \Threshold,\Quality,\CostFn) \\
 &= \Quality(\Combination{\TypeNoArg_D}(\FractionNoArg^{\text{game}}))  + \min\left(\Quality_1(\Combination{\TypeNoArg_D}_1(\FractionNoArg)) - \max_{z'_1 \le \Threshold} (Q_1(z'_1) - \CostFn_1(\Combination{\TypeNoArg_D}(\FractionNoArg), z'_1)), \Penalty\right) \\
 &- \Quality(\Combination{\TypeNoArg_D}(\FractionNoArg)) - \ProductionCost{\TypeNoArg_D} (\FractionNoArg - \FractionNoArg^{\text{game}}) \\
&\ge \Quality_{i'}(b_{i'}(\epsilon_2))  - \Quality_{i'}\left(a_{i'}(\epsilon_2) + \frac{\epsilon (1+\epsilon_1)}{\epsilon_1(1+\epsilon)}  \cdot (b_{i'}(\epsilon_2) - a_{i'}(\epsilon_2)) \right) \\
&+ (\kappa - \epsilon_4) \cdot \left(\frac{\epsilon_1 - \epsilon}{1 + \epsilon} \right) -  \ProductionCost{\TypeNoArg_D} \left(\frac{1}{1+\epsilon} - \FractionNoArg^{\text{game}}\right) \\
&\ge 0,
\end{align*}
where the last inequality holds by assumption for $0 < \epsilon < \epsilon_3$ and by continuity for $\epsilon = 0$ and $\epsilon = \epsilon_3$.
\paragraph{Guarantees under conditions.}
Suppose that $\ProductionCost{\TypeNoArg_D} > -\Quality'_1(\Threshold + \epsilon_1) (1+ \epsilon_1) -\Quality'_{i'}(a_{i'}(\epsilon_2)) (a_{i'}(\epsilon_2) - b_{i'}(\epsilon_2)) \cdot \frac{1 + \epsilon_1}{\epsilon_1}$, $\Quality_1(\Threshold) - \Quality_1(\Threshold + \epsilon_1) + \Quality_{i'}\left(a_{i'}(\epsilon_2) + \frac{(1-\bar{\FractionNoArg}_3)(1+\epsilon_1)}{\epsilon_1} \cdot (b_{i'}(\epsilon_2) - a_{i'}(\epsilon_2)) \right) - \Quality_{i'}(a_{i'}(\epsilon_2)) < 0$, and \\
$\Quality_{i'}(b_{i'}(\epsilon_2))  - \Quality_{i'}\left(a_{i'}(\epsilon_2) + \frac{\epsilon (1+\epsilon_1)}{\epsilon_1(1+\epsilon)}  \cdot (b_{i'}(\epsilon_2) - a_{i'}(\epsilon_2)) \right) + (\kappa - \epsilon_4) \cdot \left(\frac{\epsilon_1 - \epsilon}{1 + \epsilon} \right) -  \ProductionCost{\TypeNoArg_D} \left(\frac{1}{1+\epsilon} - \FractionNoArg^{\text{game}}\right)  > 0$ for all $0 \le \epsilon < \epsilon_3$. By Theorem \ref{thm:detectorfinitecosts}, we know that $\Fraction{\TypeNoArg_D}(\Penalty, \Threshold; \Quality, \CostFn) \in \left\{1, \FractionNoArg^{\text{game}}\right\} \cup A^{\TypeNoArg}(\Penalty)$, and the above arguments coupled with the tiebreaking rule show that $\Fraction{\TypeNoArg_D}(\Penalty, \Threshold; \Quality, \CostFn) \in \left\{\FractionNoArg^{\text{game}}\right\} \cup ([\FractionNoArg^{\text{game}}, \bar{\FractionNoArg}_3] \cap A^{\TypeNoArg}(\Penalty))$. 
Moreover, by Theorem \ref{thm:nodetector}, we know that $\Fraction{\TypeNoArg_D}(\emptyset; \Quality, \CostFn) = \FractionNoArg^*$, and the above arguments show that $\FractionNoArg^* = 1$. Moreover, Theorem \ref{thm:nodetector} also tells us
that $\FinalContent{\TypeNoArg_D}(\emptyset; \Quality, \CostFn) = \Combination{\TypeNoArg_D}(\FractionNoArg^*)$. Putting this together, and combining with the above arguments, we see that 
\[\Quality(\FinalContent{\TypeNoArg_D}(\Penalty, \Threshold; \Quality, \CostFn)) <  \Quality(\Combination{\TypeNoArg_D}(\Fraction{\TypeNoArg_D}(\emptyset; \Quality, \CostFn))) = \Quality(\FinalContent{\TypeNoArg_D}(\emptyset; \Quality, \CostFn))\]
as desired.

\paragraph{Condition analysis.}
It now suffices to construct $\epsilon_1, \epsilon_2, \epsilon_3$ and $\ProductionCost{\TypeNoArg_D}$ that satisfy the following for all $0 \le \epsilon < \epsilon_3$:
\[\ProductionCost{\TypeNoArg_D} > -\Quality'_1(\Threshold + \epsilon_1) (1+ \epsilon_1) -\Quality'_{i'}(a_{i'}(\epsilon_2)) (a_{i'}(\epsilon_2) - b_{i'}(\epsilon_2)) \cdot \frac{1 + \epsilon_1}{\epsilon_1},\]
\[\Quality_1(\Threshold) - \Quality_1(\Threshold + \epsilon_1) + \Quality_{i'}\left(a_{i'}(\epsilon_2) + \frac{(1-\bar{\FractionNoArg}_3)(1+\epsilon_1)}{\epsilon_1} \cdot (b_{i'}(\epsilon_2) - a_{i'}(\epsilon_2)) \right) - \Quality_{i'}(a_{i'}(\epsilon_2)) < 0\]
\[\Quality_{i'}(b_{i'}(\epsilon_2))  - \Quality_{i'}\left(a_{i'}(\epsilon_2) + \frac{\epsilon (1+\epsilon_1)}{\epsilon_1(1+\epsilon)}  \cdot (b_{i'}(\epsilon_2) - a_{i'}(\epsilon_2)) \right) \]
\[+ (\kappa - \epsilon_4) \cdot \left(\frac{\epsilon_1 - \epsilon}{1 + \epsilon} \right) -  \ProductionCost{\TypeNoArg_D} \left(\frac{1}{1+\epsilon} - \FractionNoArg^{\text{game}}\right)  > 0.\]

This is equivalent to 
\[\ProductionCost{\TypeNoArg_D} \cdot  \frac{\epsilon_1 - \epsilon}{(1 + \epsilon_1)(1+\epsilon)} > -\Quality'_1(\Threshold + \epsilon_1) \cdot \frac{\epsilon_1 - \epsilon}{(1+\epsilon)}  + \Quality'_{i'}(a_{i'}(\epsilon_2)) (b_{i'}(\epsilon_2) - a_{i'}(\epsilon_2)) \cdot \frac{\epsilon_1 - \epsilon}{\epsilon_1(1+\epsilon)} ,\]
\[\Quality_1(\Threshold) - \Quality_1(\Threshold + \epsilon_1) + \Quality_{i'}\left(a_{i'}(\epsilon_2) + \frac{\epsilon_3(1+\epsilon_1)}{\epsilon_1 (1+\epsilon_3)} \cdot (b_{i'}(\epsilon_2) - a_{i'}(\epsilon_2)) \right) - \Quality_{i'}(a_{i'}(\epsilon_2)) < 0\]
\[\Quality_{i'}(b_{i'}(\epsilon_2))  - \Quality_{i'}\left(a_{i'}(\epsilon_2) + \frac{\epsilon(1+\epsilon_1)}{\epsilon_1 (1+\epsilon)} \cdot (b_{i'}(\epsilon_2) - a_{i'}(\epsilon_2)) \right) \]
\[+ (\kappa - \epsilon_4) \cdot \left(\frac{\epsilon_1 - \epsilon}{1 + \epsilon} \right) -  \ProductionCost{\TypeNoArg_D} \cdot \frac{\epsilon_1 - \epsilon}{(1 + \epsilon_1)(1+\epsilon)}  > 0.\]

We can construct $\ProductionCost{\TypeNoArg_D} \ge 0$ as long as: 
\[-\Quality'_1(\Threshold + \epsilon_1) \cdot \frac{\epsilon_1 - \epsilon}{(1+\epsilon)} + \Quality'_{i'}(a_{i'}(\epsilon_2)) (b_{i'}(\epsilon_2) - a_{i'}(\epsilon_2)) \cdot \frac{\epsilon_1 - \epsilon}{\epsilon_1(1+\epsilon)} \]
\[< \Quality_{i'}(b_{i'}(\epsilon_2))  - \Quality_{i'}\left(a_{i'}(\epsilon_2) + \frac{\epsilon(1+\epsilon_1)}{\epsilon_1 (1+\epsilon)} \cdot (b_{i'}(\epsilon_2) - a_{i'}(\epsilon_2)) \right) + (\kappa - \epsilon_4) \cdot \left(\frac{\epsilon_1 - \epsilon}{1 + \epsilon} \right) .\]
\[\Quality_{i'}\left(a_{i'}(\epsilon_2) + \frac{\epsilon_3(1+\epsilon_1)}{\epsilon_1 (1+\epsilon_3)} \cdot (b_{i'}(\epsilon_2) - a_{i'}(\epsilon_2)) \right) - \Quality_{i'}(a_{i'}(\epsilon_2)) < \Quality_1(\Threshold + \epsilon_1) - \Quality_1(\Threshold)\]
\[\Quality_{i'}(b_{i'}(\epsilon_2))  - \Quality_{i'}\left(a_{i'}(\epsilon_2) + \frac{\epsilon(1+\epsilon_1)}{\epsilon_1 (1+\epsilon)} \cdot (b_{i'}(\epsilon_2) - a_{i'}(\epsilon_2)) \right) + (\kappa - \epsilon_4) \cdot \left(\frac{\epsilon_1 - \epsilon}{1 + \epsilon} \right) > 0.\]

We can rewrite the first inequality as:
\[- \Quality'_1(\Threshold + \epsilon_1) \cdot \frac{\epsilon_1 - \epsilon}{(1+\epsilon)}  \]
\[< \Quality_{i'}(b_{i'}(\epsilon_2))  - \Quality_{i'}\left(b_{i'}(\epsilon_2) + \frac{\epsilon_1 - \epsilon}{\epsilon_1 (1+\epsilon)} \cdot (a_{i'}(\epsilon_2) - b_{i'}(\epsilon_2)) \right)\]
\[- \Quality'_{i'}(a_{i'}(\epsilon_2)) (b_{i'}(\epsilon_2) - a_{i'}(\epsilon_2)) \cdot \frac{\epsilon_1 - \epsilon}{\epsilon_1(1+\epsilon)} + (\kappa - \epsilon_4) \cdot \left(\frac{\epsilon_1 - \epsilon}{1 + \epsilon} \right).\]

Using concavity, the right-hand side can be lower bounded by: 
\begin{align*}
  &\Quality_{i'}(b_{i'}(\epsilon_2))  - \Quality_{i'}\left(b_{i'}(\epsilon_2) + \frac{\epsilon_1 - \epsilon}{\epsilon_1 (1+\epsilon)} \cdot (a_{i'}(\epsilon_2) - b_{i'}(\epsilon_2)) \right) \\
  &- \Quality'_{i'}(a_{i'}(\epsilon_2)) (b_{i'}(\epsilon_2) - a_{i'}(\epsilon_2)) \cdot \frac{\epsilon_1 - \epsilon}{\epsilon_1(1+\epsilon)} + (\kappa - \epsilon_4) \cdot \left(\frac{\epsilon_1 - \epsilon}{1 + \epsilon} \right) \\
  &\ge \Quality'_{i'}(b_{i'}(\epsilon_2)) \cdot \frac{\epsilon_1 - \epsilon}{\epsilon_1 (1+\epsilon)} \cdot (b_{i'}(\epsilon_2) - a_{i'}(\epsilon_2))\\
  &- \Quality'_{i'}(a_{i'}(\epsilon_2)) (b_{i'}(\epsilon_2) - a_{i'}(\epsilon_2)) \cdot \frac{\epsilon_1 - \epsilon}{\epsilon_1(1+\epsilon)} + (\kappa - \epsilon_4) \cdot \left(\frac{\epsilon_1 - \epsilon}{1 + \epsilon} \right) \\
  &= (\Quality'_{i'}(b_{i'}(\epsilon_2)) - \Quality'_{i'}(a_{i'}(\epsilon_2))) \cdot (b_{i'}(\epsilon_2) - a_{i'}(\epsilon_2)) \cdot \frac{\epsilon_1 - \epsilon}{\epsilon_1(1+\epsilon)} + (\kappa - \epsilon_4) \cdot \left(\frac{\epsilon_1 - \epsilon}{1 + \epsilon} \right)\\
  &\ge -\sup_{z_i \in \text{conv}(a_{i'}(\epsilon_2), b_{i'}(\epsilon_2))} |Q''_{i'}(z_i)| \cdot (b_{i'}(\epsilon_2) - a_{i'}(\epsilon_2))^2 \cdot \frac{\epsilon_1 - \epsilon}{\epsilon_1(1+\epsilon)} + (\kappa - \epsilon_4) \cdot \left(\frac{\epsilon_1 - \epsilon}{1 + \epsilon} \right).
\end{align*}

The first inequality is thus implied by:
\[\sup_{z_i \in \text{conv}(a_{i'}(\epsilon_2), b_{i'}(\epsilon_2))} |Q''_{i'}(z_i)| \cdot (b_{i'}(\epsilon_2) - a_{i'}(\epsilon_2))^2  < \epsilon_1 (\kappa - \epsilon_4) + \epsilon_1 \Quality'_1(\Threshold + \epsilon_1) \]
We can rewrite the second inequality as: 
\[\Quality_{i'}(a_{i'}(\epsilon_2)) - \Quality_{i'}\left(a_{i'}(\epsilon_2) + \frac{\epsilon_3(1+\epsilon_1)}{\epsilon_1 (1+\epsilon_3)} \cdot (b_{i'}(\epsilon_2) - a_{i'}(\epsilon_2)) \right) >\Quality_1(\Threshold) -  \Quality_1(\Threshold + \epsilon_1)\]
The second inequality is implied by: 
\[\Quality'_{i'}(a_{i'}(\epsilon_2)) \left(\frac{\epsilon_3(1+\epsilon_1)}{\epsilon_1 (1+\epsilon_3)} \right) \cdot (a_{i'}(\epsilon_2) - b_{i'}(\epsilon_2)) > -\epsilon_1 Q'_1(\Threshold + \epsilon_1).\]
We can rewrite the third inequality as:
\[(\kappa - \epsilon_4) \cdot \left(\frac{\epsilon_1 - \epsilon}{1 + \epsilon} \right) > \Quality_{i'}\left(b_{i'}(\epsilon_2) + \frac{\epsilon_1 - \epsilon}{\epsilon_1 (1+\epsilon)} \cdot (a_{i'}(\epsilon_2) - b_{i'}(\epsilon_2)) \right)- \Quality_{i'}(b_{i'}(\epsilon_2)).\]
The third inequality is implied by: \[(\kappa - \epsilon_4) \cdot \left(\frac{\epsilon_1 - \epsilon}{1 + \epsilon} \right) >  \Quality'_{i'}(b_{i'}(\epsilon_2)) \left(\frac{\epsilon_1 - \epsilon}{\epsilon_1 (1+\epsilon)} \right) (a_{i'}(\epsilon_2) - b_{i'}(\epsilon_2)),\]
which can be written as:
\[\epsilon_1(\kappa - \epsilon_4) >  \Quality'_{i'}(b_{i'}(\epsilon_2)) (a_{i'}(\epsilon_2) - b_{i'}(\epsilon_2)).\]
which is implied by:
\[-\epsilon_1 \cdot Q'_1(\Threshold + \epsilon_1) + (\epsilon_1 \cdot Q'_1(\Threshold + \epsilon_1) - \epsilon_1 \cdot Q'_1(\Threshold)) + \epsilon_1(\kappa + Q'_1(\Threshold) - \epsilon_4)  \]
\[>  \Quality'_{i'}(a_{i'}(\epsilon_2)) (a_{i'}(\epsilon_2) - b_{i'}(\epsilon_2)) + \sup_{z_i \in \text{conv}(a_{i'}(\epsilon_2), b_{i'}(\epsilon_2))} |Q''_{i'}(z_i)| \cdot (a_{i'}(\epsilon_2) - b_{i'}(\epsilon_2))^2.\]
which is implied by:
\[-\epsilon_1 \cdot Q'_1(\Threshold + \epsilon_1) - \epsilon^2_1 \cdot \sup_{z_1 \in [\Threshold, \Threshold + \epsilon_1]} |Q''(z_1)| + \epsilon_1(\kappa + Q'_1(\Threshold) - \epsilon_4)  \]
\[>  \Quality'_{i'}(a_{i'}(\epsilon_2)) (a_{i'}(\epsilon_2) - b_{i'}(\epsilon_2)) + \sup_{z_i \in \text{conv}(a_{i'}(\epsilon_2), b_{i'}(\epsilon_2))} |Q''_{i'}(z_i)| \cdot (a_{i'}(\epsilon_2) - b_{i'}(\epsilon_2))^2.\]

\paragraph{Parameter settings.} Our goal is to construct parameters that satisfy the following three conditions:
\begin{equation}
\label{eq:condition1qualitydecrease}
\sup_{z_i \in \text{conv}(a_{i'}(\epsilon_2), b_{i'}(\epsilon_2))} |Q''_{i'}(z_i)| \cdot (b_{i'}(\epsilon_2) - a_{i'}(\epsilon_2))^2  < \epsilon_1 (\kappa - \epsilon_4) + \epsilon_1 \Quality'_1(\Threshold + \epsilon_1) 
\end{equation}
\begin{equation}
\label{eq:condition2qualitydecrease}
\Quality'_{i'}(a_{i'}(\epsilon_2)) \left(\frac{\epsilon_3(1+\epsilon_1)}{\epsilon_1 (1+\epsilon_3)} \right) \cdot (a_{i'}(\epsilon_2) - b_{i'}(\epsilon_2)) > -\epsilon_1 Q'_1(\Threshold + \epsilon_1)
\end{equation}
\begin{equation}
\label{eq:condition3qualitydecrease}
\begin{split}
& -\epsilon_1 \cdot Q'_1(\Threshold + \epsilon_1)  - \epsilon^2_1 \cdot \sup_{z_1 \in [\Threshold, \Threshold + \epsilon_1]} |Q''(z_1)|  + \epsilon_1(\kappa + Q'_1(\Threshold) - \epsilon_4) \\
&>  \Quality'_{i'}(a_{i'}(\epsilon_2)) (a_{i'}(\epsilon_2) - b_{i'}(\epsilon_2)) + \sup_{z_i \in \text{conv}(a_{i'}(\epsilon_2), b_{i'}(\epsilon_2))} |Q''_{i'}(z_i)| \cdot (a_{i'}(\epsilon_2) - b_{i'}(\epsilon_2))^2
\end{split}
\end{equation}

  Let
  \[
  p(\epsilon_2)
  := Q'_{i'}(a_{i'}(\epsilon_2))
     (a_{i'}(\epsilon_2)-b_{i'}(\epsilon_2)).
  \]
  By construction, \(p(\epsilon_2)>0\) for all sufficiently small
  \(\epsilon_2>0\), \(p(\epsilon_2)\to 0\) as \(\epsilon_2\to 0\), and,
  since \(|Q'_{i'}|\) is bounded above and below on the chosen interval,
  \(p(\epsilon_2)=\Theta(|a_{i'}(\epsilon_2)-b_{i'}(\epsilon_2)|)\).

  Choose \(\epsilon_4>0\) small enough that
  \[
  \max(0,-Q'_1(\nu))
  < \kappa-\epsilon_4 .
  \]
  Then choose a constant \(\rho\) satisfying
  \[
  \max(0,-Q'_1(\nu))
  < \rho
  < \kappa-\epsilon_4 .
  \]
  For each sufficiently small \(\epsilon_1\), choose \(\epsilon_2\) so that
  \[
  p(\epsilon_2)=\rho\epsilon_1 .
  \]
  This is possible by the continuity of \(p\) and the fact that
  \(p(\epsilon_2)=\Theta(\epsilon_2)\) locally. In particular,
  \[
  a_{i'}(\epsilon_2)-b_{i'}(\epsilon_2)=O(\epsilon_1),
  \]
  and hence the curvature term
  \[
  \sup_{z_i\in\operatorname{conv}(a_{i'}(\epsilon_2),b_{i'}(\epsilon_2))}
  |Q''_{i'}(z_i)|\,
  (a_{i'}(\epsilon_2)-b_{i'}(\epsilon_2))^2
  \]
  is \(O(\epsilon_1^2)\).

  We now verify \eqref{eq:condition1qualitydecrease}--\eqref{eq:condition3qualitydecrease}. Since
  \[
  \epsilon_1(\kappa-\epsilon_4)+\epsilon_1 Q'_1(\nu+\epsilon_1)
  =
  \epsilon_1(\kappa+Q'_1(\nu)-\epsilon_4)+O(\epsilon_1^2),
  \]
  and \(\kappa+Q'_1(\nu)-\epsilon_4>0\), condition \eqref{eq:condition1qualitydecrease} holds for
  sufficiently small \(\epsilon_1\). Condition \eqref{eq:condition2qualitydecrease} holds because
  \(\epsilon_3\) is chosen sufficiently close to \(\epsilon_1\), so the
  multiplicative factor converges to \(1\), and
  \[
  p(\epsilon_2)=\rho\epsilon_1
   >
  -\epsilon_1 Q'_1(\nu+\epsilon_1)
  \]
  for sufficiently small \(\epsilon_1\), where we use that $\rho > -\Quality'_1(\Threshold)$ and $Q'_1(\nu+\epsilon_1) \rightarrow \Quality'_1(\Threshold)$.  
  Finally, condition \eqref{eq:condition3qualitydecrease} holds
  because its left-hand side is lower bounded as 
  \[
  \epsilon_1(\kappa+Q'_1(\nu)-\epsilon_4)-\epsilon_1Q'_1(\nu+\epsilon_1)
  -O(\epsilon_1^2)
  =
  \epsilon_1(\kappa-\epsilon_4)-O(\epsilon_1^2),
  \]
  while its right-hand side is upper bounded as
  \[
  p(\epsilon_2)+O(\epsilon_1^2)=\rho\epsilon_1+O(\epsilon_1^2),
  \]
  and \(\rho<\kappa-\epsilon_4\) was chosen with slack.

\end{proof}

\begin{lemma}
\label{lemma:qualityfiniteconstant}
 Fix $\Penalty > 0$, $\Threshold < \infty$, $\Quality$, and finite costs $\CostFn$. Suppose that the second or third condition in Corollary \ref{cor:monotoneusage} holds.  Then, the presence of the detector keeps quality the same for some user type $\TypeNoArg_N$ 
\[ \Quality( \FinalContent{\TypeNoArg_N}(\Penalty, \Threshold; \Quality, \CostFn)) = \Quality( \FinalContent{\TypeNoArg_N}(\emptyset; \Quality, \CostFn)).\]   
\end{lemma}
\begin{proof}
Let $\TypeNoArg_N$ be such that $\HumanContent{\TypeNoArg_N}_1 <  \LLMContent{\TypeNoArg_N}_1 < \Threshold$, and let $\HumanContent{\TypeNoArg_N}_{i} = \LLMContent{\TypeNoArg_N}_{i}$ for all $i \ge 2$. By Theorem \ref{thm:detectorfinitecosts}, we know that $\Fraction{\TypeNoArg_N}(\Penalty, \Threshold; \Quality, \CostFn) = \FractionNoArg^*$ and $\FinalContent{\TypeNoArg_N}(\Penalty, \Threshold; \Quality, \CostFn) = \Combination{\TypeNoArg_N}(\FractionNoArg^*)$. By Theorem \ref{thm:nodetector}, we know that $\Fraction{\TypeNoArg_N}(\emptyset; \Quality, \CostFn) = \FractionNoArg^*$ and $\FinalContent{\TypeNoArg_N}(\emptyset; \Quality, \CostFn) = \Combination{\TypeNoArg_N}(\FractionNoArg^*)$. This implies that $\Quality(\FinalContent{\TypeNoArg_N}(\Penalty, \Threshold; \Quality, \CostFn)) = \Quality(\FinalContent{\TypeNoArg_N}(\emptyset; \Quality, \CostFn))$ as desired. 
    
\end{proof}

\subsection{Proofs from lemmas}
The desired results follow from these lemmas.
\begin{proof}[Proof of Theorem \ref{thm:qualitydecrease}]
  This follows from Lemma \ref{lemma:qualityfinitedecrease} and Lemma \ref{lemma:qualityinfinitedecrease}.  
\end{proof}
\begin{proof}[Proof of Theorem \ref{thm:qualitygeneral}]
This follows from Lemma \ref{lemma:qualityinfiniteincrease}, Lemma \ref{lemma:qualityinfiniteconstant}, Lemma \ref{lemma:qualityfiniteincrease}, and Lemma \ref{lemma:qualityfiniteconstant}.
\end{proof}

\section{Proofs for Section \ref{sec:inverseU}}\label{appendix:inverseU}

\subsection{Proof of Theorem \ref{thm:detectedattributeUshaped}}

We split into two lemmas.

\begin{lemma}
\label{lemma:Ushaped}
Fix $\Penalty > 0$, $\Threshold < \infty$, $\Quality$, and $\CostFn$.   For any type $\TypeNoArg \in \mathcal{T}$, it holds that:
\[ \HumanContent{\TypeNoArg}_1 \le \FinalContent{\TypeNoArg}_1(\emptyset; \Quality, \CostFn) \text{ and } \FinalContent{\TypeNoArg}_1(\emptyset; \Quality, \CostFn) \ge  \FinalContent{\TypeNoArg}_1(\Penalty, \Threshold; \Quality, \CostFn). \]
\end{lemma}
\begin{proof}
First, we apply Theorem \ref{thm:nodetector} to see that $\FinalContent{\TypeNoArg}_1(\emptyset; \Quality, \CostFn) = \Combination{\TypeNoArg}_1(\FractionNoArg^*)$. Moreover, using the assumption on the type space that $\HumanContent{\TypeNoArg}_1 < \LLMContent{\TypeNoArg}_1$, we know that $\Combination{\TypeNoArg}_1(\FractionNoArg)$ is increasing in $\FractionNoArg$.  

We show that $\HumanContent{\TypeNoArg}_1 \le \FinalContent{\TypeNoArg}_1(\emptyset; \Quality, \CostFn)$. This follows from the fact that: 
\[\FinalContent{\TypeNoArg}_1(\emptyset; \Quality, \CostFn) = \Combination{\TypeNoArg}_1(\FractionNoArg^*) \ge \Combination{\TypeNoArg}_1(0) = \HumanContent{\TypeNoArg}_1\]
as desired.

Now, we show that $\FinalContent{\TypeNoArg}_1(\emptyset; \Quality, \CostFn) \ge  \FinalContent{\TypeNoArg}_1(\Penalty, \Threshold; \Quality, \CostFn)$. We split into two cases: infinite gaming costs and finite gaming costs. 

\paragraph{Case 1: Infinite gaming costs.} Using Theorem \ref{thm:infinitecostdetector}, we know that $\FinalContent{\TypeNoArg}_1(\Penalty, \Threshold, \Quality, \CostFn^{\infty}) = \Combination{\TypeNoArg}_1(\Fraction{\TypeNoArg}(\Penalty, \Threshold, \Quality, \CostFn^{\infty}))$. We split into subcases: $\Combination{\TypeNoArg}_1(\FractionNoArg^*) \le \Threshold$ and $\Combination{\TypeNoArg}_1(\FractionNoArg^*) > \Threshold$.

\textit{Case 1a: $\Combination{\TypeNoArg}_1(\FractionNoArg^*) \le \Threshold$.} Using Theorem \ref{thm:infinitecostdetector}, we know that $\Fraction{\TypeNoArg}(\Penalty, \Threshold, \Quality, \CostFn^{\infty}) = \FractionNoArg^*$, which means that:
\[\FinalContent{\TypeNoArg}_1(\emptyset; \Quality, \CostFn) = \Combination{\TypeNoArg}_1(\FractionNoArg^*) = \Combination{\TypeNoArg}_1(\Fraction{\TypeNoArg}(\Penalty, \Threshold, \Quality, \CostFn^{\infty})) = \FinalContent{\TypeNoArg}_1(\Penalty, \Threshold, \Quality, \CostFn^{\infty}) \]
as desired. 

\textit{Case 1b: $\Combination{\TypeNoArg}_1(\FractionNoArg^*) > \Threshold$.}
Using Theorem \ref{thm:infinitecostdetector}, we know that  $\Fraction{\TypeNoArg}(\Penalty, \Threshold, \Quality, \CostFn^{\infty}) \in \left\{ \FractionNoArg^*, \FractionNoArg^{\text{game}}\right\}$. If $\Fraction{\TypeNoArg}(\Penalty, \Threshold, \Quality, \CostFn^{\infty}) = \FractionNoArg^*$, this again means that
\[\FinalContent{\TypeNoArg}_1(\emptyset; \Quality, \CostFn) = \Combination{\TypeNoArg}_1(\FractionNoArg^*) = \Combination{\TypeNoArg}_1(\Fraction{\TypeNoArg}(\Penalty, \Threshold, \Quality, \CostFn^{\infty})) =  \FinalContent{\TypeNoArg}_1(\Penalty, \Threshold, \Quality, \CostFn^{\infty}). \]
If $\Fraction{\TypeNoArg}(\Penalty, \Threshold, \Quality, \CostFn^{\infty}) = \FractionNoArg^{\text{game}}$, then we know that  $\FractionNoArg^{\text{game}} < \infty$. Since $\FractionNoArg^{\text{game}} < \infty$, then we know that $\Combination{\TypeNoArg}_1(\FractionNoArg^{\text{game}}) = \Threshold$. This would mean that:
\[\FinalContent{\TypeNoArg}_1(\emptyset; \Quality, \CostFn) = \Combination{\TypeNoArg}_1(\FractionNoArg^*) > \Threshold = \Combination{\TypeNoArg}_1(\FractionNoArg^{\text{game}})  = \Combination{\TypeNoArg}_1(\Fraction{\TypeNoArg}(\Penalty, \Threshold, \Quality, \CostFn^{\infty})) = \FinalContent{\TypeNoArg}_1(\Penalty, \Threshold, \Quality, \CostFn^{\infty}). \]
 
\paragraph{Case 2: Finite gaming costs.} We split into subcases: $\Combination{\TypeNoArg}_1(\FractionNoArg^*) \le \Threshold$ and $\Combination{\TypeNoArg}_1(\FractionNoArg^*) > \Threshold$.

\textit{Case 2a: $\Combination{\TypeNoArg}_1(\FractionNoArg^*) \le \Threshold$.} Using Theorem \ref{thm:detectorfinitecosts}, we know that $\Fraction{\TypeNoArg}(\Penalty, \Threshold; \Quality, \CostFn) = \FractionNoArg^*$ and $\FinalContent{\TypeNoArg}_1(\Penalty, \Threshold; \Quality, \CostFn) = \Combination{\TypeNoArg}_1(\Fraction{\TypeNoArg}(\Penalty, \Threshold; \Quality, \CostFn))$, which means that:
\[\FinalContent{\TypeNoArg}_1(\emptyset; \Quality, \CostFn) = \Combination{\TypeNoArg}_1(\FractionNoArg^*) = \Combination{\TypeNoArg}_1(\Fraction{\TypeNoArg}(\Penalty, \Threshold; \Quality, \CostFn)) \ge \FinalContent{\TypeNoArg}_1(\Penalty, \Threshold; \Quality, \CostFn) \]
as desired. 

\textit{Case 2b: $\Combination{\TypeNoArg}_1(\FractionNoArg^*) > \Threshold$.}
Using Theorem \ref{thm:detectorfinitecosts}, we know that  $\Fraction{\TypeNoArg}(\Penalty, \Threshold; \Quality, \CostFn) \in \left\{ \FractionNoArg^*, \FractionNoArg^{\text{game}}\right\} \cup A^{\TypeNoArg}(\Penalty)$. If $\Fraction{\TypeNoArg}(\Penalty, \Threshold; \Quality, \CostFn) = \FractionNoArg^*$, this again means that
\[\FinalContent{\TypeNoArg}_1(\emptyset; \Quality, \CostFn) = \Combination{\TypeNoArg}_1(\FractionNoArg^*) = \Combination{\TypeNoArg}_1(\Fraction{\TypeNoArg}(\Penalty, \Threshold; \Quality, \CostFn)) \ge \FinalContent{\TypeNoArg}_1(\Penalty, \Threshold; \Quality, \CostFn). \]
If $\Fraction{\TypeNoArg}(\Penalty, \Threshold, \Quality) = \FractionNoArg^{\text{game}}$, then we know that  $\FractionNoArg^{\text{game}} < \infty$. Since $\FractionNoArg^{\text{game}} < \infty$, then we know that $\Combination{\TypeNoArg}_1(\FractionNoArg^{\text{game}}) = \Threshold$. This would mean that:
\[\FinalContent{\TypeNoArg}_1(\emptyset; \Quality, \CostFn) = \Combination{\TypeNoArg}_1(\FractionNoArg^*) > \Threshold = \Combination{\TypeNoArg}_1(\FractionNoArg^{\text{game}}) = \Combination{\TypeNoArg}_1(\Fraction{\TypeNoArg}(\Penalty, \Threshold; \Quality, \CostFn)) \ge \FinalContent{\TypeNoArg}_1(\Penalty, \Threshold; \Quality, \CostFn). \]
If $\Fraction{\TypeNoArg}(\Penalty, \Threshold; \Quality, \CostFn) \in A^{\TypeNoArg}(\Penalty)$ and $\Fraction{\TypeNoArg}(\Penalty, \Threshold; \Quality, \CostFn) \neq \FractionNoArg^*$, then we know by Theorem \ref{thm:detectorfinitecosts} that $\Combination{\TypeNoArg}_1(\FractionNoArg^*) > \Threshold$, which means that:
\[\FinalContent{\TypeNoArg}_1(\emptyset; \Quality, \CostFn) = \Combination{\TypeNoArg}_1(\FractionNoArg^*) > \Threshold \ge G(\Combination{\TypeNoArg}_1(\Fraction{\TypeNoArg}(\Penalty,\Threshold,  \Quality, \CostFn))) = \FinalContent{\TypeNoArg}_1(\Penalty, \Threshold; \Quality, \CostFn). \]
as desired. 
\end{proof}

\begin{lemma}
\label{lemma:strictineq}
Fix $\Penalty > 0$, $\Threshold < \infty$, $\Quality$, $\CostFn$.  There exists a type $\TypeNoArg \in \mathcal{T}$ such that 
\[ \HumanContent{\TypeNoArg}_1 < \FinalContent{\TypeNoArg}_1(\emptyset; \Quality, \CostFn) \text{ and } \FinalContent{\TypeNoArg}_1(\emptyset; \Quality, \CostFn) >  \FinalContent{\TypeNoArg}_1(\Penalty, \Threshold; \Quality, \CostFn). \]
\end{lemma}
\begin{proof}
We construct $\TypeNoArg$ so that  $\HumanContent{\TypeNoArg}_i = \LLMContent{\TypeNoArg}_i$ for $i \neq 1$, so the two outputs only differ along the first coordinate. We take $\HumanContent{\TypeNoArg}_1 = \Threshold - 1$ and $\LLMContent{\TypeNoArg}_1 = \Threshold + \epsilon$ where $\epsilon > 0$.

First, we show that for sufficiently large $\ProductionCost{\TypeNoArg}$ it holds that $\FinalContent{\TypeNoArg}(\emptyset; \Quality, \CostFn) = \LLMContent{\TypeNoArg}$ for any $0 < \epsilon < 1/2$. Using Theorem \ref{thm:nodetector}, we know that $\FinalContent{\TypeNoArg}(\emptyset; \Quality, \CostFn) = \Combination{\TypeNoArg}(\FractionNoArg^*)$, so it suffices to show that $\FractionNoArg^* = 1$. Recall that $\FractionNoArg^*$ is the minimum optimizer of $\max_{\FractionNoArg \in [0,1]} (\Quality(\Combination{\TypeNoArg}(\FractionNoArg)) - \ProductionCost{\TypeNoArg}(1 - \FractionNoArg))$. Since the objective $\Quality(\Combination{\TypeNoArg}(\FractionNoArg)) - \ProductionCost{\TypeNoArg}(1 - \FractionNoArg)$ is concave, it suffices to show that the derivative as $\FractionNoArg \rightarrow 1$ is strictly positive. Note that this is equal to 
\[\Quality'_1(\LLMContent{\TypeNoArg}_1) \cdot  (\LLMContent{\TypeNoArg}_1 - \HumanContent{\TypeNoArg}_1) + \ProductionCost{\TypeNoArg} \ge \ProductionCost{\TypeNoArg} -\frac{3}{2} \cdot \sup_{0 < \epsilon < 1/2} \left|\Quality'_1(\Threshold + \epsilon) \right|.\] This is strictly positive for sufficiently large $\ProductionCost{\TypeNoArg}$, as desired. 

For $\ProductionCost{\TypeNoArg}$ satisfying the above, note that:
\[ \HumanContent{\TypeNoArg}_1 < \FinalContent{\TypeNoArg}_1(\emptyset; \Quality, \CostFn) = \LLMContent{\TypeNoArg}_1, \]
which proves the first inequality. 

Next, for any fixed value of $\ProductionCost{\TypeNoArg}$, we show that for sufficiently small $\epsilon > 0$ it holds that $(\Fraction{\TypeNoArg}(\Penalty, \Threshold; \Quality, \CostFn), \FinalContent{\TypeNoArg}(\Penalty, \Threshold; \Quality, \CostFn)) \neq (1, \LLMContent{\TypeNoArg})$. It suffices to show that $\Utility{\TypeNoArg}(\FractionNoArg^{\text{game}}, \Combination{\TypeNoArg}(\FractionNoArg^{\text{game}}), \Penalty, \Threshold, \Quality, \CostFn) > \Utility{\TypeNoArg}(1, \LLMContent{\TypeNoArg}, \Penalty, \Threshold, \Quality, \CostFn)$. For any $\epsilon > 0$, it holds that: 
\begin{align*}
\Utility{\TypeNoArg}(\FractionNoArg^{\text{game}}, \Combination{\TypeNoArg}(\FractionNoArg^{\text{game}})) - \Utility{\TypeNoArg}(1, \LLMContent{\TypeNoArg}, \Penalty, \Threshold, \Quality, \CostFn) &= \Quality(\Combination{\TypeNoArg}(\FractionNoArg^{\text{game}})) - \Quality(\LLMContent{\TypeNoArg})+ \Penalty - \ProductionCost{\TypeNoArg}(1 - \FractionNoArg^{\text{game}}) \\
&= \Quality_1(\Threshold) - \Quality_1(\Threshold + \epsilon)+ \Penalty - \ProductionCost{\TypeNoArg} \cdot  \frac{\epsilon}{1 + \epsilon} \\
&\ge_{(A)} \Penalty -\epsilon \cdot \Quality'_1(\Threshold) - \ProductionCost{\TypeNoArg} \cdot  \epsilon
\end{align*}
where (A) uses the concavity of $\Quality_1$. For sufficiently small $\epsilon > 0$, we see that this is strictly positive. 

Now, taking values of $\ProductionCost{\TypeNoArg}$ and $\epsilon$ that satisfy the above, we split into two cases: (1) $\Fraction{\TypeNoArg}(\Penalty, \Threshold; \Quality, \CostFn) \neq 1$ and (2) $\Fraction{\TypeNoArg}(\Penalty, \Threshold; \Quality, \CostFn) = 1$ and $\FinalContent{\TypeNoArg}(\Penalty, \Threshold; \Quality, \CostFn)) \neq  \LLMContent{\TypeNoArg})$. 

\paragraph{Case 1: $\Fraction{\TypeNoArg}(\Penalty, \Threshold; \Quality, \CostFn) \neq 1$.} In this case, we know that: \[\FinalContent{\TypeNoArg}_1(\Penalty, \Threshold; \Quality, \CostFn) \le \Combination{\TypeNoArg}_1(\Fraction{\TypeNoArg}(\Penalty, \Threshold; \Quality, \CostFn)) < \LLMContent{\TypeNoArg}_1 = \FinalContent{\TypeNoArg}_1(\emptyset; \Quality, \CostFn),\]
which proves the second inequality.

\paragraph{Case 2: $\Fraction{\TypeNoArg}(\Penalty, \Threshold; \Quality, \CostFn) = 1$ and $\FinalContent{\TypeNoArg}(\Penalty, \Threshold; \Quality, \CostFn) \neq  \LLMContent{\TypeNoArg}$.} By Theorems \ref{thm:infinitecostdetector} and \ref{thm:detectorfinitecosts}, we know that $\FinalContent{\TypeNoArg}_i(\Penalty, \Threshold; \Quality, \CostFn) = \Combination{\TypeNoArg}_i(\Fraction{\TypeNoArg}(\Penalty, \Threshold; \Quality, \CostFn))$ for all $i \ge 2$. Thus, the assumptions of this case mean that $\FinalContent{\TypeNoArg}_1(\Penalty, \Threshold; \Quality, \CostFn)) \neq  \LLMContent{\TypeNoArg}_1$, which by the definition of post-processing means that $\FinalContent{\TypeNoArg}_1(\Penalty, \Threshold; \Quality, \CostFn)) <  \LLMContent{\TypeNoArg}_1$. Putting this together, we obtain that:
\[\FinalContent{\TypeNoArg}_1(\Penalty, \Threshold; \Quality, \CostFn)) <  \LLMContent{\TypeNoArg}_1 = \FinalContent{\TypeNoArg}_1(\emptyset; \Quality, \CostFn),\]
which proves the second inequality.
\end{proof}

Theorem \ref{thm:detectedattributeUshaped} follows from these lemmas. 
\begin{proof}[Proof of Theorem \ref{thm:detectedattributeUshaped}]
The first part follows from Lemma \ref{lemma:Ushaped}, and the second part follows from Lemma \ref{lemma:strictineq}. 
\end{proof}

\section{Additional details, ablations, and full word lists for Section \ref{subsec:empirical}}\label{appendix:empirical}

We provide additional empirical details in Appendix \ref{appendix:details}, ablations in Appendix \ref{appendix:ablations}, and full word lists in Appendix \ref{appendix:wordlists}. 

\begin{figure}
    \centering
\begin{tabular}{|c|c|c|c|}
    \hline
    Window & Style (RTF) & Topic (RTF) & Style (All) \\    \hline
    2014-17 & $0.0 \pm 0.0$ & $5.6 \pm 1.0$ & $13.6 \pm 1.0$ \\
    2015-18 & $0.0 \pm 0.0$ & $4.2 \pm 0.7$ & $10.6 \pm 1.5$ \\
    2016-19 & $0.6 \pm 0.8$ & $7.2 \pm 0.4$ & $16.8 \pm 3.5$ \\
    2017-20 & $0.0 \pm 0.0$ & $4.0 \pm 0.6$ & $18.2 \pm 2.7$ \\
    2018-21 & $0.2 \pm 0.4$ & $3.8 \pm 1.7$ & $10.4 \pm 2.2$ \\
    2019-22 & $0.4 \pm 0.5$ & $6.2 \pm 1.6$ & $15.2 \pm 2.7$ \\
    2020-23 & $0.0 \pm 0.0$ & $5.4 \pm 1.2$ & $18.4 \pm 1.7$ \\
    2021-24 & $0.2 \pm 0.4$ & $10.2 \pm 1.2$ & $24.0 \pm 2.1$ \\
    2022-25 & $26.4 \pm 1.4$ & $9.2 \pm 1.6$ & $55.4 \pm 1.6$ \\
    \hline
\end{tabular}
\caption{Numerical values shown in Figure \ref{tab:empirical}.}
\label{tab:empiricaltable}
\end{figure}

\subsection{Details of empirical setup}\label{appendix:details}

\paragraph{Details of building vocabulary $V$.} We only consider non-stop words that are at least 3 letters long. For each trial, we then build a vocabulary $V$ consisting of the 10,000 words with highest frequency across the sampled dataset with the papers in each month from 1/1/2013 to 12/31/2025. 

\paragraph{Details of LLM judge configuration.} We use the latest version of GPT-5.4-mini as of May 4th, 2026. We set the temperature to be the default value. We take the judge prompt to be ``"""You are an expert linguist analyzing word frequency changes in academic computer science papers (arXiv CS abstracts).

A word's usage frequency changed significantly over a time period. Classify it as either:

- "topic": The frequency change is primarily because the word relates to a research topic, method, dataset, tool, technology, or domain-specific concept that became more or less popular (e.g., "transformer", "bert", "diffusion", "covid", "blockchain", "federated", "adversarial", "pruning").

- "style": The frequency change primarily reflects a change in writing style, rhetoric, word choice, or language patterns not tied to a specific research topic (e.g., "delve", "comprehensive", "notably", "leveraging", "crucial", "showcasing", "underscores", "innovative").

Respond with ONLY "topic" or "style". No other text."""''. 

\paragraph{Details of ``rise-then-fall'' classification.} Given a word $w$ and an interval indexed by $s$, we compute the maximum frequency of the word across any quarter in [10/1/$s$, 9/30/($s+3$)]. We require that the maximum frequency is at least $R = 1.5$ times the frequency in the first quarter, that the final frequency is at most $F = 0.95$ times the maximum frequency, and that the number of sign changes is at most $C = 4$. 

All experiments are run on 1 CPU. 

\subsection{Ablations}\label{appendix:ablations}

We show ablations where we adjust the parameters of the rise-then-fall classification: the rise threshold (Figure \ref{fig:rise}), the fall threshold (Figure \ref{fig:fall}), and  the max number of sign changes (Figure \ref{fig:signs}). We also show ablations where we take GPT-5-nano-2025-08-07 to be the judge (Figure \ref{fig:judge}). Our results readily generalize across these settings. 

\begin{figure}
    \centering
\includegraphics[scale=0.45]{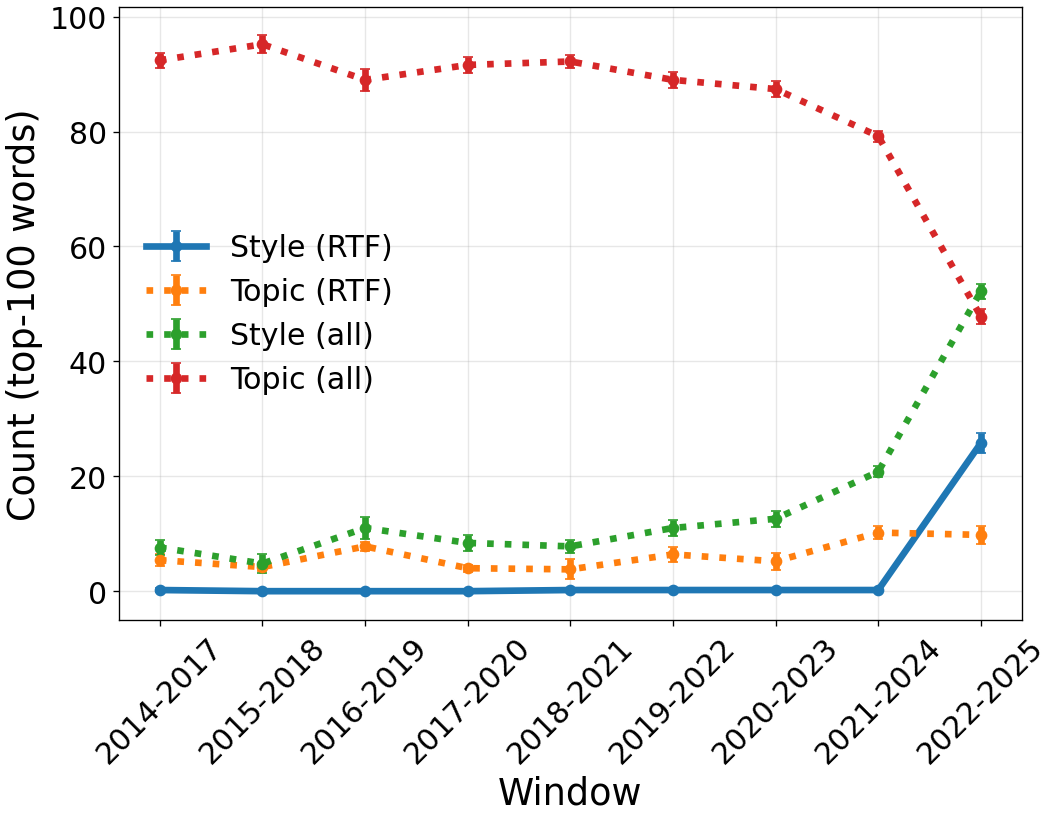}
    \caption{Variation of Figure \ref{tab:empirical} where the judge is GPT-5-nano-2025-08-07.}
    \label{fig:judge}
\end{figure}

\begin{figure}
\begin{subfigure}{0.5\textwidth}
\includegraphics[width=\textwidth]{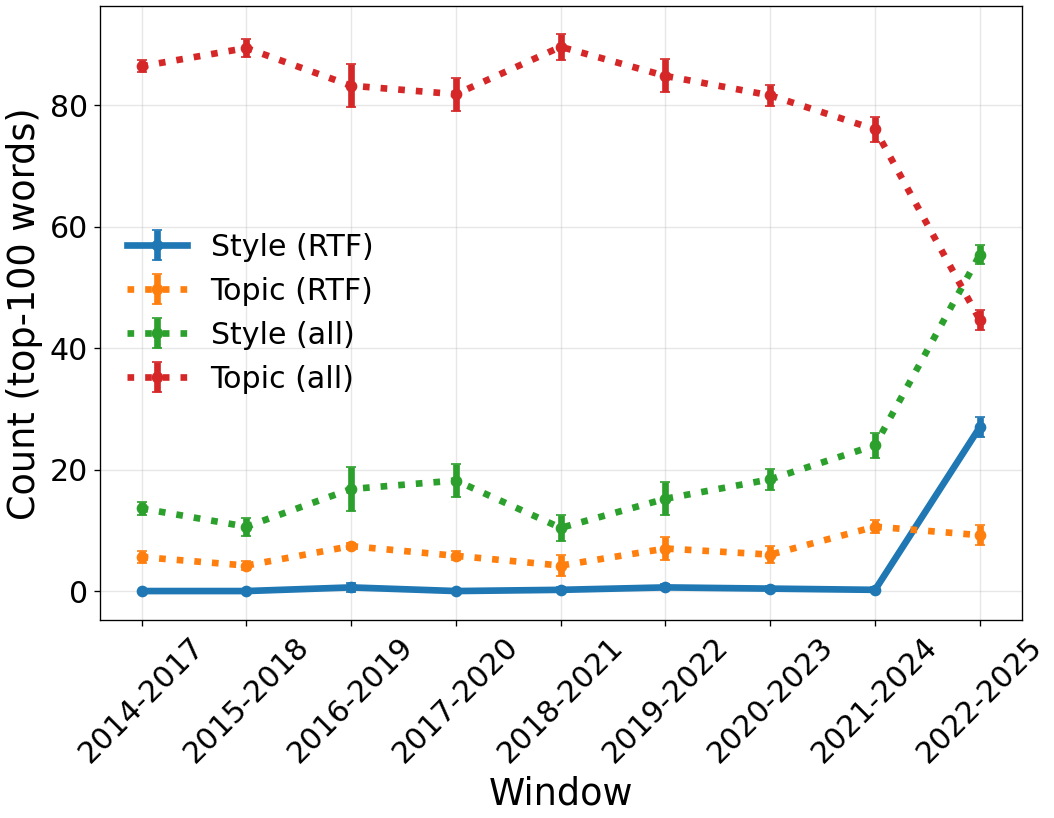}
\caption{$R = 1.1$}
\end{subfigure}
\begin{subfigure}{0.5\textwidth}
\includegraphics[width=\textwidth]{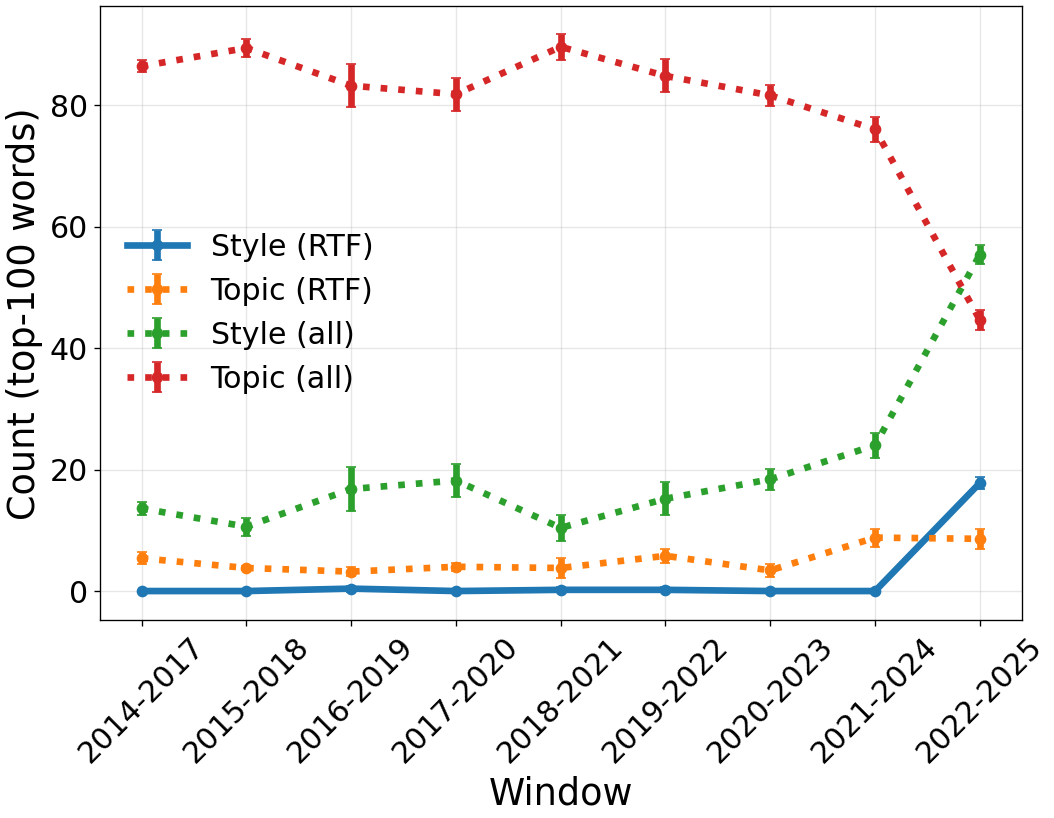}
\caption{$R = 3.0$}
\end{subfigure}
\caption{Variation of Figure \ref{tab:empirical} where rise-then-fall classification takes the rise threshold to 1.1 (left) and 3 (right), rather than 1.5. }
\label{fig:rise}
\end{figure}

\begin{figure}
\begin{subfigure}{0.5\textwidth}
\includegraphics[width=\textwidth]{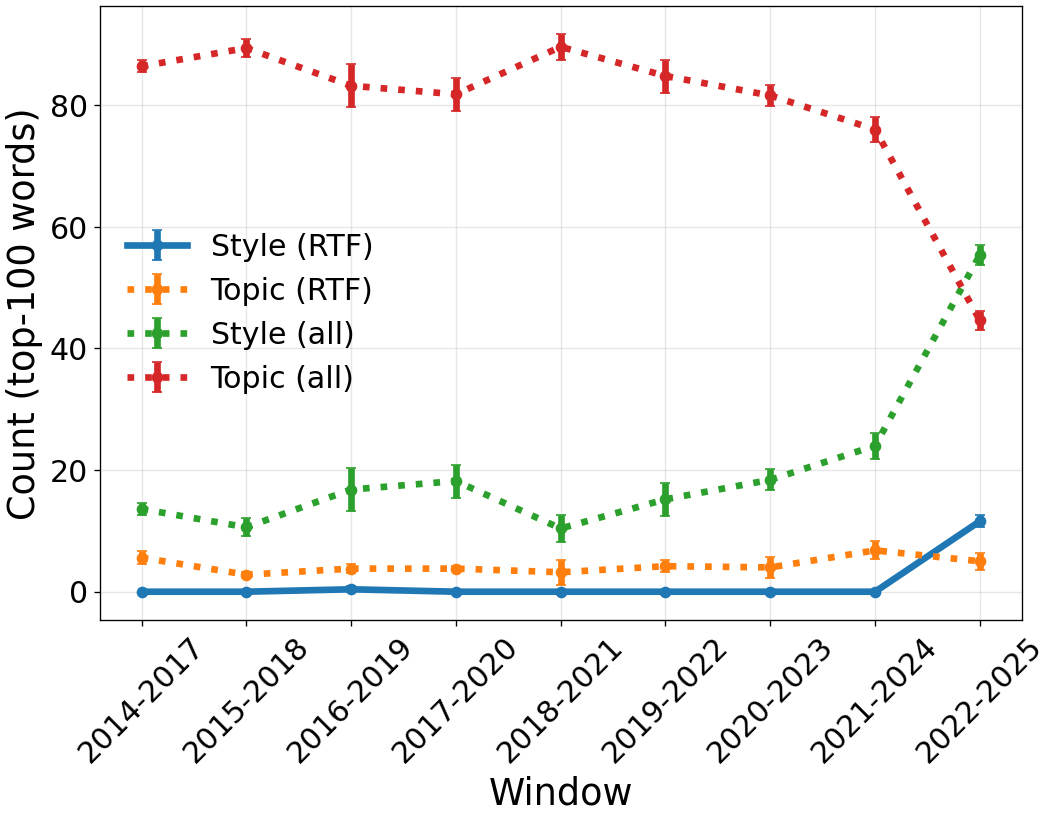}
\caption{$F = 0.8$}
\end{subfigure}
\begin{subfigure}{0.5 \textwidth}
\includegraphics[width=\textwidth]{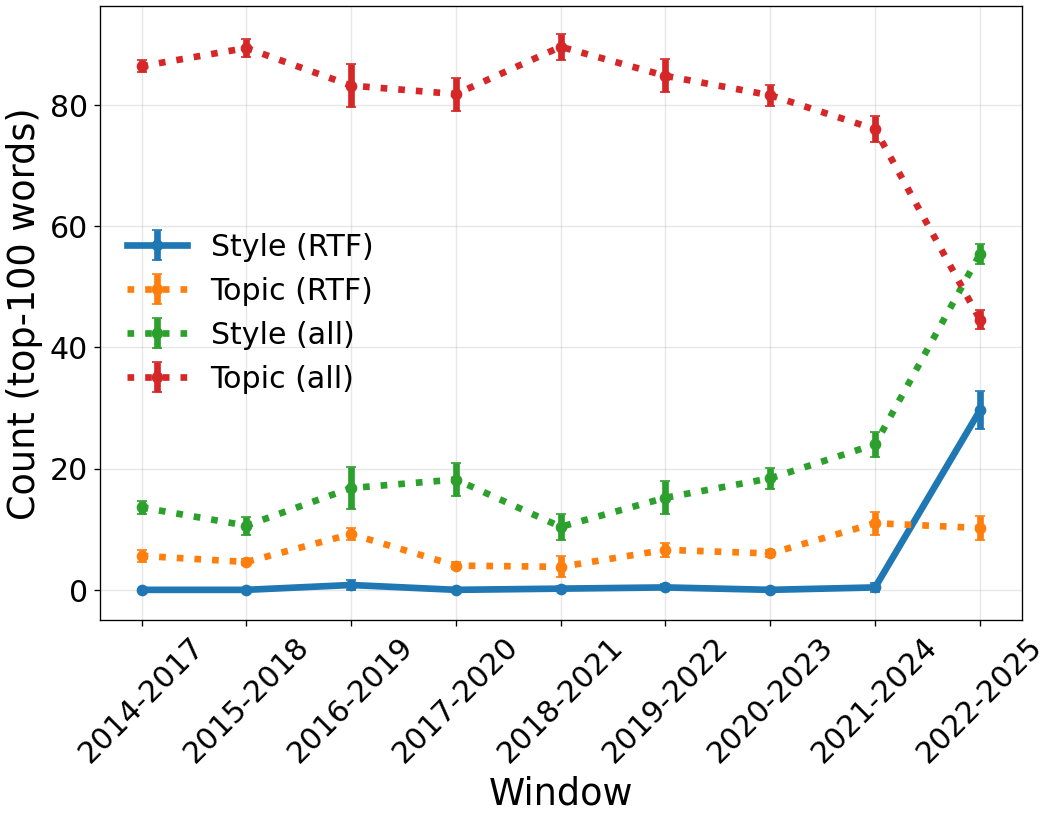}
\caption{$F = 0.98$}
\end{subfigure}
\caption{Variation of Figure \ref{tab:empirical} where rise-then-fall classification takes the fall threshold to 0.8 (left) and 0.98 (right), rather than 0.95. }
\label{fig:fall}
\end{figure}

\begin{figure}
\begin{subfigure}{0.5\textwidth}
\includegraphics[width=\textwidth]{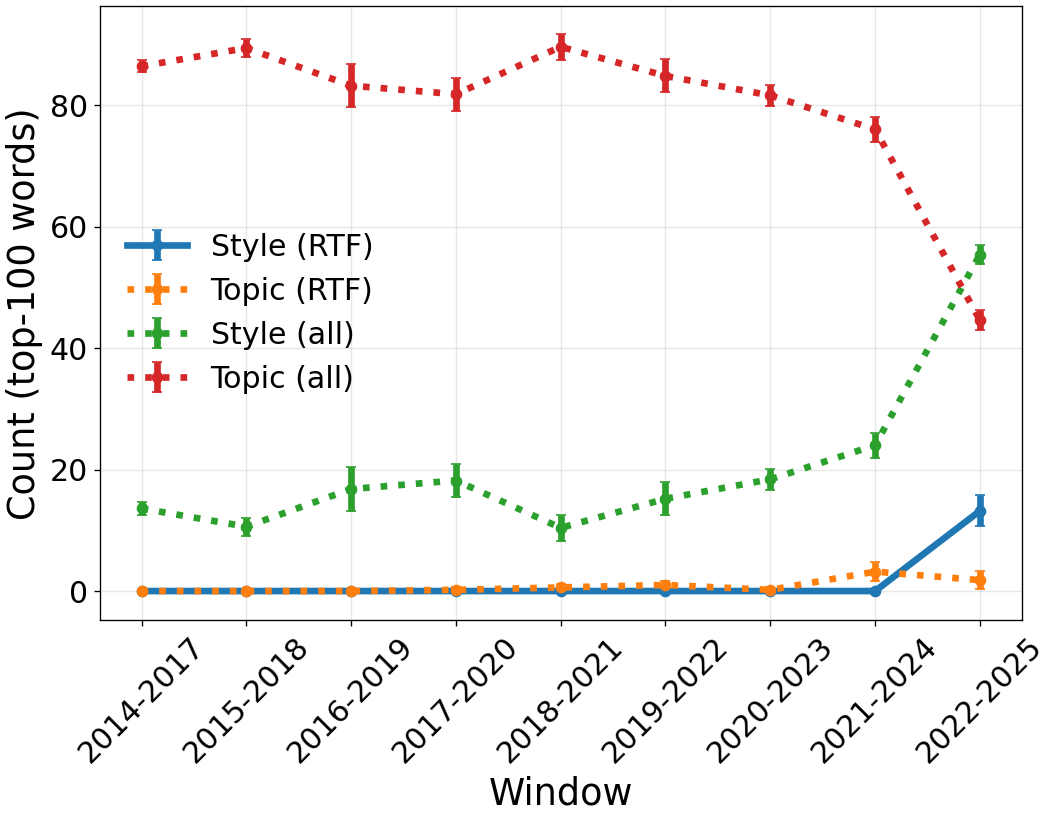}
\caption{$C = 2$}
\end{subfigure}
\begin{subfigure}{0.5 \textwidth}
\includegraphics[width=\textwidth]{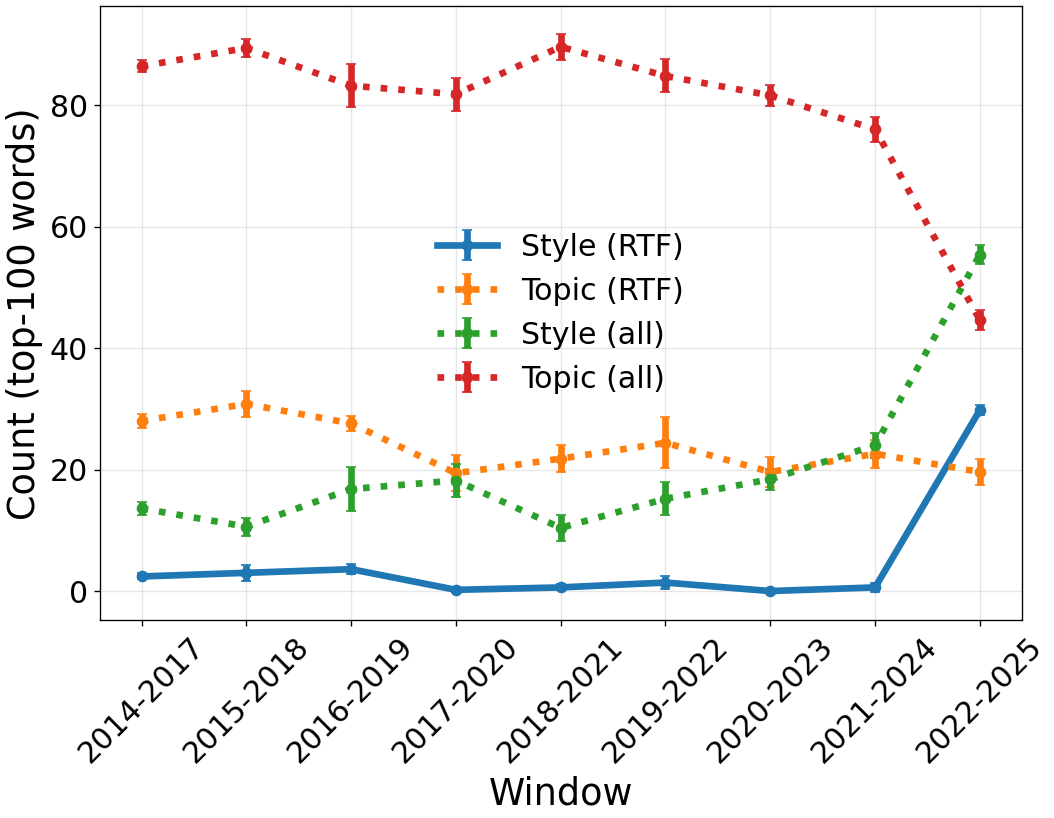}
\caption{$C = 6$}
\end{subfigure}
\caption{Variation of Figure \ref{tab:empirical} where rise-then-fall classification takes the number of sign changes $C$ is set to 2 (left) and 6 (right), rather than 4. }
\label{fig:signs}
\end{figure}

\newpage 

\subsection{Lists of words}\label{appendix:wordlists}

We show a list of rise-then-fall style words, rise-then-fall topic words, and non-rise-then-fall style words, taking one of the trials as an illustrative example (Table \ref{table:wordlist}).

\begin{table}
\small 
\begin{tabular}{|p{1.2cm}|p{4.0cm}|p{3.5cm}|p{6.2cm}|}
    \hline
    Window & RTF style & RTF topic & Non-RTF style \\
    \hline
    2014-2017 & \emph{(none)} & leaky, microservices, pandemic, relu, stdp & adequacy, attentive, benefiting, caption, crafting, essays, footprints, informativeness, learn, listeners, localizes, necessitating, trained \\
    \hline
    2015-2018 & \emph{(none)} & aoi, celeba, lstm, rnn, train & argumentative, cer, frozen, individualized, irl, isa, paraphrases, regressing, tasks, trained \\
    \hline
    2016-2019 & \emph{(none)} & autonomous, convolutional, deep, generative, loss, unet, urllc & abusive, audits, bells, bound, driving, duplicates, https, ideological, memorability, methods, models, slip, sota, supervisions, tasks, tee, tees, trained, whistles \\
    \hline
    2017-2020 & \emph{(none)} & bert, capsules, hpo, mask, ntk & achievable, asc, bound, bounds, capacity, experiments, explanations, however, https, learn, methods, models, number, optimal, problem, rate, schemes, tasks, trained, users \\
    \hline
    2018-2021 & explanations & coronavirus, counterfactual, distancing, pandemic, vln & contextualized, data, downstream, experiments, however, https, methods, models, problem, sota, tasks, vos \\
    \hline
    2019-2022 & \emph{(none)} & contrastive, coronavirus, federated, gnns, nerfs, pandemic, swin, vit & downstream, existing, explanations, extensive, https, methods, models, number, order, paper, pretext, problem, research, shifts, sota, tasks, tod \\
    \hline
    2020-2023 & \emph{(none)} & gnns, masked, metaverse, peft & downstream, existing, extensive, global, however, https, methods, models, number, order, paper, problem, scenarios, shifts, sota, tasks, tod, word, works \\
    \hline
    2021-2024 & \emph{(none)} & aigc, chatgpt, denoising, isac, llm, llms, masked, metaverse, rlhf & additionally, address, advancements, art, based, capabilities, comprehensive, diverse, downstream, effectiveness, enhance, enhancing, existing, extensive, however, https, issue, limitations, models, number, paper, performance, potential, problem, sota, tasks, various, word \\
    \hline
    2022-2025 & additionally, addressing, advancement, advancements, capabilities, challenges, delves, effectively, encompassing, enhance, enhances, enhancing, ensuring, facilitating, innovative, integrating, intricate, leveraging, notably, potential, significant, thereby, underscores, utilizing & aigc, clip, hallucinations, instruction, lmms, mamba, mistral, prompting, rag & across, address, comprehensive, crucial, diverse, enabling, findings, highlighting, insights, integration, introduce, introduces, large, models, offering, offers, one, particularly, pivotal, problem, proposed, remarkable, show, showcasing, struggle, tailored, underscore, underscoring, used, various, within \\
    \hline
\end{tabular}
\caption{Full list of rise-then-fall style words, rise-then-fall topic words, and non-rise-then-fall words within the top 100 words with most greatest change. The list is generated from one of the trials, following the empirical setup in Section \ref{subsec:empirical}.}
\label{table:wordlist}
\end{table}

\end{document}